\documentclass[twoside,11pt]{article}

\usepackage{jmlr2e}
\usepackage[utf8]{inputenc}
\usepackage{cancel}
\usepackage{amsmath}
\usepackage{amsfonts}
\usepackage{lineno}
\usepackage{subcaption}
\usepackage{multirow}
\usepackage{float}
\usepackage{xcolor}
\usepackage{tcolorbox}
\RequirePackage{enumerate,soul}
\usepackage{tabularx,multicol,booktabs}
\usepackage{mhequ}
\definecolor{darkred}{RGB}{150,50,50}
\definecolor{brown}{RGB}{250,100,100}
\definecolor{green}{RGB}{000,150,100}
\definecolor{purple}{RGB}{250,000,180}
\def\purple{\color{purple}}
\def\blue{\color{blue}}

\def\brown{\color{brown}}

\jmlrheading{1}{2000}{1-48}{4/00}{10/00}{meila00a}{}

\firstpageno{1}
\ShortHeadings{}{Semi-Supervised Off Policy Reinforcement Learning}

\begin{document}
\title{Semi-Supervised Off Policy Reinforcement Learning}

\author{\name Aaron Sonabend-W \email asonabend@g.harvard.edu \\
       \addr Department of Biostatistics\\
       Harvard University\\
       Boston, MA 02115, USA
       \AND
\name Nilanjana Laha \email nlaha@hsph.harvard.edu \\
       \addr Department of Biostatistics\\
       Harvard University\\
       Boston, MA 02115, USA
              \AND
\name Ashwin N. Ananthakrishnan
\email aananthakrishnan@mgh.harvard.edu  \\
       \addr Division of Gastroenterology\\ Massachusetts General Hospital\\ Boston, MA 02115, USA
       \AND
\name Tianxi Cai
\email tcai@hsph.harvard.edu  \\
       \addr Department of Biostatistics\\
       Harvard University\\
       Boston, MA 02115, USA
       \AND
\name Rajarshi Mukherjee \email ram521@mail.harvard.edu  \\
       \addr Department of Biostatistics\\
       Harvard University\\
       Boston, MA 02115, USA
}
\editor{}
\maketitle
\begin{abstract}
Reinforcement learning (RL) has shown great success in estimating sequential treatment strategies which take into account patient heterogeneity. However, health-outcome information, which is used as the reward for reinforcement learning methods, is often not well coded but rather embedded in clinical notes. Extracting precise outcome information is a resource intensive task, so most of the available well-annotated cohorts are small. To address this issue, we propose a semi-supervised learning (SSL) approach that efficiently leverages a small sized labeled data with true outcome observed, and a large unlabeled data with outcome surrogates. In particular, we propose a semi-supervised, efficient approach to Q-learning and doubly robust off policy value estimation. Generalizing SSL to sequential treatment regimes brings interesting challenges: 1) Feature distribution for Q-learning is unknown as it includes previous outcomes. 2) The surrogate variables we leverage in the modified SSL framework are predictive of the outcome but not informative to the optimal policy or value function. We provide theoretical results for our Q-function and value function estimators to understand to what degree efficiency can be gained from SSL. Our method is at least as efficient as the supervised approach, and moreover safe as it robust to mis-specification of the imputation models.
\end{abstract}




\section{Introduction}

Finding optimal treatment strategies that can incorporate patient heterogeneity is a cornerstone of personalized medicine. When treatment options change over time, optimal sequential 
treatment rules (STR) can be learned using longitudinal patient
data. With increasing availability of large-scale longitudinal data such as electronic health records (EHR) data in recent years, reinforcement learning (RL) has found much success in estimating such optimal STR   \citep{2019PM}. Existing RL methods include G-estimation \citep{robins2004},  Q-learning \citep{Watkins1989,murphy2005}, A-learning \citep{MurphyAlearning} and directly maximizing the value function \citep{Zhao2015}. Both G-estimation and $A$-learning attempt to model only the component of the outcome regression relevant to the treatment contrast, while $Q$-learning posits complete models for the outcome regression. Although G-estimation and $A$-learning models can be more efficient and robust to mis-specification, $Q$-learning is widely adopted due to its ease of implementation, flexibility and interpretability \citep{Watkins1989,DTRbook,schulte2014}.

Learning STR with EHR data, however, often faces an additional challenge of whether outcome information is readily available. Outcome information, such as development of a clinical event or whether a patient is considered as a responder, is often not well coded but rather embedded in clinical notes. Proxy variables such as diagnostic codes or mentions of relevant clinical terms in clinical notes via natural language processing (NLP), while predictive of the true outcome, are often not sufficiently accurate to be used directly in place of the outcome \citep{hong2019semi,zhang2019high,cheng2020robust}. On the other hand, extracting precise outcome information often requires manual chart review, which is resource intensive, particularly when the outcome needs to be annotated over time. This indicates the need for a semi-supervised learning (SSL) approach that can efficiently leverage a small sized labeled data $\Lsc$ with true outcome observed and a large sized unlabeled data $\Usc$ for predictive modeling. It is worthwhile to note that the SSL setting differs from the standard missing data setting in that the probability of missing tends to 1 asymptotically, which violates the positivity assumption required by the classical missing data methods \citep{chakrabortty}.

While SSL methods have been well developed for prediction, classification and regression tasks \citep[e.g.][]{Chapelle2006,zhu05,BlitzerZ08,Wang2011,qiao2018,chakrabortty}, there is a paucity of literature on SSL methods for estimating optimal treatment rules. Recently, \cite{cheng2020robust} and \cite{kallus2020role} proposed SSL methods for estimating an average causal treatment effect. \cite{Finn2016} proposed a semi-supervised RL method which achieves impressive empirical results and outperforms simple approaches such as direct imputation of the reward. However, there are no theoretical guarantees and the approach lacks causal validity and interpretability within a domain context. Additionally, this method does not leverage available surrogates. In this paper, we fill this gap by proposing a theoretically justified SSL approach to Q-learning using  a large unlabeled data $\Usc$ which contains sequential observations on features $\bO$, treatment assignment $A$, and surrogates $\bW$ that are imperfect proxies of $Y$, as well as a small set of labeled data $\Lsc$ which contains true outcome $Y$ at multiple stages along with $\bO$, $A$ and $\bW$. We will also develop robust and efficient SSL approach to estimating the value function of the derived optimal STR, defined as the expected counterfactual outcome under the derived STR. 

To describe the main contributions of our proposed SSL approach to RL, we first note two important distinctions between the proposed framework and classical SSL methods. First, existing SSL literature often assumes that $\mathcal{U}$ is large enough that the feature distribution is known \citep{Wasserman2007}. However, under the RL setting, the outcome of the stage $t-1$, denoted by $Y_{t-1}$, becomes a feature of stage $t$ for predicting $Y_t$. As such, the feature distribution for predicting $Y_t$ can not be viewed as known in the $Q$-learning procedure. Our methods for estimating an optimal STR and its associated value function, carefully adapt to this sequentially missing data structure. Second, we modify the SSL framework to handle the use of surrogate variables $\bW$ which are predictive of the outcome through the joint law $\Pbb_{Y,\bO,A,\bW}$, but are not part of the conditional distribution of interest $\Pbb_{Y|\bO,A}$. To address these issues, we propose a two-step fitting procedure for finding an optimal STR and for estimating its value function in the SSL setting. Our method consists of using the outcome-surrogates ($\bW$) and features ($\bO,A$) for non-parametric estimation of the missing outcomes ($Y$). We subsequently use these imputations to estimate $Q$ functions, learn the optimal treatment rule and estimate its associated value function. We provide theoretical results to understand when and to what degree efficiency can be gained from  $\bW$ and $\bO,A$. 

We further show that  our approach is robust to mis-specification of the imputation models. To account for potential mis-specification in the models for the $Q$ function, we provide a double robust value function estimator for the derived STR. If either the regression models for the Q functions or the propensity score functions are correctly specified, our value function estimators are consistent for the true value function. 

We organize the rest of the paper as follows. In Section \ref{section: set up} we formalize the problem mathematically and provide some notation to be used in the development and analysis of the methods. In Section \ref{section: SS Q learning} we discuss traditional $Q$-learning and propose an SSL estimation procedure for the optimal STR.  Section \ref{section: SS value function} details an SSL doubly robust estimator of the value function for the derived STR. In Section \ref{theory} we provide theoretical guarantees for our approach and discuss implications of our assumptions and results. Section \ref{section: sims and application} is devoted for numerical experiments  as well as real data analysis with an inflammatory bowel disease (IBD) data-set. We end with a discussion of the methods and possible extensions in Section \ref{section: discussion}. The proposed method has been implemented in R and the code can be found at \url{github.com/asonabend/SSOPRL}. Finally all the technical proofs and supporting lemmas are collected in Appendices \ref{sec:proof_main_results} and \ref{sec:proof_technical_lemmas}.

\section{Problem setup}\label{section: set up}

We consider a longitudinal observational study with outcomes, confounders and treatment indices potentially available over multiple stages. Although our method is generalizable for any number of stages, for ease of presentation we will use two time points of (binary) treatment allocation as follows. For time point $t\in \{1,2\}$, let  $\bO_t\in\mathbb{R}^{d^o_t}$ denote the vector of  covariates measured prior at stage $t$ of dimension $d^o_t$; $A_t\in\{0,1\}$ a treatment indicator variable; and $Y_{t+1}\in\mathbb{R}$ the outcome observed at stage $t+1$, for which higher values of $Y_{t+1}$ are considered beneficial. Additionally we observe surrogates $\bW_t \in\mathbb{R}^{d^\omega_ t}$, a $d^\omega_ t$-dimensional vector of post-treatment covariates potentially predictive of $Y_{t+1}$. In the labeled data where $\bY = (Y_2,Y_3)\trans$ is annotated, we observe a random sample of $n$ independent and identically distributed (iid) random vectors, denoted by 
$$
\mathcal{L}=\{\bL_i = (\bUvec_i\trans,\bY_i\trans)\trans\}_{i=1}^n, \quad \mbox{where $\bU_{ti}=(\bO_{ti}\trans, A_{ti},\bW_{ti}\trans)\trans$ and $\bUvec_i=(\bU_{1i}\trans,\bU_{2i}\trans)\trans$.}
$$ 
We additionally observe an unlabeled set consisting of $N$ iid random vectors, 
$$\mathcal{U}=\{\bUvec_{j}\}_{j=1}^N$$
with $N \gg n$. We denote the entire data as $\mathbb{S}=(\mathcal{L}\cup\mathcal{U})$. To operationalize our statistical arguments we denote the joint distribution of the observation vector $\bL_i$ in $\mathcal{L}$ as $\Pbb$. In order to connect to the unlabeled set, we assume that any observation vector $\bUvec_{j}$ in $\mathcal{U}$ has the distribution induced by $\Pbb$.

We are interested in finding the optimal STR and estimating its {\em value function} to be defined as expected counterfactual outcomes under the derived regime. To this end, let $Y_{t+1}^{(a)}$ be the potential outcome for a patient at time $t+1$ had the patient been assigned at time $t$ to treatment $a\in \{0,1\}$.  A dynamic treatment regime is a set of functions $\mathcal{D}=(d_1,d_2)$, where $d_t(\cdot)\in\{0,1\}$ , $t=1,2$ map from the patient's history up to time $t$ to the treatment choice $\{0,1\}$.
We define the patient's history as $\bH_1\equiv [\bH_{10}\trans, \bH_{11}\trans]\trans$ with $\bH_{1k} = \bphi_{1k}(\bO_1)$, $\bH_2=[\bH_{20}\trans, \bH_{21}\trans]\trans$ with $\bH_{2k}=\bphi_{2k}(\bO_1,A_1,\bO_2)$,  where $\{\bphi_{tk}(\cdot), t=1,2, k=0,1\}$ are pre-specified basis functions. We then define features derived from patient history for regression modeling as $\bX_1\equiv[\bH_{10}\trans,A_1\bH_{11}\trans]\trans$ and $\bX_2\equiv[\bH_{20}\trans,A_{2}\bH_{21}\trans]\trans$. For ease of presentation, we also let $\bHcheck_1 = \bH_1\trans$,  $\bHcheck_2 = (Y_2, \bH_2\trans)\trans$, $\bXcheck_1 = \bX_1$, $\bXcheck_2 = (Y_2, \bX_2\trans)\trans$, and $\bSigma_t=\mathbb{E}[\bXcheck_t\bXcheck_t\trans]$. 

Let $\Ebb_\Dsc$ be the expectation with respect to the measure that generated the data under regime $\Dsc$. Then these sets of rules $\Dsc$ have an associated value function which we can write as $V(\Dsc)=\mathbb{E}_\Dsc\left[Y_2^{(d_1)}+Y_3^{(d_2)}\right]$. Thus, an optimal dynamic treatment regime is a rule $\Dscbar=(\dbar_1,\dbar_2)$ such that $\Vbar=V\left(\Dscbar\right)\ge V\left(\Dsc\right)$ for all $\Dsc$ in a suitable class of admissible decisions \citep{DTRbook}. To identify $\Dscbar$ and $\Vbar$ from the observed data we will require the following sets of standard assumptions \citep{robins1997, schulte2014}: (i) consistency -- $Y_{t+1}=Y_{t+1}^{(0)}I(A_t=0)+Y_{t+1}^{(1)}I(A_t=1)\text{ for }t=1,2$, (ii) no unmeasured confounding -- $Y_{t+1}^{(0)},Y_{t+1}^{(1)}\indep A_t|\bH_t\text{ for }t=1,2$ and (iii) positivity -- $\Pbb(A_t|\bH_t)>\nu$, $\text{ for }t=1,2,\:A_t\in\{0,1\}$, for some fixed $\nu>0$. 

We will develop SSL inference methods to derive optimal STR $\Dscbar$ as well the associated value function $\Vbar$  by leveraging the richness of the unlabeled data and the predictive power of surrogate variables which allows us to gain crucial statistical efficiency. Our main contributions in this regard can be described as follows. First, we provide a systematic generalization of the $Q$-learning framework with theoretical guarantees to the semi-supervised setting with improved efficiency. Second, we provide a doubly robust estimator of the value function in the semi-supervised setup. Third, our $Q$-learning procedure and value function estimator are flexible enough to allow for standard off-the-shelf machine learning tools and are shown to perform well in finite-sample numerical examples. 


\section{Semi-Supervised \texorpdfstring{$Q$}{Lg}-learning}\label{section: SS Q learning}

In this section we propose a semi-supervised Q-learning approach to deriving an optimal STR. To this end, we first recall the basic mechanism of traditional linear parametric $Q$-learning  \citep{DTRbook} and then detail our proposed method. We defer the theoretical guarantees to Section \ref{theory}.  

\subsection{Traditional \texorpdfstring{$Q$}{Lg}-learning}

$Q$-learning is a backward recursive algorithm that identifies optimal STR by optimizing two stage Q-functions defined as: $$Q_2(\bHcheck_2,A_2)\equiv\mathbb{E}[Y_3|\bHcheck_2,A_2], \quad \mbox{and}\quad
Q_1(\bHcheck_1,A_1)\equiv\mathbb{E}[Y_2+\underset{a_2}{\text{max}}\:Q_2(\bHcheck_2,a_2)|\bHcheck_1,A_1]
$$ 
\citep{sutton2018,murphy2005}.
In order to perform inference one typically proceeds by positing models for the $Q$ functions. In its simplest form one assumes a (working) linear model for some parameters $\btheta_t=(\bbeta_t\trans,\bgamma_t\trans)\trans$, $t=1,2$, as follows: 
\begin{align}\label{linear_Qs}
\begin{split}
Q_1(\bHcheck_1,A_1;\btheta_1^0)=& \bXcheck_1\trans\btheta_1^0=\bH_{10}\trans \bbeta_1^0+A_{1}(\bH_{11}\trans \bgamma_1^0) ,\\
Q_2(\bHcheck_2,A_2;\btheta_2^0)=&\bXcheck_2\trans\btheta_2^0 = Y_2 \beta_{21}^0 +
\bH_{20}\trans\bbeta_{22}^0
+A_{2}(\bH_{21}\trans\bgamma_2^0).
\end{split}
\end{align}
Typical $Q$-learning consists of performing a least squares regression for the second stage to estimate $\bthetahat_2$ followed by defining the stage 1 pseudo-outcome for $i=1,...,n$ as 
\[
\Yhat_{2i}^*=Y_{2i}+\underset{a_2}{\text{max}}\:Q_2(\bHcheck_{2i},a_2;\bthetahat_2)=Y_{2i}(1+\hat\beta_{21})+\bH_{20i}\trans{\bbetahat}_{22}
+[\bH_{21i}\trans{\bgammahat}_2]_+,
\]
where $[x]_+=xI(x>0)$. One then proceeds to estimate $\bthetahat_1$ using least squares again, with $\Yhat_2^*$ as the outcome variable. Indeed, valid inference on $\Dscbar$ using the method described above crucially depends on the validity of the model assumed. However as we shall see, even without validity of this model we will be able to provide valid inference on suitable analogues of the $Q$-function working model parameters, and on the value function using a double robust type estimator. To that end it will be instructive to define the least square projections of $Y_3$ and $Y_2^*$ onto $\bXcheck_2$ and $\bXcheck_1$ respectively. The linear regression working models given by \eqref{linear_Qs} have $\btheta_1^0,\:\btheta_2^0$ as unknown regression parameters. To account for the potential mis-specification of the working models in \eqref{linear_Qs}, we define the target population parameters $\bthetabar_1,\bthetabar_2$ as the population solutions to the expected normal equations  
\[
\mathbb{E}\left\{\bXcheck_1(\Ybar_2^*-\bXcheck_1\trans\bthetabar_1)\right\}=\bzero, \quad \mbox{and}\quad \mathbb{E}\left\{ \bXcheck_2\trans\left(Y_3-\bXcheck_2\trans\bthetabar_2\right)\right\}=\bzero,
\]
where $\Ybar_2^*=Y_2+\underset{a_2}{\text{max}}\:Q_2(\bHcheck_2,a_2;\bthetabar_2)$. As these are linear in the parameters, uniqueness and existence for
$\bthetabar_1,\bthetabar_2$ are well defined. In fact, $Q_1(\bHcheck_1,A_1;\bthetabar_1)=\bXcheck_1\trans\bthetabar_1,Q_2(\bHcheck_2,A_2;\bthetabar_2)=\bXcheck_2\trans\bthetabar_2$ are the $L_2$ projection of $\mathbb{E}(Y_2^*|\bXcheck_1)\in\mathcal{L}_2\left(\Pbb_{\bXcheck_1}\right),\:\mathbb{E}(Y_3|\bXcheck_2)\in\mathcal{L}_2\left(\Pbb_{\bXcheck_2}\right)$ onto the subspace of all linear functions of $\bXcheck_1,\bXcheck_2\trans$ respectively. Therefore, $Q$ functions in \eqref{linear_Qs} are the best linear predictors of $\Ybar_2^*$ conditional on $\bXcheck_1$ and $Y_3$ conditional on $\bXcheck_2\trans$.

Traditionally, one only has access to labeled data $\mathcal{L}$, and hence proceeds by estimating $(\btheta_1,\btheta_2)$ in \eqref{linear_Qs} by solving the following sample version set of normal equations:
\begin{align}\label{full_EE}
\begin{split}
\Pbb_n 
\left[\begin{matrix}
\bXcheck_2(Y_3-\bXcheck_2\trans\btheta_2)
\end{matrix}\right] \equiv
\Pbb_n\left[\begin{matrix}
Y_2\{Y_3-(Y_2,\bX_2\trans)\btheta_2\}\\
\bX_2\{Y_3-(Y_2,\bX_2\trans)\btheta_2\}
\end{matrix}\right]  
=&\textbf{0},\\
\Pbb_n \left[
\bX_1
\{Y_2(1+\beta_{21})+\bH_{20}\trans{\bbeta}_{22}
+[\bH_{21}\trans{\bgamma}_2]_+-\bX_1\trans\btheta_1\}\right]=&\bf0.
\end{split}
\end{align}
\citep{DTRbook}, where $\Pbb_n$ denotes the empirical measure: i.e. for a measurable function $f:\mathbb{R}^p\mapsto\mathbb{R}$ and random sample $\{\bL_i\}_{i=1}^n$, $\Pbb_nf=\frac{1}{n}\sum_{i=1}^nf(\bL_i)$. 
The asymptotic distribution for the $Q$ function parameters in the fully-supervised setting has been well studied \cite[see][]{laber2014}. 

\subsection{Semi-supervised \texorpdfstring{$Q$}{Lg}-learning}\label{sec: ssQ}
We next detail our robust imputation-based semi-supervised $Q$-learning that leverages the unlabeled data $\mathcal{U}$ to replace the unobserved $Y_{t}$ in \eqref{full_EE} with their properly imputed values for subjects in $\Usc$. Our SSL procedure includes three key steps: (i) imputation, (ii) refitting, and (iii) projection to the unlabeled data. 
In step (i), we develop flexible imputation models for the conditional mean functions $\{\mu_t(\cdot), \mu_{2t}(\cdot), t = 2, 3\}$, where $\mu_t(\bUvec) = \mathbb{E}(Y_{t}|\bUvec)$ and  $\mu_{2t}(\bUvec) = \mathbb{E}(Y_{2}Y_{t}|\bUvec)$. 
The refitting in step (ii) will ensure the validity of the SSL estimators under potential mis-specifications of the imputation models. 

\subsubsection*{Step I: Imputation.}

Our first imputation step involves weakly parametric or non-parametric prediction modeling to approximate the conditional mean functions $\{\mu_t(\cdot), \mu_{2t}(\cdot), t = 2, 3\}$. Commonly used models such as non-parametric kernel smoothing, basis function expansion or kernel machine regression can be used. We denote the corresponding estimated mean functions as $\{\mhat_t(\cdot), \mhat_{2t}(\cdot), t = 2, 3\}$ under the corresponding imputation models $\{m_t(\bUvec), m_{2t}(\bUvec), t=2,3\}$. Theoretical properties of our proposed SSL estimators on specific choices of the imputation models are provided in section \ref{theory}. We also provide additional simulation results comparing different imputation models in section \ref{section: sims and application}.



\subsubsection*{Step II: Refitting.}
To overcome the potential bias in the fitting from the imputation model, especially under model mis-specification, we update the imputation model with an additional refitting step by expanding it to include linear effects of $\{\bX_t, t=1,2\}$ with cross-fitting to control overfitting bias. 
Specifically, to ensure the validity of the SSL algorithm from the refitted imputation model, we note that the final imputation models for $\{Y_t, Y_{2t}, t=2,3\}$, denoted by $\{\mubar_t(\bUvec), \mubar_{2t}, t=2,3\}$, need to satisfy 
\begin{equation*}
\begin{alignedat}{3}
 \Ebb\left[\bXvec\{Y_2 - \mubar_2(\bUvec)\}\right] = \bzero, &\quad&
\Ebb\left\{Y_2^2 - \mubar_{22}(\bUvec)\right\} =&0 &, \\
\Ebb\left[\bX_2\{Y_3 - \mubar_3(\bUvec)\}\right] = \bzero, &\quad&
\Ebb\left\{Y_2Y_3 - \mubar_{23}(\bUvec)\right\} =&0 &.
\end{alignedat}
\end{equation*}
where $\bXvec = (1,\bX_1\trans,\bX_2\trans)\trans$.
We thus propose a refitting step that expands $\{m_t(\bUvec), m_{2t}(\bUvec), t=2,3\}$ to additionally adjust for linear effects of $\bX_1$ and/or $\bX_2$ to ensure the subsequent projection step is unbiased. To this end, let $\{\mathcal{I}_k,k=1,...,K\}$ denote $K$ random equal sized partitions of the labeled index set $\{1,...,n\}$, and let $\{\mhat_{t}\supnk(\bUvec), \mhat_{2t}\supnk(\bUvec), t=2,3\}$ be the counterpart of $\{\mhat_t(\bUvec), \mhat_{2t}(\bUvec),t=2,3\}$ with labeled observations  in $\{1,..,n\}\setminus\mathcal{I}_k$. We then obtain $\bEtahat_2$, $\etahat_{22}$, $\bEtahat_3$, $\etahat_{23}$ respectively as the solutions to
\begin{equation}\label{eta_EE_Q1}
\begin{alignedat}{3}
\sum_{k=1}^{K}\sum_{i\in\mathcal{I}_k}\bXvec_i\left\{Y_{2i}- \mhat_2\supnk(\bUvec_i)-\bEta_2\trans\bXvec_i\right\}  = &\bzero, \ & \sum_{k=1}^{K}\sum_{i\in\mathcal{I}_k}
\left\{Y_{2i}^2-\mhat_{22}\supnk(\bUvec_i)-
\eta_{22}\right\} = &0& ,  \\
\sum_{k=1}^{K}\sum_{i\in\mathcal{I}_k}
\bX_{2i}\left\{Y_{3i}-\mhat_{3}\supnk(\bUvec_i)-
\bEta_3\trans\bX_{2i}
\right\} = &\bzero, \ & 
\sum_{k=1}^{K}\sum_{i\in\mathcal{I}_k}
\left\{Y_{2i}Y_{3i}-\mhat_{23}\supnk(\bUvec_i)-
\eta_{23}\right\} =&0& .
\end{alignedat}
\end{equation}

Finally, we impute $Y_2$, $Y_3$, $Y_{2}^2$ and $Y_{2}Y_3$ respectively as
$\muhat_2(\bUvec) = K^{-1}\sum_{k=1}^K \mhat_2\supnk(\bUvec) + \bEtahat_2\trans\bXvec$, 
$\muhat_3(\bUvec) = K^{-1}\sum_{k=1}^K \mhat\supnk_3(\bUvec) + \bEtahat_3\trans\bX_2$, 
$\muhat_{22}(\bUvec)=K^{-1}\sum_{k=1}^K \mhat\supnk_{22}(\bUvec) + \etahat_{22}$, and $\muhat_{23}(\bUvec)=K^{-1}\sum_{k=1}^K \mhat\supnk_{23}(\bUvec) + \etahat_{23}$.

\subsubsection*{Step III: Projection}\label{section: SSQL}

In the last step, we proceed to estimate $\bthetahat$ 
by replacing $\{Y_t, Y_{2}Y_{t}, t=2,3\}$ in (\ref{full_EE}) with their the imputed values $\{ \muhat_t(\bUvec),  \muhat_{2t}(\bUvec), t= 2,3\}$ and project to the unlabeled data. Specifically, we obtain the final SSL estimators for $\btheta_1$ and $\btheta_2$ via the following steps:
\begin{enumerate} 
	\item Stage 2 regression:
	we obtain the SSL estimator for $\btheta_2$ as 
    \begin{align*}
    \bthetahat_2=(\bbetahat_2\trans,\bgammahat_2\trans)\trans: 
    \mbox{the solution to}\quad
    \begin{split}
    \Pbb_N&
    \begin{bmatrix}
    \muhat_{23}(\bUvec)-
    [\muhat_{22}(\bUvec),\muhat_2(\bUvec)\bX_2\trans]\btheta_2\\
        \bX_2\{\muhat_3(\bUvec)-
        [\muhat_2(\bUvec),\bX_2\trans]\btheta_2\}
        \end{bmatrix}
        =\textbf{0}
    \end{split}
    \end{align*}
    
\item We compute the imputed pseudo-outcome: 
\[
\Ytilde_{2}^*=\muhat_{2}(\bUvec)+\underset{a\in\{0,1\}}{\text{max }}Q_2\left(\bH_{2},\muhat_{2}(\bUvec),a;\bthetahat_2\right),
\]
\item Stage 1 regression: we estimate $\bthetahat_1=(\bbetahat_1\trans,\bgammahat_1\trans)\trans$ as the solution to:
\begin{align*}
\Pbb_N \left\{\bX_1
(\Ytilde_2^*-\bX_1\trans\btheta_1)\right\}=\bf0.
\end{align*}

\end{enumerate}

Based on the SSL estimator for the Q-learning model parameters, we can then obtain an estimate for the optimal treatment protocol as:
\[
\dhat_t \equiv \dhat_t(\bH_t)\equiv d_t(\bH_t; \bthetahat_t), \mbox{ where } d_t(\bH_t,\btheta_t) = \argmax{a \in \{0,1\}}Q_t(\bH_t,a;\btheta_t)=I\left(\bH_{t1}\trans\bgamma_t>0\right), \: t = 1, 2.
\]
Theorems \ref{theorem: unbiased theta2} and \ref{theorem: unbiased theta1} of Section \ref{theory} demonstrate
the consistency and asymptotic normality of the SSL estimators $\{\bthetahat_t,t=1,2\}$ for their respective population parameters $\{\bthetabar_t,t=1,2 \}$ even in the possible mis-specification of \eqref{linear_Qs}. As we explain next, this in turn yields desirable statistical results for evaluating the resulting policy $\dbar_t \equiv \dbar_t(\bH_t) \equiv d_t(\bH_t,\bthetabar_t) = \argmax{a \in \{0,1\}}Q_t(\bHcheck_t,a;\bthetabar_t)$ for $t=1,2$.

\section{Semi Supervised Off-Policy Evaluation of the Policy}\label{section: SS value function}

To evaluate the performance of the optimal policy $\Dscbar = \{\dbar_t(\bH_t), t=1,2\}$, derived under the Q-learning framework, one may estimate the expected population outcome under the policy $\Dscbar$: 
\[
\Vbar\equiv
\mathbb{E}\left[\mathbb{E}\{Y_2+\mathbb{E}\{Y_3|\bHcheck_2,A_2=\dbar_2(\bH_2)\}|\bH_1,A_1=\dbar_1(\bH_1)\}\right].
\]

If models in \eqref{linear_Qs} are correctly specified, then under standard causal assumptions (consistency, no unmeasured confounding, and positivity), an asymptotically consistent supervised estimator for the value function can be obtained as
\[
\Vhat_Q=\Pbb_n\left[\Qopt_1(\bHcheck_1;\bthetahat_1)\right],
\]
where  $\Qopt_t(\bHcheck_t;\btheta_t)\equiv Q_t\left(\bHcheck_t,d_t(\bH_t;\btheta_t);\btheta_t\right)$. However, $\Vhat_Q$ is likely to be biased when the outcome models in \eqref{linear_Qs} are mis-specified. This occurs frequently in practice since $Q_1(\bHcheck_1,A_1)$ is especially difficult to specify.

To improve the robustness to model mis-specification, we augment $\Vhat_Q$ via propensity score weighting. This gives us an SSL doubly robust (SSL$\subDR$) estimator for $\Vbar$. To this end, we define propensity scores: $$\pi_t(\bHcheck_t)=\Pbb\{A_t=1|\bHcheck_t\}, \quad t=1,2.$$ To estimate $\{\pi_t(\cdot),t=1,2\}$, we impose the following generalized linear models (GLM):
\begin{align}\label{logit_Ws}
\pi_t(\bHcheck_t;\bxi_t)=&\sigma\left(\bHcheck_t\trans\bxi_t\right), \quad \mbox{with}\quad \sigma(x)\equiv1/(1+e^{-x}) \quad \mbox{for}\quad t=1,2.
\end{align}
We use the logistic model with potentially non-linear basis functions $\bHcheck$ for simplicity of presentation but one may choose other GLM or alternative basis expansions to incorporate non-linear effects in the propensity model. We estimate $\bxi=(\bxi_1\trans,\bxi_2\trans)\trans$ based on the standard maximum likelihood estimators using labeled data, denoted by $\bxihat = (\bxihat_1\trans,\bxihat_2\trans)\trans$.
We denote the limit of $\bxihat$ as $\bxibar = (\bxibar_1\trans,\bxibar_2\trans)\trans$. Note that this is not necessarily equal to the true model parameter under correct specification of \eqref{logit_Ws}, but corresponds to the population solution of the fitted models. 

Our framework is flexible to allow an SSL approach to estimate the propensity scores. As these are nuisance parameters needed for estimation of the value function, and SSL for GLMs has been widely explored \citep[See][Ch. 2]{ChakraborttyThesis}, we proceed with the usual GLM estimation to keep the discussion focused. However, SSL for propensity scores can be beneficial in certain cases, as we show in Proposition \ref{lemma: v funcion var}.

\subsection{SUP\texorpdfstring{$\subDR$}{Lg} Value Function Estimation }
To derive a supervised doubly robust (SUP$\subDR$) estimator for $\Vbar$ overcoming confounding in the observed data, we let $\bTheta = (\btheta\trans,\bxi\trans)\trans$ and define the inverse probability weights (IPW) using the propensity scores
\begin{align*}
\omega_1(\bHcheck_1,A_1,\bTheta)
&\equiv
\frac{d_1(\bH_1;\btheta_1)A_1}{\pi_1(\bHcheck_1;\bxi_1)}+\frac{\{1-d_1(\bH_1;\btheta_1)\}\{1-A_1\}}{1-\pi_1(\bHcheck_1;\bxi_1)}, \quad \mbox{and}\\ \omega_2(\bHcheck_2,A_2,\bTheta)
&\equiv
\omega_1(\bHcheck_1,A_1,\bTheta)\left(\frac{d_2(\bH_2;\btheta_2)A_2}{\pi_2(\bHcheck_2;\bxi_2)}+\frac{\{1-d_2(\bH_2;\btheta_2)\}\{1-A_2\}}{1-\pi_2(\bHcheck_2;\bxi_2)}\right).
\end{align*}
Then we augment  $\Qopt_1(\bH_1;\bthetahat_1)$ based on the estimated propensity scores via
\begin{align*}
\begin{split}
\Vsc\subSUPDR(\bL; \bThetahat)=
\Qopt_1(\bH_1;\bthetahat_1)
+&\omega_1(\bHcheck_1,A_1,\bThetahat)
\left[
Y_2-\left\{\Qopt_1(\bH_1, \bthetahat_1)- \Qopt_2(\bHcheck_2;\bthetahat_2)
\right\}\right]\\
+&\omega_2(\bHcheck_2,A_2,\bThetahat)\left\{
Y_3-\Qopt_2(\bHcheck_2; \bthetahat_2)
\right\}
\end{split}
\end{align*}
and estimate $\Vbar$ as 
\begin{align}\label{eq: lab value fun}
\Vhat\subSUPDR = \Pbb_n\left\{\Vsc\subSUPDR(\bL; \bThetahat)\right\} .
\end{align}

\begin{remark}
The importance sampling estimators previously proposed in \citet{Jiang2016} and \citet{WDR} for value function estimation employ similar augmentation strategies. However, they consider a fixed policy, and we account for the fact that the STR is estimated with the same data. The construction of augmentation in $\Vhat\subSUPDR$ also differs from the usual augmented IPW estimators \citep{DTRbook}. As we are interested in the value had the population been treated with function $\Dscbar$ and not a fixed sequence $(A_1,A_2)$, we augment the weights for a fixed treatment (i.e. $A_t=1$) with the propensity score weights for the estimated regime $I(A_t = \dbar_t)$. Finally, we note that this estimator can easily be extended to incorporate non-binary treatments.
\end{remark}

The supervised value function estimator $\Vhat\subSUPDR$ is doubly robust in the sense that if either the outcome models of the propensity score models are correctly specified, then $\Vhat\subSUPDR\stackrel{\Pbb}{\rightarrow} \Vbar$ in probability. Moreover, under certain reasonable assumptions, $\Vhat\subSUPDR$ 
is asymptotically normal. Theoretical guarantees and proofs for this procedure are shown in Appendix \ref{app_DR_Vfun}.

\subsection{SSL\texorpdfstring{$\subDR$}{Lg} Value Function Estimation}\label{section: SS value function est}
Analogous to semi-supervised $Q$-learning, we propose a procedure for adapting the augmented value function estimator to leverage $\mathcal{U}$, by imputing suitable functions of the unobserved outcome in \eqref{eq: lab value fun}. Since $\bHcheck_2$ involves $Y_2$, both $\omega_2(\bHcheck_2,A_2;\bTheta)$ and $\Qopt_2(\bHcheck_2;\btheta_2)=Y_2 \beta_{21}+\Qopt_{2-}(\bH_2;\btheta_2)$ are not available in the unlabeled set, where $\Qopt_{2-}(\bH_2;\btheta_2)=\bH_{20}\trans\bbeta_{22} + [\bH_{21}\trans\bgamma_2]_+$. By writing $\Vsc\subSUPDR(\bL; \bThetahat)$ as
\begin{align*}
\begin{split}
\Vsc\subSUPDR(\bL; \bThetahat)=
\Qopt_1(\bH_1;\bthetahat_1)
+&\omega_1(\bHcheck_1,A_1,\bThetahat)
\left\{
(1+\betahat_{21})Y_2-\Qopt_1(\bH_1, \bthetahat_1)+ \Qopt_{2-}(\bH_2;\bthetahat_2)
\right\}\\
+&\omega_2(\bHcheck_2,A_2,\bThetahat)\left\{
Y_3-\betahat_{21}Y_2 -\Qopt_{2-}(\bH_2; \bthetahat_2)
\right\},
\end{split}
\end{align*}
we note that to impute $\Vsc\subSUPDR(\bL; \bThetahat)$ for subjects in $\Usc$,  we need to impute $Y_2$, $\omega_2(\bHcheck_2,A_2; \bThetahat)$, and $Y_t \omega_2(\bHcheck_2,A_2; \bThetahat)$ for $t=2,3$. We define the conditional mean functions \[\mu^v_2(\bUvec)\equiv\mathbb{E}[Y_2|\bUvec],  \quad \mu^v_{\omega_2}(\bUvec)\equiv\mathbb{E}[\omega_2(\bHcheck_2,A_2; \bThetabar)|\bUvec],\quad \mu^v_{t\omega_2}(\bUvec)\equiv\mathbb{E}[Y_t\omega_2(\bHcheck_2,A_2; \bThetabar)|\bUvec], 
\]
for $t=2,3$, where
$\bThetabar = (\bthetabar\trans,\bxibar\trans)\trans$.
As in Section \ref{sec: ssQ} we approximate these expectations using a flexible imputation model followed by a refitting step for bias correction under possible mis-specification of the imputation models.

\subsubsection*{Step I: Imputation}
We fit flexible weakly parametric or non-parametric models to the labeled data to approximate the functions $\{\mu^v_2(\bUvec), \mu^v_{\omega_2}(\bUvec), \mu^v_{t\omega_2}(\bUvec),\:t=2,3\}$ with unknown parameter $\bTheta$ estimated via the SSL $Q$-learning as in Section \ref{sec: ssQ} and the propensity score modeling as discussed above. Denote the respective imputation models as $\{ m_2(\bUvec), 
m_{\omega_2}(\bUvec),m_{t\omega_2}(\bUvec),\:t=2,3\}$ and their fitted values as
$\{\mhat_2(\bUvec), 
\mhat_{\omega_2}(\bUvec),\mhat_{t\omega_2}(\bUvec),\:t=2,3\}$. 

\subsubsection*{Step II: Refitting}
To correct for potential biases arising from finite sample estimation and model mis-specifications, we perform refitting to obtain final imputed models for $\{Y_2,\omega_2(\bHcheck_2, A_2; \bThetabar),$ $Y_t\omega_2(\bHcheck_2, A_2; \bThetabar),$ $t=2,3\}$ as $\{\mubar^v_2(\bUvec)=m_2(\bUvec)+\eta_2^v, \mubar^v_{\omega_2}(\bUvec)=m_{\omega_2}(\bUvec)+\eta_{\omega_2}^v, \mubar^v_{t\omega_2}(\bUvec)=m_{t\omega_2}(\bUvec)+\eta_{t\omega_2}^v,\:t=2,3\}$. As for the estimation of $\btheta$ for $Q$-learning training, these refitted models are not required to be correctly specified but need to satisfy the following constraints:
\begin{equation*}
\begin{alignedat}{3}
\Ebb\left[\omega_1(\bHcheck_1,A_1;\bThetabar)\left\{Y_2-\mubar_2^v(\bUvec)\right\}\right] &=0, \\
\Ebb\left[\Qopt_{2-}(\bUvec;\btheta_2)\left\{\omega_2(\bHcheck_2,A_2;\bThetabar)-\mubar_{\omega_2}^v(\bUvec)\right\}\right] &= 0,\\
\Ebb\left[\omega_2(\bHcheck_2,A_2;\bThetabar)Y_t- \mubar^v_{t\omega_2}(\bUvec)\right] &= 0,\:t=2,3.
\end{alignedat}
\end{equation*}
To estimate $\eta_2^v$ $\eta_{\omega_2}^v$, and $\eta_{t\omega_2}^v$ under these constraints, we again employ cross-fitting and obtain $\etahat_2^v$ $\etahat_{\omega_2}^v$, and $\etahat_{t\omega_2}^v$ as the solution to the following estimating equations
\begin{equation}\label{V function reffitting}
\begin{alignedat}{3}
\sum_{k=1}^{K}\sum_{i\in\mathcal{I}_k}\omega_1(\bHcheck_{1i},A_{1i};\bThetahat)\left\{Y_2-\mhat_2\supnk(\bUvec_i)-
\etahat_2^v\right\} &=0, \\
\sum_{k=1}^{K}\sum_{i\in\mathcal{I}_k}{ \Qopt_{2-}(\bUvec_i;\bthetahat_2)}\left\{\omega_2(\bHcheck_{2i},A_{2i};\bThetahat)-\mhat_{\omega_2}\supnk(\bUvec_i)-
\etahat_{\omega_2}^v\right\} &=0, \\
\sum_{k=1}^{K}\sum_{i\in\mathcal{I}_k}\left\{\omega_2(\bHcheck_{2i},A_{2i};\bThetahat)Y_{ti}-\mhat_{t\omega_2}\supnk(\bUvec_i)-
\etahat_{t\omega_2}^v\right\} &=0,  \:t=2,3.
\end{alignedat}
\end{equation}

The resulting imputation functions for $Y_2,\omega_2(\bHcheck_2, A_2; \bThetabar)$ and $Y_t\omega_2(\bHcheck_2, A_2; \bThetabar)$  are respectively constructed as $\muhat^v_2(\bUvec)= K^{-1}\sum_{k=1}^{K}
\mhat_2\supnk(\bUvec)+
\etahat_2^v,$ 
$\muhat^v_{\omega_2}(\bUvec)= K^{-1}\sum_{k=1}^{K}
\mhat_{\omega_2}(\bUvec)+
\etahat_{\omega_2}^v,$ and 
$\muhat^v_{t\omega_2}(\bUvec)= K^{-1}\sum_{k=1}^{K}
\mhat_{t\omega_2}\supnk(\bUvec)+
\etahat_{t\omega_2}^v,$ for $t=2,3$.

\subsubsection*{Step III: Semi-supervised augmented value function estimator.}

Finally, we proceed to estimate the value of the policy $\Vbar$, using the following semi-supervised augmented estimator:
	\begin{align}\label{SS_value_fun}
	\Vhat\subSSLDR=\Pbb_N\left\{
	\Vsc\subSSLDR(\bUvec;\bThetahat,\muhat)\right\},
	\end{align}
	 where $\Vschat\subSSLDR(\bUvec)$ is the semi-supervised augmented estimator for observation $\bUvec$ defined as:

	\begin{align*}
	\begin{split}
	\Vsc\subSSLDR(\bUvec;\bThetahat,\muhat)
	=&
	\Qopt_1(\bHcheck_1;\bthetahat_1)
	+\omega_1(\bHcheck_1,A_1,\bThetahat)
	\left[(1+\betahat_{21})\muhat_2^v(\bUvec)-
	 \Qopt_1(\bHcheck_1;\bthetahat_1)+\Qopt_{2-}(\bH_2;\bthetahat_2)
	\right]\\
	+&
	\muhat_{3\omega_2}(\bUvec)-\betahat_{21}\muhat_{2\omega_2}(\bUvec)-\Qopt_{2-}(\bH_2;\bthetahat_2)\muhat_{\omega_2}(\bUvec) .
	\end{split}
	\end{align*}

The above SSL estimator uses both labeled and unlabeled data along with outcome surrogates to estimate the value function, which yields a gain in efficiency as we show in Proposition \ref{lemma: v funcion var}. As its supervised counterpart, $\Vhat\subSSLDR$ is doubly robust in the sense that if either the $Q$ functions or the propensity scores are correctly specified, the value function will converge in probability to the true value $\Vbar$. Additionally, it does not assume that the estimated treatment regime was derived from a different sample. These properties are summarized in Theorem \ref{thrm_ssV_fun} and Proposition \ref{cor_dr_V} of the following section.

\section{Theoretical Results}\label{theory}

In this section we discuss our assumptions and theoretical results for the semi-supervised $Q$-learning and value function estimators. Throughout, we define the norm $\|g(x)\|_{L_2(\Pbb)}\equiv\sqrt{\int g(x)^2d\Pbb(x)}$ for any real valued function $g(\cdot)$. Additionally, let $\{U_n\}$, and $\{V_n\}$ be two sequences of random variables. We will use $U_n=O_\Pbb(V_n)$ to denote stochastic boundedness of the sequence $\{U_n/V_n\}$, that is, for any $\epsilon>0$, $\exists M_\epsilon,n_\epsilon\in\mathbb{R}$ such that $\Pbb\left(|U_n/V_n|>M_\epsilon\right)<\epsilon$ $\forall n>n_\epsilon$. We use $U_n=o_\Pbb(V_n)$ to denote that $U_n/V_n\stackrel{\Pbb}{\rightarrow}0.$

\subsection{Theoretical Results for SSL Q-learning}

\begin{assumption}\label{assumption: covariates}
	(a) Sample size for $\mathcal{U}$, and $\mathcal{L}$, are such that $n/N\longrightarrow 0$ as $N,n\longrightarrow\infty$, (b) $\bHcheck_t\in\mathcal{H}_t$, $\bXcheck_t\in\mathcal{X}_t$ have finite second moments and compact support in $\mathcal{H}_t\subset\mathbb{R}^{q_t}$, $\mathcal{X}_t\subset\mathbb{R}^{p_t}$ $t=1,2$ respectively (c) $\bSigma_1,\:\bSigma_2$ are nonsingular.
\end{assumption}

\begin{assumption}\label{assumption: Q imputation}
	Functions $m_s$, $s\in\{2,3,22,23\}$ are such that (i) $\sup_{\bUvec}|m_s(\bUvec)|<\infty$, and (ii) the estimated functions $\hat m_s$ satisfy (ii) $\sup_{\bUvec}|\mhat_s(\bUvec)-m_s(\bUvec)|=o_\Pbb(1)$. 
\end{assumption}

\begin{assumption}\label{assumption SS linear model}
Suppose $\Theta_1,\Theta_2$ are open bounded sets, and $p_1,p_2$ fixed under \eqref{linear_Qs}.
We define the following class of functions:
\begin{align*}
\begin{split}
\mathcal{Q}_t&\equiv\left\{Q_t:\mathcal{X}_1\mapsto\mathbb{R}|\btheta_1\in\Theta_1\subset\mathbb{R}^{p_t}\right\},\:t=1,2.
\end{split}
\end{align*}
Further suppose for $t=1,2$, the solutions for  $\mathbb{E}[S^\theta_t(\btheta_t)]=\bzero,$ i.e. $\bthetabar_1$ and $ \bthetabar_2$ satisfy

\be
S^\theta_2(\btheta_2)=\frac{\partial}{\partial\btheta_2\trans}
\|Y_3-Q_2(\bXcheck_2;\btheta_2)\|_2^2,\:
S^\theta_1(\btheta_1)=\frac{\partial}{\partial\btheta_1\trans}
\|Y_2^*-Q_1(\bXcheck_1;\btheta_1)\|_2^2.
\ee

The target parameters satisfy $\bthetabar_t\in\Theta_t\:,t=1,2$. We write $\bbetabar_t,\bgammabar_t$ as the components of $\bthetabar_t$, according to equation \eqref{full_EE}.
\end{assumption}
Assumption \ref{assumption: covariates} (a) distinguishes our setting from the standard missing data context. Theoretical results for the missing completely at random (MCAR) setting generally assume that the missingness probability is bounded away from zero \citep{tsiatis2006}, which enables the use of standard semiparametric theory. However, in our setting one can intuitively consider the probability of observing an outcome being $\frac{n}{n+N}$ which converges to $0$. 

Assumption \ref{assumption: Q imputation} is fairly standard as it just requires boundedness of the imputation functions -- which is natural to expect from the boundedness of the covariates. We also require uniform convergence of the estimated functions to their limit. This allows for the normal equations targeting the imputation residuals in \eqref{eta_EE_Q1} and \eqref{V function reffitting} to be well defined. Moreover, several off-the-shelf flexible imputation models for estimation can satisfy these conditions. See for example, local polynomial estimators, basis expansion regression like natural cubic splines or wavelets \citep{tsybakov2009}. In particular, it is worth noting that we do not require any specific rate of convergence. As a result, the required condition is typically much easier to verify for many off-the-shelf algorithms. It is likely that other classes of models such as random forests can satisfy Assumption \ref{assumption: Q imputation}. Recent work suggests that it is plausible to use the existing point-wise convergence results to show uniform convergence. \citep[see][]{scornet2015,biau2008}. 

Assumption \ref{assumption SS linear model} is fairly standard in the literature and ensures well-defined population level solutions for $Q$-learning regressions $\bthetabar$ exist, and belong to that parameter space. In this regard, we differentiate between population solutions $\bthetabar$ and true model parameters $\btheta^0$ shown in equation \eqref{linear_Qs}. If the working models are mis-specified, Theorems \ref{theorem: unbiased theta2} and \ref{theorem: unbiased theta1} still guarantee the $\bthetahat$ is consistent and asymptotically normal centered at the population solution $\bthetabar$. However, when equation \eqref{linear_Qs} is correct, $\bthetahat$ is asymptotically normal and consistent for the true parameter $\btheta^0$. Now we are ready to state the theoretical properties of the semi-supervised $Q$-learning procedure described in Section \ref{sec: ssQ}. 

\begin{theorem}[Distribution of $\bthetahat_2$]\label{theorem: unbiased theta2} 
Under Assumptions \ref{assumption: covariates}-\ref{assumption SS linear model}, $\bthetahat_2$ satisfies
\[
\sqrt n
(\bthetahat_2-\bthetabar_2)
=
\bSigma_2^{-1}\frac{1}{\sqrt n}\sum_{i=1}^n
\bpsi_2(\bL_i;\bthetabar_2)
+o_\Pbb\left(1\right)\stackrel{d}{\rightarrow}\mathcal{N}\bigg({\bf 0},\bV_{2 \scriptscriptstyle \sf SSL}(\bthetabar_2)\bigg),
\]

where $\bSigma_2=\Ebb[\bXcheck_2\bXcheck_2\trans]$ is defined in Section \ref{section: set up}, the influence function $\bpsi_2$ is given by
\[
\bpsi_2(\bL;\bthetabar_2)
=
\begin{bmatrix}
\{Y_2Y_3-\mubar_{23}(\bUvec)\}-\bar\beta_{21}
\{Y_2^2-\mubar_{22}(\bUvec)\}-
Q_{2-}(\bH_2,A_2;\bthetabar_2)
\{Y_2-\mubar_2(\bUvec)\}\\
\bX_2
\{Y_3-\mubar_3(\bUvec)\}-\bar\beta_{21}
\bX_2
\{Y_2-\mubar_2(\bUvec)\}
\end{bmatrix},
\]
and $
\bV_{2 \scriptscriptstyle \sf SSL}(\bthetabar_2)=\bSigma_2^{-1}\Ebb\left[\bpsi_2(\bL;\bthetabar_2)\bpsi_2(\bL;\bthetabar_2)\trans\right]\left(\bSigma_2^{-1}\right)\trans$.  \end{theorem}

We hold off remarks until the end of the results for the $Q$-learning parameters. Since the first stage regression depends on the second stage regression through a non-smooth maximum function, we make the following standard assumption \citep{laber2014} in order to provide valid statistical inference. 
\begin{assumption}\label{assumption: non-regularity} Non-zero estimable population treatment effects $\bgammabar_t$, $t=1,2$: i.e. the population solution to \eqref{full_EE}, is such that (a) $\bH_{21}\trans\bgammabar_2\neq0$ for all $\bH_{21}\neq\bzero$, and (b) $\bgammabar_1$ is such that $\bH_{11}\trans\bgammabar_1\neq0$ for all $ \bH_{11}\neq\bzero$. 
\end{assumption}

Assumption \ref{assumption: non-regularity} yields regular estimators for the stage one regression and the value function, which depend on non-smooth components of the form $[x]_+$. This is needed to achieve asymptotic normality of the $Q$-learning parameters for the first stage regression. Note that the estimating equation for the stage one regression in Section \ref{section: SSQL} includes $\left[\bH_{21}\trans\bgammahat_2\right]_+$. Thus, for the asymptotic normality of $\bthetahat_1$, we require $\sqrt n\Pbb_n\left(\left[\bH_{21}\trans\bgammahat_2\right]_+-\left[\bH_{21}\trans\bar{\bgamma}_2\right]_+\right)$ to be asymptotically normal. 
The latter is automatically true if $\bH_{11}$ contains continuous covariates as $\Pbb\left(\bH_{21}\trans\bgammabar_2=0\right)=0$. Violation of Assumption \ref{assumption: non-regularity} will yield non-regular estimates which translate into poor coverage for the confidence intervals (see \citet{laber2014} for a thorough discussion on this topic).

\begin{theorem}[Distribution of $\hat{\btheta}_{1}$]\label{theorem: unbiased theta1} 
	Under Assumptions \ref{assumption: covariates}-\ref{assumption SS linear model}, and \ref{assumption: non-regularity} (a), $\bthetahat_1$ satisfies
	
	\[
	\sqrt n(\hat{\btheta}_1-\bthetabar_1)=\bSigma_1^{-1}\frac{1}{\sqrt n}\sum_{i=1}^n\bpsi_1(\bL_i;\bthetabar_1)+o_\Pbb(1)\stackrel{d}{\rightarrow}\mathcal{N}\bigg({\bf 0},\bV_{1 \scriptscriptstyle \sf SSL}(\bthetabar_1)\bigg)
	\]
	where $\bSigma_1^{-1}=\Ebb[\bXcheck_1\bXcheck_1\trans]$, the influence function $\bpsi_1$ is given by
	\begin{align*}
	\bpsi_1(\bL;\bthetabar_1)=
	&\bX_1(1+\bar\beta_{21})\{Y_2-\mubar_2(\bUvec)\}
	+
	\mathbb{E}\left[\bX_1
	\left(Y_2,\bH_{20}\trans\right)\right]
	\bpsi_{\beta_2}(\bL;\bthetabar_2)\\
	+&
	\mathbb{E}\left[\bX_1\bH_{21}\trans|\bH_{21}\trans\bgammabar_2>0\right]\Pbb\left(\bH_{21}\trans\bgammabar_2>0\right)
	\bpsi_{\gamma_2}(\bL;\bthetabar_2
	),
	\end{align*} 
	$\bV_{1 \scriptscriptstyle \sf SSL}(\bthetabar_1)=\bSigma_1^{-1}\Ebb\left[\bpsi_1(\bL;\bthetabar_1)\bpsi_1(\bL;\bthetabar_1)\trans\right]\left(\bSigma_1^{-1}\right)\trans$, and $\bpsi_{\beta_2}$, $\bpsi_{\gamma_2}$ are the elements corresponding to $\bbetabar_2$, $\bgammabar_2$ of the influence function $\bpsi_2$ defined in Theorem \ref{theorem: unbiased theta2}.
\end{theorem}

\begin{remark}
1) Theorems \ref{theorem: unbiased theta2} and \ref{theorem: unbiased theta1} establish the $\sqrt n$-consistency and asymptotic normality (CAN) of $\bthetahat_1,\bthetahat_2$ for any $K\geq2$. Beyond asymptotic normality at $\sqrt n$ scale, these theorems also provide an asymptotic linear expansion of the estimators with influence functions $\bpsi_1$ and $\bpsi_2$ respectively.  

2) $\bV_{1 \scriptscriptstyle \sf SSL}(\bthetabar)$, $\bV_{2 \scriptscriptstyle \sf SSL}(\bthetabar)$ reflect an efficiency gain over the fully supervised approach due to sample $\mathcal{U}$ and the surrogates contribution in prediction performance. This gain is formalized in Proposition \ref{lemma: Q parameter var} which quantifies how correlation between surrogates and outcome increases efficiency.\\ 
3) Let $\bpsi=[\bpsi_1\trans,\bpsi_2\trans]\trans$, we collect the vector of estimated $Q$-learning parameters $\btheta=(\btheta_1\trans,\btheta_2\trans)\trans$, then under Assumptions \ref{assumption: covariates}-\ref{assumption SS linear model}, \ref{assumption: non-regularity} (a), we have
    \begin{align*}
	\sqrt n
	(\bthetahat-\bthetabar)
	=&
	\bSigma^{-1}\frac{1}{\sqrt n}\sum_{i=1}^n
	\bpsi(\bL_i;\bthetabar)
	+o_\Pbb(1)\stackrel{d}{\rightarrow}\mathcal{N}\bigg({\bf 0},\bV\subSSL\left(\bthetabar\right)\bigg)
\end{align*}
with $\bV\subSSL(\bthetabar)=\bSigma^{-1}\Ebb\left[\bpsi(\bL;\bthetabar)\bpsi(\bL;\bthetabar)\trans\right]\left(\bSigma^{-1}\right)\trans$.\\
4) Theorems \ref{theorem: unbiased theta2} and \ref{theorem: unbiased theta1} hold even when the $Q$ functions are mis-specified, that is, $\bthetahat_1,\bthetahat_2$ are CAN for $\bthetabar_1,\bthetabar_2$. Furthermore, if model \eqref{linear_Qs} is correctly specified then we can simply replace $\bthetabar$ with $\btheta^0$ in the above result.\\
3) We estimate $\bV\subSSL(\bthetabar)$ via sample-splitting as
\begin{align*}
\widehat\bV\subSSL(\bthetahat)&=\widehat\bSigma^{-1}\widehat\bA(\btheta)\left(\widehat\bSigma^{-1}\right)\trans,\text{ where}\\
\widehat\bA(\bthetahat)&=
n^{-1}\sum_{k=1}^{K}\sum_{i\in\mathcal{I}_k}
\bpsi\supnk\left(\bL_i;\bthetahat\right)\bpsi\supnk\left(\bL_i;\bthetahat\right)\trans,\\
\widehat\bSigma_t&=
\Pbb_n
\left\{\bX_t\bX_t\trans\right\},\:\:t=1,2.
\end{align*}
\end{remark}

Note that we can decompose $\bpsi$ into the influence function for each set of parameters. For example, we have $\bpsi_2=\left(\bpsi_{\beta_2}\trans,\bpsi_{\gamma_2}\trans\right)\trans$  where 
$
\bpsi_{\gamma_2}(\bL;\bthetabar_2)=\bH_{21}A_2\left[\{Y_3-\mubar_3(\bUvec)\}-\bar\beta_{21}\{Y_2-\mubar_2(\bUvec)\}\right].
$
Therefore we can decompose the variance-covariance matrix into a component for each parameter, the variance-covariance for the treatment effect for stage 2 regression $\bgamma_2$ is \[
\mathbb{E}\left[\bpsi_{\gamma_2}(\bL;\bthetabar_2)\bpsi_{\gamma_2}(\bL;\bthetabar_2)\trans\right]=
\mathbb{E}\left[\bH_{21}\bH_{21}\trans A_{2}^2\left\{
    Y_3-\mubar_3(\bUvec)-\beta_{21}
    \left(Y_2-\mubar_2(\bUvec)\right)\right\}^2\right].
\]
This gives us some insight into how the predictive power of $\bUvec$, which contains surrogates $\bW_1,\bW_2$, decreases parameter standard errors. This is the case for the influence functions for estimating $\bthetabar_1$, $\bthetabar_2$ as well. We formalize this result with the following proposition. Let $\bthetahat_{\subSUP}$ be the estimator for the fully supervised $Q$-learning procedure (i.e. only using labeled data), with influence function and asymptotic variance denoted as $\bpsi\subSUP$ and $\bV\subSUP$ respectively (see Appendix \ref{section: Q learning result proofs} for the exact form of $\bpsi\subSUP$ and $\bV\subSUP$). 

For the following proposition we need the imputation models $\mubar_s$, $s\in\{2,3,22,23\}$ to satisfy additional constraints of the form $\Ebb\left[\bX_2\bX_2\trans\{Y_2Y_3-\mubar_{23}(\bUvec)\}\right]=\bzero$. We list them in Assumption \ref{assumption: additional constraints}, Appendix \ref{section: Q learning result proofs}. One can construct estimators which satisfy such conditions by simply augmenting $\bEta_2,$ $\eta_{22},$ $\bEta_3,$ $\eta_{23}$ in \eqref{eta_EE_Q1} with additional terms in the refitting step.

\begin{proposition}\label{lemma: Q parameter var}
Under Assumptions \ref{assumption: covariates}-\ref{assumption SS linear model}, \ref{assumption: non-regularity} (a), and \ref{assumption: additional constraints}
then
\[
\bV\subSSL(\bthetabar)=\bV\subSUP(\bthetabar)-\bSigma^{-1}\text{Var}\left[\bpsi\subSUP(\bL;\bthetabar)-\bpsi\subSSL(\bL;\bthetabar)\right]\left(\bSigma^{-1}\right)\trans.
\]

\end{proposition}
\begin{remark}
Proposition \ref{lemma: Q parameter var} illustrates how the estimates for the semi-supervised $Q$-learning parameters are at least as efficient, if not more so, than the supervised ones. Intuitively, the difference in efficiency is explained by how much information is gained by incorporating the surrogates $\bW_1,\bW_2$ into the estimation procedure. If there is no new information in the surrogate variables, then residuals found in $\bpsi\subSSL(\bL;\btheta)$ will be of similar magnitude to those in $\bpsi\subSUP(\bL;\btheta)$, and thus the difference in efficiency will be small:  $\text{Var}\left[\bpsi\subSUP(\bL;\bthetabar)-\bpsi\subSSL(\bL;\bthetabar)\right]\approx0$. In this case both methods will yield equally efficient parameters. The gain in precision is especially relevant for the treatment interaction coefficients $\bgamma_1,\bgamma_2$ used to learn the dynamic treatment rules. Finally, note that for Proposition \ref{lemma: Q parameter var}, we do not need the correct specification of $Q$-functions or imputation models.
\end{remark}

\subsection{Theoretical Results for SSL Estimation of the Value Function}

If model \eqref{linear_Qs} is correct, one only needs to add Assumption \ref{assumption: non-regularity} (b) 
for $\Pbb_N\{\Qopt_1(\bH_1;\bthetahat_1)\}$ to be a consistent estimator of the value function $\Vbar$ \citep{zhu2019}. However, as we discussed earlier, \eqref{linear_Qs} is likely mis-specified. Therefore, we show our semi-supervised value function estimator is doubly robust. We also show it is asymptotically normal and more efficient that its supervised counterpart. To that end, define the following class of functions:
	\[
	\mathcal{W}_t\equiv\left\{\pi_t:\mathcal{H}_t\mapsto\mathbb{R}|\bxi_t\in\Omega_t\right\},\:t=1,2,
	\]
under propensity score models $\pi_1,\:\pi_2$ in \eqref{logit_Ws}.
\begin{assumption}\label{assumption: donsker w}
	
Let the population equations $\Ebb\left[S^\xi_t(\bHcheck_t;\bxi_t)\right]=\bzero, t=1,2$ have solutions $\bxibar_1,\bxibar_2$, where
\be
S^\xi_t(\bHcheck_t;\bxi_t)=&\frac{\partial}{\partial\bxi_t}\log \left[\pi_t(\bHcheck_t;\bxi_t)^{A_t}\{1-\pi_t(\bHcheck_t;\bxi_t)\}^{(1-A_t)}\right],\:t=1,2,
\ee
(i) $\Omega_1,\Omega_2$ are open, bounded sets and the population solutions satisfy $\bxibar_t\in\Omega_t,t=1,2$,\\
	(ii) for $\bxibar_t,t=1,2$, $\inf\limits_{\bHcheck_t\in\mathcal{H}_1}\pi_1(\bHcheck_t;\bxibar_t)>0$,\\
	(iii) Finite second moment: $\Ebb\left[S^\xi_t(\bHcheck_t;\bTheta_t)^2\right]\le\infty$, and Fisher information matrix: $\Ebb\left[\frac{\partial}{\partial\bxi_t}S^\xi_t(\bHcheck_t;\bTheta_t)\right]$ exists and is non singular,\\
	(iv) Second-order partial derivatives of $S^\xi_t(\bHcheck_t;\bTheta_t)$ with respect to $\bxi$ exist and for every $\bHcheck_t$, and satisfy $|\partial^2S^\xi_t(\bHcheck_t;\bTheta_t)/\partial\bxi_i\partial\bxi_j|\leq \tilde S_t(\bHcheck_t)$ for some integrable measurable function $\tilde S_t$ in a neighborhood of $\bxibar$.

\end{assumption}

\begin{assumption}\label{assumption: V imputation}
Functions $m_2,m_{\omega_2},m_{t\omega_2}$ $t=2,3$ are such that (i) $\sup_{\bUvec}|m_s(\bUvec)|<\infty$, and (ii) the estimated functions $\hat m_s$ satisfy (ii) $\sup_{\bUvec}|\mhat_s(\bUvec)-m_s(\bUvec)|=o_\Pbb(1)$, $s\in\{2,\omega_2,2\omega_2,3\omega_2\}$.
\end{assumption}
Assumption \ref{assumption: donsker w} is standard for Z-estimators \citep[see][Ch.~5.6]{vaart_donsker}. Assumption \ref{assumption: V imputation} is the propensity score equivalent version of Assumption \ref{assumption: Q imputation}. Finally, we use $\bpsi^\xi$ and and $\bpsi^\theta$ to denote the influence function for $\bxihat$, and $\bthetahat$ respectively. We are now ready to state our theoretical results for the value function estimator in equation \eqref{SS_value_fun}. The proof, and the exact form of $\bpsi^\xi$ can be found in Appendix \ref{appendix: value fun}.

\begin{theorem}[Asymptotic Normality for $\Vhat\subSSLDR$]\label{thrm_ssV_fun}
	Under Assumptions \ref{assumption: covariates}-\ref{assumption: V imputation}, $\Vhat\subSSLDR$ defined in \eqref{SS_value_fun} satisfies
	\[
	\sqrt n
	\left\{
	\Vhat\subSSLDR-\mathbb{E}_\mathbb{S}\left[\Vsc\subSSLDR(\bL;\bThetabar,\mubar)\right]
	\right\}
	=
	\frac{1}{\sqrt n}\sum_{i=1}^n\psi^{v}_{\subSSLDR}(\bL_i;\bThetabar)
	+o_\Pbb\left(1\right),
	\]

	where
	\[
	\frac{1}{\sqrt n}\sum_{i=1}^n\psi^{v}_{\subSSLDR}(\bL_i;\bThetabar)
	\stackrel{d}{\longrightarrow}\mathcal{N}\left(0,\sigma^2\subSSLDR\right).
	\]
	Here
	\begin{align*}
	\psi^{v}\subSSLDR(\bL;\bThetabar)
	=&
	\nu\subSSLDR(\bL;\bThetabar)
	+
	\bpsi^\theta(\bL)\trans
	\frac{\partial}{\partial \btheta}\int\Vsc\subSUPDR(\bL;\bTheta)d\Pbb_{\bL}\bigg|_{\bTheta=\bThetabar}\\
	&\hspace{2.1cm}+
	\bpsi^\xi(\bL)\trans
	\frac{\partial}{\partial \bxi}\int\Vsc\subSUPDR(\bL;\bTheta)d\Pbb_{\bL}\bigg|_{\bTheta=\bThetabar},\\
	\nu_{\subSSLDR}(\bL;\bThetabar)
	=&
	\omega_1(\bHcheck_1,A_1;\bThetabar_1)
	(1+\bar\beta_{21})
	\left\{
	Y_2-\mubar_2^v(\bUvec)
	\right\}
	+
	\omega_2(\bHcheck_2,A_2,\bThetabar_2)Y_3-
	\mubar_{3\omega_2}(\bUvec)\\
	-&
	\bar\beta_{21}\left\{\omega_2(\bHcheck_2,A_2,\bThetabar_2)Y_2-
	\mubar_{2\omega_2}(\bUvec)\right\}
	-
	\Qopt_{2-}(\bH_2; \bthetabar_2)
	\left\{
	\omega_2(\bHcheck_2,A_2,\bThetabar_2)-\mubar_{\omega_2}(\bUvec)
	\right\}
	,
	\end{align*}
	$\sigma\subSSLDR^2=\Ebb\left[\psi^{v}_{\subSSLDR}(\bL;\bThetabar)^2\right],$ and $\Vsc\subSUPDR(\bL;\bTheta)$ is as defined in \eqref{eq: lab value fun}.
\end{theorem}

\begin{proposition}[Double Robustness of $\Vhat\subSSLDR$ as an estimator of $\Vbar$]\label{cor_dr_V} (a) If either $\|Q_t(\bHcheck_t, A_t;  \bthetahat_t)-Q_t(\bHcheck_t,A_t)\|_{L_2(\Pbb)}\rightarrow 0$, or $\| \pi_t(\bHcheck_t; \bxihat_t)-\pi_t(\bHcheck_t)\|_{L_2(\Pbb)}\rightarrow 0$ for $t=1,2$, then under Assumptions \ref{assumption: covariates}-\ref{assumption: V imputation}, $\Vhat\subSSLDR$ satisfies

	\[
	\Vhat\subSSLDR\stackrel{\Pbb}{\longrightarrow}\Vbar.
	\]
	
	(b) If $\|Q_t(\bHcheck_t, A_t;  \bthetahat_t)-Q_t(\bHcheck_t,A_t)\|_{L_2(\Pbb)}\| \pi_t(\bHcheck_t; \bxihat_t)-\pi_t(\bHcheck_t)\|_{L_2(\Pbb)}=o_\Pbb\left(n^{-\frac{1}{2}}\right)$ for $t=1,2$, then under Assumptions \ref{assumption: covariates}-\ref{assumption: V imputation}, $\Vhat\subSSLDR$ satisfies

	\[
	\sqrt n\left(\Vhat\subSSLDR-\Vbar\right)\stackrel{d}{\longrightarrow}\mathcal{N}\left(0,\sigma^2\subSSLDR\right).
	\]
\end{proposition}

Next we define the supervised influence function for estimator $\Vhat\subSUPDR$. Let $\bpsi^\theta\subSUP$, be the influence function for the supervised estimator $\bthetahat\subSUP$ for model \eqref{linear_Qs}. The influence function for SUP$\subDR$ Value Function Estimation estimator \eqref{eq: lab value fun} and its variance is (see Theorem \ref{thrm_supV_fun} in Appendix \ref{app_DR_Vfun}):

\begin{align*}
	\psi^{v}\subSUPDR(\bL;\bThetabar)
	=&
	\Vsc\subSUPDR(\bL;\bThetabar)-\mathbb{E}_\mathbb{S}\left[\Vsc\subSUPDR(\bL;\bThetabar)\right]\\
	+&
	\bpsi\subSUP^\theta(\bL)\trans
	\frac{\partial}{\partial \btheta}\int\Vsc\subSUPDR(\bL;\bTheta)d\Pbb_{\bL}\bigg|_{\bTheta=\bThetabar}
	+
	\bpsi^\xi(\bL)\trans
	\frac{\partial}{\partial \bxi}\int\Vsc\subSUPDR(\bL;\bTheta)d\Pbb_{\bL}\bigg|_{\bTheta=\bThetabar},\\
	\sigma\subSUPDR^2=&\mathbb{E}\left[\psi^{v}_{\subSUPDR}(\bL;\bThetabar)^2\right].
\end{align*}

The flexibility of our SSL value function estimator $V\subSSLDR$, allows the use of either supervised or SSL approach for estimation of propensity score nuisance parameters $\bxi$. For SSL estimation, we can use an approach similar to Section \ref{sec: ssQ}, \citep[see][Ch. 2]{chakrabortty} for details. This can be beneficial in that we can then quantify the efficiency gain of $V\subSSLDR$ vs. $V\subSUPDR$ by comparing the asymptotic variances. In light of this, we assume SSL is used for $\bxi$ when estimating $V\subSSLDR$.

Before stating the result we discuss an additional requirement for the imputation models. As for Proposition \ref{lemma: Q parameter var}, models $\mubar^v_2(\bUvec),$ $\mubar^v_{\omega_2}(\bUvec),$ $\mubar^v_{t\omega_2}(\bUvec),$  $t=2,3$ need to satisfy a few additional constraints of the form 
\[
\Ebb\left[\omega_1(\bHcheck_1,A_1;\bThetabar_1)\Qopt_{2-}(\bH_2;\bthetabar_1)\{Y_2-\mubar^v_2(\bUvec)\}\right]=\bzero.
\]
As there are several constraints, we list them in Appendix \ref{appendix: value fun}, and condense them in Assumption \ref{assumption: value additional constraints}, Appendix \ref{appendix: value fun}. Again, one can construct estimators which satisfy such conditions by simply augmenting $\eta_2^v,$ $\eta_{\omega_2}^v,$ $\eta_{t\omega_2}^v,$ $t=2,3$ in \eqref{V function reffitting} with additional terms in the refitting step.

\begin{proposition}\label{lemma: v funcion var}
Under Assumptions \ref{assumption: covariates}-\ref{assumption: V imputation}, and \ref{assumption: value additional constraints}, asymptotic variances
$\sigma\subSSLDR^2$, $\sigma\subSUPDR^2$ satisfy 
\[
\sigma\subSSLDR^2=\sigma\subSUPDR^2-\text{Var}\left[\psi^{v}\subSUPDR(\bL;\bThetabar)-\psi^{v}\subSSLDR(\bL;\bThetabar)\right].
\]
\end{proposition}

\begin{remark}\label{remark: se for V}

1) Proposition \ref{cor_dr_V} illustrates how $\Vhat\subSSLDR$ is asymptotically unbiased if either the $Q$ functions or the propensity scores are correctly specified.\\
2) An immediate consequence of Proposition \ref{lemma: v funcion var} is that the semi-supervised estimator is at least as efficient (or more) as its supervised counterpart, that is $\text{Var}\left[\psi\subSSLDR(\bL;\bTheta)\right]\le\text{Var}\left[\psi\subSUPDR(\bL;\bTheta)\right]$. As with Proposition \ref{lemma: Q parameter var}, the difference in efficiency is explained by the information gain from incorporating surrogates.\\
3) To estimate standard errors for $V\subSSLDR(\bUvec;\bThetabar)$, we will approximate the derivatives of the expectation terms $\frac{\partial}{\partial \bTheta}\int\Vsc\subSUPDR(\bL;\bThetabar)d\Pbb_{\bL}$ using kernel smoothing to replace the indicator functions. In particular, let $\mathbb{K}_h(x)=\frac{1}{h}\sigma(x/h)$, $\sigma$ defined as in \eqref{logit_Ws}, we approximate $d_t(\bH_t,\btheta_2)=I(\bH_{t1}\trans\bgamma_t>0)$ with $\mathbb{K}_h(\bH_{t1}\trans\bgamma_t)$ $t=1,2$, and define the smoothed propensity score weights as
\begin{align*}
\tilde\omega_1(\bHcheck_1,A_1,\bTheta)
&\equiv
\frac{A_1\mathbb{K}_h(\bH_{11}\trans\bgamma_1)}{\pi_1(\bHcheck_1;\bxi_1)}
+
\frac{\left\{1-A_1\right\}\left\{1-\mathbb{K}_h(\bH_{11}\trans\bgamma_1)\right\}}{1-\pi_1(\bHcheck_1;\bxi_1)}, \quad \mbox{and}\\ \tilde\omega_2(\bHcheck_2,A_2,\bTheta)
&\equiv
\tilde\omega_1(\bHcheck_1,A_1,\bTheta)\left[\frac{A_2\mathbb{K}_h(\bH_{21}\trans\bgamma_2)}{\pi_2(\bHcheck_2;\bxi_2)}+
\frac{\left\{1-A_2\right\}\left\{1-\mathbb{K}_h(\bH_{21}\trans\bgamma_2)\right\}}{1-\pi_2(\bHcheck_2;\bxi_2)}\right].
\end{align*}

We simply replace the propensity score functions with these smooth versions in $\psi^{v}_{\subSSLDR}(\bL;\bThetabar)$, detail is given in Appendix \ref{variance estimation of Vss}. To estimate the variance we use the sample-split estimators:
\[
\hat\sigma^2\subSSLDR=
n^{-1}\sum_{k=1}^{K}\sum_{i\in\mathcal{I}_k}\psi^{v{\scriptscriptstyle (\text{-}k)}}\subSSLDR(\bUvec_i;\bThetahat)^2.
\]

\end{remark}

\section{Simulations and application to EHR data:}\label{section: sims and application}
We perform extensive simulations to evaluate the finite sample performance of our method. Additionally we apply our methods to an EHR study of treatment response for patients with inflammatory bowel disease to identify optimal treatment sequence. These data have treatment response outcomes available for a small subset of patients only. 

\subsection{Simulation results}
We compare our SSL Q-learning methods to fully supervised $Q$-learning using labeled datasets of different sizes and settings. We focus on the efficiency gains of our approach. First we discuss our simulation settings, then go on to show results for the $Q$ function parameters under correct and incorrect working models for \eqref{linear_Qs}. We then show value function summary statistics under correct models, and mis-specification for the $Q$ models in \eqref{linear_Qs} and the propensity score function $\pi_2$ in \eqref{logit_Ws}.

Following a similar set-up as in \citet{schulte2014}, we first consider a simple scenario with a single confounder variable at each stage with $\bH_{10}=\bH_{11}=(1,O_1)\trans$, $\bHcheck_{20}=(Y_2,1,O_1,A_1,O_1A_1,O_2)\trans$, and $\bH_{21}=(1,A_1,O_2)\trans$. Specifically, we sequentially generate 
\begin{alignat*}{3}
& O_1\sim \Bern(0.5), &\quad& A_1\sim \Bern(\sigma\left\{\bH_{10}\trans\bxi_1^0\right\}),&\quad& Y_2\sim\Nsc(\bXcheck_1\trans\btheta^0_1, 1), \\
& O_2 \sim \Nsc(\bHcheck_{20}\trans\pmb\delta^0,2),&\quad&
A_2 \sim\Bern\left(\sigma\left\{\bH_{20}\trans\bxi_2^0+\xi_{26}^0O_2^2\right\}\right), \quad \mbox{and} &\quad& Y_3 \sim\Nsc(m_3\left\{\bHcheck_{20}\right\}, 2). 
\end{alignat*}
where $m_3\{\bHcheck_{20}\}=\bH_{20}\trans\bbeta^0_2+A_2(\bH_{21}\trans\bgamma^0_2)+\beta_{27}^0O_2^2Y_2\sin\{[O_2^2(Y_2+1)]^{-1}\}$. 
Surrogates are generated as $W_t=\floor{Y_{t+1}+Z_t},$ $Z_t\sim\Nsc(0,\sigma^2_{z,t})$, $t=1,2$ where $\floor{x}$ corresponds to the integer part of $x\in\mathbb{R}$. Throughout, we let $\bxi_1^0=(0.3,-0.5)\trans$, $\bbeta_1^0=(1,1)\trans$, $\bgamma_1^0=(1,-2)\trans$
$\pmb\delta^0=(0,0.5,-0.75,0.25)\trans$, $\bxi^0_2=(0, 0.5, 0.1,-1,-0.1)\trans$ $\bbeta_2^0=(.1, 3, 0,0.1,-0.5,-0.5)\trans$, $\bgamma_2^0=(1, 0.25, 0.5)\trans$.

We consider an additional case to mimic the structure of the EHR data set used for the real-data application. Outcomes $Y_t$ are binary, and we use a higher number of covariates for the $Q$ functions and multivariate count surrogates $\bW_t$ $t=1,2$. Data is simulated with $\bH_{10}=(1,O_1,\dots,O_6)\trans$,  $\bH_{11}=(1,O_2,\dots,O_6)\trans$,  $\bHcheck_{20}=(Y_2,1,O_1,\dots,O_6,A_1,Z_{21},Z_{22})\trans$, and $\bH_{21}=(1,O_1,\dots,O_4,A_1,Z_{21},Z_{22})\trans$, generated according to
\begin{alignat*}{3}
& \bO_1\sim\Nsc(\bzero,I_6), &\quad& A_1\sim\Bern(\sigma\{\bH_{10}\trans\bxi_1^0\}),&\quad& Y_2\sim\Bern(\sigma\{\bXcheck_1\trans\btheta^0_1\}), \\
& \bO_2=\left[I\left\{Z_1>0\right\},I\left\{Z_2>0\right\}\right]\trans&\quad&
A_2 \sim\Bern\left(\tilde m_2\{\bHcheck_{20}\}\right), \quad \mbox{and} &\quad& Y_3 \sim\Bern(\tilde m_3\left\{\bHcheck_{20}\right\}),
\end{alignat*}
with $\tilde m_2=\sigma\left\{\bH_{20}\trans\bxi_2^0+\tilde\bxi_2\trans\bO_2\right\}$, $\tilde m_3(\bHcheck_{20})=\bH_{20}\trans\bbeta^0_2+A_2(\bH_{21}\trans\bgamma^0_2)+\tilde\bbeta_2\trans\bO_2Y_2\sin\{\|\bO_2\|^2_2/(Y_2+1)\}$ and $Z_l=O_{1l}\delta_l^0+\epsilon_z$, $\epsilon_z\sim\Nsc(0,1)$ $l=1,2$. The dimensions for the $Q$ functions are 13 and 37 for the first and second stage respectively, which match with our IBD dataset discussed in Section \ref{section: IBD}. The surrogates are generated according to $\bW_t=\floor{\bZ_t}$, with $\bZ_t\sim\Nsc\left(\balpha\trans(1,\bO_t,A_t,Y_t),I\right)$. Parameters are set to $\bxi_1^0=(-0.1,1,-1,0.1)\trans$, $\bbeta_1^0=(0.5, 0.2,-1,-1,0.1,-0.1,0.1)\trans$, $\bgamma_1^0=(1, -2,-2,-0.1,0.1,-1.5)\trans$, $\bxi^0_2=(0, 0.5, 0.1,-1,1,-0.1)\trans$, $\bbeta_2^0=(1,\bbeta_1^0, 0.25,-1,-0.5)\trans$, $\bgamma_2^0=(1, 0.1,-0.1,0.1,-0.1,0.25,-1,-0.5)\trans$, and $\balpha=(1,\bzero,1)\trans$.

For all settings, we fit models $Q_1(\bH_1,A_1)=\bH_{10}\trans\bbeta^0_1+A_1(\bH_{11}\trans\bgamma^0_1)$, $Q_2(\bHcheck_2,A_2)=\bHcheck_{20}\trans\bbeta^0_2+A_2(\bH_{21}\trans\bgamma^0_2)$ for the $Q$ functions, $\pi_1(\bH_1)=\sigma\left(\bH_{10}\trans\bxi_1\right)$ and $\pi_2(\bHcheck_2)=\sigma\left(\bH_{20}\trans\bxi_2\right)$ for the propensity scores. The parameters $\xi_{26}^0$ and $\beta_{27}^0$ and $\tilde\bxi_2,\tilde\bbeta_2$ index mis-specification in the fitted Q-learning outcome models and the propensity score models with a value of 0 corresponding to a correct specification. In particular, we set $\xi_{26}^0=1$, $\tilde\bxi_2=\frac{1}{\|(1,\dots,1)\|_2}(1,\dots,1)\trans$, and $\beta_{27}^0=1$, $\tilde\bbeta_2=\frac{1}{\|(1,\dots,1)\|_2}(1,\dots,1)\trans$ for mis-specification of propensity score $\pi_2$ and $Q_1,$ $Q_2$ functions respectively. We set $\xi_{26}^0=\beta_{27}^0=0$ and $\tilde\bxi=\bzero$, $\tilde\bbeta_2=\bzero$ for correct model specification. Under mis-specification of the outcome model or propensity score model, the term omitted by the working models is highly non-linear, in which case the imputation model will be mis-specified as well. We note that our method does not need correct specification of the imputation model. For the imputation models, we considered both random forest (RF) with 500 trees and basis expansion (BE) with piecewise-cubic splines with 2 equally spaced knots on the quantiles 33 and 67 \citep{HastieT}. Finally, we consider two choices of $(n,N)$: $(135,1272)$ which are similar to the sizes of our EHR study and larger sizes of $(500,10000)$. For each configuration, we summarize results based on $1,000$ replications. 

\begin{table}[ht]
\centering
\centerline{(a) $n=135$ and $N=1272$}
\scalebox{0.8}{
\begin{tabular}{@{}lccccccclcccccl@{}}
\toprule
 & \multicolumn{2}{c}{Supervised} &  & \multicolumn{11}{c}{Semi-Supervised} \\ \cmidrule(lr){2-3} \cmidrule(lr){5-15} 
 & \multicolumn{2}{c}{} &  & \multicolumn{5}{c}{Random Forests} &  & \multicolumn{5}{c}{Basis Expansion} \\ \cmidrule(l){5-9} \cmidrule(l){11-15} 
Parameter & Bias & ESE &  & Bias & ESE & ASE & CovP & RE &  & Bias & ESE & ASE & CovP & RE \\ \cmidrule(l){1-3} \cmidrule(l){5-9} \cmidrule(l){11-15} 
  $\gamma_{11}$=1.4 & -0.03 & 0.41 &  & 0.00 & 0.26 & 0.24 & 0.93 & 1.57 &  & 0.00 & 0.24 & 0.23 & 0.93 & 1.68 \\ 
  $\gamma_{12}$=-2.6 & 0.04 & 0.58 &  & -0.01 & 0.36 & 0.34 & 0.94 & 1.61 &  & -0.02 & 0.35 & 0.31 & 0.90 & 1.69 \\ 
  $\gamma_{21}$=0.8 & 0.00 & 0.34 &  & 0.01 & 0.21 & 0.20 & 0.93 & 1.61 &  & 0.00 & 0.20 & 0.19 & 0.94 & 1.71 \\ 
  $\gamma_{22}$=0.2 & -0.02 & 0.45 &  & -0.01 & 0.28 & 0.28 & 0.95 & 1.60 &  & -0.01 & 0.27 & 0.26 & 0.94 & 1.70 \\ 
  $\gamma_{23}$=0.5 & 0 & 0.18 &  & 0.01 & 0.11 & 0.11 & 0.94 & 1.59 &  & 0.00 & 0.11 & 0.11 & 0.94 & 1.68 \\ \bottomrule
\end{tabular}}\vspace{.1in}
\centerline{(b) $n=500$ and $N=10,000$}
\scalebox{0.8}{
\begin{tabular}{@{}lccccccclcccccl@{}}
\toprule
 & \multicolumn{2}{c}{Supervised} &  & \multicolumn{11}{c}{Semi-Supervised} \\ \cmidrule(lr){2-3} \cmidrule(lr){5-15} 
 & \multicolumn{2}{c}{} &  & \multicolumn{5}{c}{Random Forests} &  & \multicolumn{5}{c}{Basis Expansion} \\ \cmidrule(l){5-9} \cmidrule(l){11-15} 
Parameter & Bias & ESE &  & Bias & ESE & ASE & CovP & RE &  & Bias & ESE & ASE & CovP & RE \\ \cmidrule(l){1-3} \cmidrule(l){5-9} \cmidrule(l){11-15} 
  $\gamma_{11}$=1.4 & 0.01 & 0.22 &  & 0.01 & 0.12 & 0.11 & 0.92 & 1.76 &  & 0.01 & 0.12 & 0.11 & 0.92 & 1.80 \\ 
  $\gamma_{12}$=-2.6 & 0 & 0.29 &  & 0 & 0.17 & 0.16 & 0.93 & 1.73 &  & -0.01 & 0.16 & 0.15 & 0.93 & 1.80 \\ 
  $\gamma_{21}$=0.8 & 0.00 & 0.17 &  & 0.00 & 0.10 & 0.09 & 0.93 & 1.80 &  & 0.00 & 0.09 & 0.09 & 0.93 & 1.86 \\ 
  $\gamma_{22}$=0.2 & -0.01 & 0.23 &  & 0 & 0.13 & 0.12 & 0.93 & 1.81 &  & 0 & 0.13 & 0.12 & 0.94 & 1.83 \\ 
  $\gamma_{23}$=0.5 & 0.00 & 0.09 &  & 0.00 & 0.05 & 0.05 & 0.94 & 1.78 &  & 0.00 & 0.05 & 0.05 & 0.95 & 1.81 \\ \bottomrule
\end{tabular}}
\caption{Bias, empirical standard error (ESE) of the supervised and the SSL estimators with either random forest imputation or basis expansion imputation strategies for $\bgammabar_1,\bgammabar_2$ when (a) $n=135$ and $N=1272$ and (b) $n=500$ and $N=10,000$. For the SSL estimators, we also obtain the average of the estimated standard errors (ASE) as well as the empirical coverage probabilities (CovP) of the 95\% confidence intervals.}
\label{tab:simgamma}
\end{table}

We start discussing results under correct specification of the $Q$ functions. In Table \ref{tab:simgamma}, we present the results for the estimation of treatment interaction coefficients $\bgammabar_1,\bgammabar_2$, under the correct model specification, continuous outcome setting with $\beta_{27}^0=\xi_{26}^0=0$. The complete tables for all $\bthetabar$ parameters for the continuous and EHR-like settings can be found in Appendix \ref{appendix: alt sims}. We report bias, empirical standard error (ESE), average standard error (ASE), 95\% coverage probability (CovP) and relative efficiency (RE) defined as the ratio of supervised ESE over SSL estimate ESE. Overall, compared to the supervised approach, the proposed semi-supervised $Q$-learning approach has substantial gains in efficiency while maintaining comparable or even lower bias. This is likely due to the refitting step which helps take care of the finite sample bias, both from the missing outcome imputation and $Q$ function parameter estimation. Imputation with BE yields slightly better estimates than when using RF, both in terms of efficiency and bias. Coverage probabilities are close to the nominal level due to the good performance of the standard error estimation. 

We next turn to $Q$-learning parameters under mis-specification of \eqref{linear_Qs}. Figure \ref{fig_misspec_Q_sin} shows the bias and root mean square error (RMSE) for the treatment interaction coefficients in the 2-stage $Q$ functions. We focus on the continuous setting, where we set $\beta_{27}^0\in\{-1,0,1\}$. Note that $\beta_{27}^0\neq0$ implies that both $Q$ functions are mis-specified as the fitting of $Q_1$ depends on formulation of $Q_2$ as seen in \eqref{full_EE}. Semi-supervised $Q$-learning is more efficient for any degree of mis-specification for both small and large finite sample settings. As the theory predicts, there is no real difference in efficiency gain of SSL across mis-specification of the $Q$ function models. This is because asymptotic distribution of $\bgammahat\subSSL$ shown in Theorems \ref{theorem: unbiased theta2} \& \ref{theorem: unbiased theta1} are centered on the target parameters $\bgammabar$. Thus, both SSL and SUP have negligible bias regardless of the true value of $\beta_{27}^0$.

\begin{figure}[ht]
	\centering
	\includegraphics[width=1.05\textwidth]{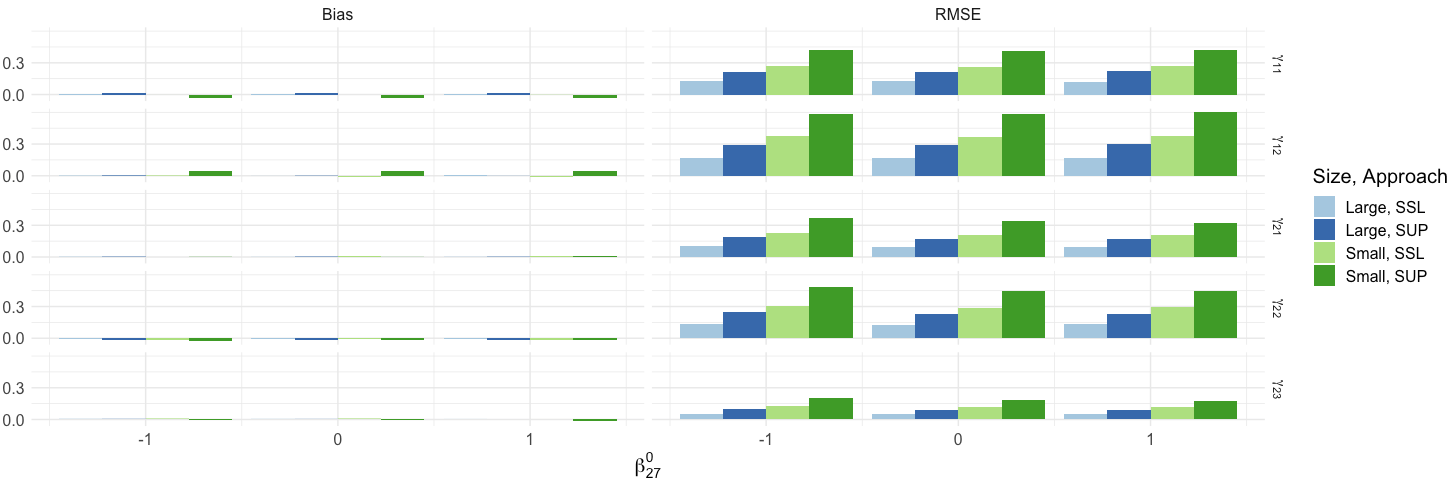}
	\caption{Monte Carlo estimates of bias and RMSE ratios for estimation of $\gamma_{11},\:\gamma_{12}$, $\gamma_{21},\:\gamma_{22},\:\gamma_{23}$ under mis-specification of the $Q$-functions through $\beta_{27}^0$. Results are shown for the large ($N=10,000$, $n=500$) and small ($N=1,272$, $n=135$) data samples for the continuous setting over 1,000 simulated datasets.}
	\label{fig_misspec_Q_sin} 
\end{figure}

\begin{table}[ht]
\centering
\centerline{(a) $n=135$ and $N=1272$}
\scalebox{0.7}{
\begin{tabular}{rlrlrrlrrrrrlrrrrr}
\hline
\multicolumn{1}{l}{}        &             & \multicolumn{1}{l}{} &  & \multicolumn{2}{c}{Supervised}              &  & \multicolumn{11}{c}{Semi-Supervised}                                                                                                                                                                                                   \\ \cmidrule{5-6} \cmidrule{8-18} 
\multicolumn{1}{l}{}        &             & \multicolumn{1}{l}{} &  & \multicolumn{1}{l}{} & \multicolumn{1}{l}{} &  & \multicolumn{5}{c}{Random Forests}                                                                               &  & \multicolumn{5}{c}{Basis Expansion}                                                                              \\  \cmidrule{8-12} \cmidrule{14-18} 
Setting                     & Model       & $\Vbar$              &  & Bias                 & ESE                  &  & Bias                 & ESE                  & ASE                  & CovP                 & RE                   &  & Bias                 & ESE                  & ASE                  & CovP                 & RE                   \\ \cmidrule{1-3} \cmidrule{5-6} \cmidrule{8-12} \cmidrule{14-18} 
 & Correct & 6.08 &  & 0.02 & 0.27 &  & 0.04 & 0.21 & 0.24 & 0.97 & 1.27 &  & 0.02 & 0.23 & 0.25 & 0.97 & 1.18 \\ 
  Continuous & Missp. $Q$ & 6.34 &  & 0.01 & 0.24 &  & 0.03 & 0.19 & 0.22 & 0.97 & 1.27 &  & 0.00 & 0.20 & 0.22 & 0.97 & 1.20 \\ 
   & Missp. $\pi$ & 6.08 &  & 0.01 & 0.28 &  & 0.02 & 0.22 & 0.24 & 0.97 & 1.24 &  & 0.01 & 0.25 & 0.25 & 0.97 & 1.12 \\ 
   & Correct & 1.38 &  & 0.09 & 0.15 &  & 0.05 & 0.12 & 0.12 & 0.94 & 1.24 &  & 0.04 & 0.13 & 0.12 & 0.95 & 1.12 \\ 
  EHR & Missp. $Q$ & 1.43 &  & 0.09 & 0.14 &  & 0.04 & 0.12 & 0.12 & 0.96 & 1.12 &  & 0.03 & 0.14 & 0.12 & 0.95 & 1.02 \\ 
   & Missp. $\pi$ & 1.38 &  & 0.09 & 0.15 &  & 0.05 & 0.14 & 0.13 & 0.96 & 1.13 &  & 0.04 & 0.14 & 0.13 & 0.96 & 1.05 \\   \end{tabular}}\vspace{.1in}
\centerline{(b) $n=500$ and $N=10,000$}
\scalebox{0.7}{
\begin{tabular}{rlrlrrlrrrrrlrrrrr}
\toprule
\multicolumn{1}{l}{}        &             & \multicolumn{1}{l}{} &  & \multicolumn{2}{c}{Supervised}              &  & \multicolumn{11}{c}{Semi-Supervised}                                                                                                                                                                                                   \\ \cmidrule{5-6} \cmidrule{8-18} 
\multicolumn{1}{l}{}        &             & \multicolumn{1}{l}{} &  & \multicolumn{1}{l}{} & \multicolumn{1}{l}{} &  & \multicolumn{5}{c}{Random Forests}                                                                               &  & \multicolumn{5}{c}{Basis Expansion}                                                                              \\  \cmidrule{8-12} \cmidrule{14-18} 
Setting                     & Model       & $\Vbar$              &  & Bias                 & ESE                  &  & Bias                 & ESE                  & ASE                  & CovP                 & RE                   &  & Bias                 & ESE                  & ASE                  & CovP                 & RE                   \\ \cmidrule{1-3} \cmidrule{5-6} \cmidrule{8-12} \cmidrule{14-18} 
 & Correct & 6.08 &  & 0.02 & 0.15 &  & 0.03 & 0.11 & 0.12 & 0.96 & 1.32 &  & 0.02 & 0.13 & 0.13 & 0.95 & 1.16 \\ 
  Continuous & Missp. $Q$ & 6.34 &  & 0.01 & 0.13 &  & 0.03 & 0.10 & 0.10 & 0.96 & 1.31 &  & 0.01 & 0.11 & 0.11 & 0.96 & 1.16 \\ 
   & Missp. $\pi$ & 6.08 &  & 0.01 & 0.14 &  & 0.03 & 0.11 & 0.12 & 0.96 & 1.28 &  & 0.02 & 0.12 & 0.12 & 0.95 & 1.16 \\ 
   & Correct & 1.38 &  & 0.02 & 0.07 &  & 0.01 & 0.04 & 0.06 & 0.99 & 1.55 &  & 0.00 & 0.06 & 0.06 & 0.98 & 1.23 \\ 
  EHR & Missp. $Q$ & 1.43 &  & 0.01 & 0.07 &  & 0.00 & 0.04 & 0.05 & 0.99 & 1.66 &  & 0.00 & 0.05 & 0.06 & 0.98 & 1.35 \\ 
   & Missp. $\pi$ & 1.38 &  & 0.02 & 0.08 &  & 0.01 & 0.06 & 0.07 & 0.99 & 1.22 &  & 0.00 & 0.07 & 0.07 & 0.97 & 1.03 \\  \bottomrule
\end{tabular}}
\caption{Bias, empirical standard error (ESE) of the supervised estimator $\Vhat\subSUPDR$ and bias, ESE, average standard error (ASE) and coverage probability (CovP) for $\Vhat\subSSLDR$ with either random forest imputation or basis expansion imputation strategies when (a) $n=135$ and $N=1272$ and (b) $n=500$ and $N=10,000$. We show performance and relative efficiency across both simulation settings for estimation under correct models, and mis-specification of $Q$ function or propensity score function.}
\label{tab:simvalues}
\end{table}

Next we analyze performance of the doubly robust value function estimators for both continuous and EHR-like settings. Table \ref{tab:simvalues} shows bias and RMSE across different sample sizes, and comparing SSL vs. SUP estimators. Results are shown for the correct specification of the $Q$ functions and propensity scores, and when either is mis-specified. Bias across simulation settings is relatively similar between $\Vhat\subSSLDR$ and $\Vhat\subSUPDR$, and appears to be small relative to RMSE. The low magnitude of bias suggests both estimators are robust to model mis-specification. There is an exception on the EHR setting with small sample size, for which the bias is non-negligible. This is likely due to the fact that the $Q$ function parameters to estimate are 13+37, and the propensity score functions have 12 parameters which add up to a large number relative to the labeled sample size: $n=135$. The SSL bias is lower in this case which could be due to the refitting step, which helped to reduce the finite sample bias. Efficiency gains of $\Vhat\subSSLDR$ are consistent across model specification. We next illustrate our approach using an IBD dataset.

\subsection{Application to an EHR Study of Inflammatory Bowel Disease}\label{section: IBD}

Anti–tumor necrosis factor (anti-TNF) therapy has greatly changed the management and improved the outcomes of patients with inflammatory bowl disease (IBD) \citep{peyrin2010anti}. However, it remains unclear whether a specific anti-TNF agent has any advantage in efficacy over other agents, especially at the individual level. There have been few randomized clinical trials performed to directly compare anti-TNF agents for treating IBD patients \citep{sands2019vedolizumab}. Retrospective studies comparing  infliximab and adalimumab for treating IBD have found limited and sometimes conflicting evidence of their relative effectiveness \citep{inokuchi2019long,lee2019comparison,osterman2017infliximab}. There is even less evidence regarding optimal STR for choosing these treatments over time \citep{ashwin2016}. To explore this, we performed RL using data from a cohort of IBD patients previously identified via machine learning algorithms from the EHR systems of two tertiary referral academic centers in the Greater Boston metropolitan area \citep{ashwin2012}. We focused on the subset of $N=1,272$ patients who initiated either Infliximab ($A_1=0$) or Adalimumab ($A_1=1$) and continued to be treated by either of these two therapies during the next 6 months. The observed treatment sequence distributions are shown in Table \ref{Table: obs As}. The outcomes of interest are the binary indicator of treatment response at 6 months ($t=2$) and at 12 months ($t=3$), both of which were only available on a subset of $n=135$ patients whose outcomes were manually annotated via chart review. 

To derive the STR, we included gender, age, Charlson co-morbidity index \citep{charlson}, prior exposure to anti-TNF agents, as well as mentions of clinical terms associated with IBD such as bleeding complications extracted from the clinical notes via natural language processing (NLP) features as confounding variables at both time points. To improve the imputation of $Y_t$, we use 15 relevant NLP features such as mentions of rectal or bowel resection surgery as surrogates at $t=1,2$. We transformed all count variables using $x\mapsto \log(1+x)$ to decrease skewness in the distributions, and centered continuous features. We used RF with 500 trees to carry out the imputation step, and 5-fold cross-validation (CV) to estimate the value function.

The supervised and semi-supervised estimates are shown in Table \ref{Table: IBD results} for the $Q$-learning models and in Table \ref{Table: V IBD results} for the value functions associated with the estimated STR. Similar to those observed in the simulation studies, the semi-supervised $Q$-learning has more power to detect significant predictors of treatment response. Relative efficiency for almost all $Q$ function estimates is near or over 2. The supervised $Q$-learning does not have the power to detect predictors such as prior use of anti-TNF agents, which are clearly relevant to treatment response \citep{ashwin2016}. Semi-supervised $Q$-learning is able to detect that the efficacy of Adalimumab wears off as patients get older, meaning younger patients in the first stage experienced a higher rate of treatment response to Adalimumab, a finding that cannot be detected with supervised $Q$-learning. Additionally, supervised $Q$-learning does not pick up that there is a higher rate of response to Adalimumab among patients that are male or have experienced an abscess. This translates into a far from optimal treatment rule as seen in the cross-validated value function estimates. Table \ref{Table: V IBD results} reflects that using our semi-supervised approach to find the regime and to estimate the value function of such treatment rules yields a more efficient estimate, as the semi-supervised value function estimate $\Vhat\subSUPDR$ yielded a smaller standard error than that of the supervised estimate $\Vhat\subSUPDR$. However, the standard errors are large relative to the point estimates. On the upside, they both yield estimates very close in numerical value which is reassuring: both should be unbiased as predicted by theory and simulations.

\begin{table}[ht]
\centering
\begin{tabular}{clcc}
\hline
                                            &   & \multicolumn{2}{c}{$A_1$} \\ \cline{2-4} 
\multicolumn{1}{c}{}                       &   & 0           & 1           \\
\multicolumn{1}{c}{\multirow{2}{*}{$A_2$}} & 0 & 912         & 327         \\
\multicolumn{1}{c}{}                       & 1 & 27          & 183         \\ \hline
\end{tabular}\caption{Distribution of treatment trajectories for observed sample of size 1407.}\label{Table: obs As}

\end{table}

\begin{table}[ht]
\scalebox{0.55}{
\begin{tabular}{lcccccccccclccccccccc}
\cmidrule{1-10} \cmidrule{12-21}
\multicolumn{10}{c}{Stage 1 Regression} &  & \multicolumn{10}{c}{Stage 2 Regression} \\ \cmidrule{1-10} \cmidrule{12-21} 
 & \multicolumn{3}{c}{Supervised} &  & \multicolumn{3}{c}{Semi-Supervised} &  &  &  & \multicolumn{4}{c}{Supervised} &  & \multicolumn{3}{c}{Semi-Supervised} &  &  \\ \cmidrule{2-4} \cmidrule{6-8} \cmidrule{12-15} \cmidrule{17-19}
Parameter & Estimate & SE & P-val &  & Estimate & SE & P-val &  & RE &  & Parameter & Estimate & SE & P-val &  & Estimate & SE & P-val &  & RE \\ \cmidrule{1-4} \cmidrule{6-8} \cmidrule{10-10} \cmidrule{12-15} \cmidrule{17-19} \cmidrule{21-21} 
Intercept & \textbf{0.424} & \textbf{0.082} & \textbf{0.00} &  & \textbf{0.518} & \textbf{0.028} & \textbf{0.00} &  & 2.937 &  & $Y_1$ & \textbf{0.37} & \textbf{0.11} & \textbf{0.00} &  & \textbf{0.55} & \textbf{0.05} & \textbf{0.00} &  & 2.08 \\ 
  Female & -0.237 & 0.167 & 0.16 &  & \textbf{-0.184} & \textbf{0.067} & \textbf{0.007} &  & 2.514 &  & Intercept & 0.08 & 0.06 & 0.17 &  & 0.04 & 0.02 & 0.14 &  & 2.40 \\ 
  Age & 0.155 & 0.088 & 0.081 &  & \textbf{0.18} & \textbf{0.034} & \textbf{0.00} &  & 2.588 &  & Female & -0.01 & 0.10 & 0.92 &  & -0.00 & 0.05 & 0.98 &  & 2.21 \\ 
  Charlson Score & 0.006 & 0.072 & 0.929 &  & -0.047 & 0.026 & 0.075 &  & 2.776 &  & Age & 0.05 & 0.06 & 0.35 &  & \textbf{0.07} & \textbf{0.02} & \textbf{0.00} &  & 2.33 \\ 
  Prior anti-TNF & -0.038 & 0.06 & 0.524 &  &\textbf{ -0.085} & \textbf{0.019} & \textbf{0.00} &  & 3.177 &  & Charlson Score & 0.04 & 0.04 & 0.33 &  & \textbf{0.06} & \textbf{0.02} & \textbf{0.01} &  & 2.06 \\ 
  Perianal & \textbf{0.138} & \textbf{0.06} & \textbf{0.022} &  & \textbf{0.179} & \textbf{0.022} & \textbf{0.00} &  & 2.688 &  & Prior anti-TNF & -0.05 & 0.05 & 0.29 &  & \textbf{-0.09} & \textbf{0.02} & \textbf{0.00} &  & 2.39 \\ 
  Bleeding & 0.049 & 0.08 & 0.54 &  & 0.058 & 0.03 & 0.055 &  & 2.675 &  & Perianal & -0.01 & 0.04 & 0.80 &  & \textbf{-0.03} & \textbf{0.02} & \textbf{0.06} &  & 2.31 \\ 
  A1 & 0.163 & 0.488 & 0.739 &  & 0.148 & 0.206 & 0.473 &  & 2.374 &  & Bleeding & -0.04 & 0.05 & 0.49 &  & -0.03 & 0.03 & 0.29 &  & 2.14 \\ 
  Female$\times A_1$ & 0.168 & 0.696 & 0.81 &  & -0.042 & 0.287 & 0.886 &  & 2.424 &  & A1 & 0.11 & 0.25 & 0.67 &  & 0.03 & 0.10 & 0.74 &  & 2.60 \\ 
  Age$\times A_1$ & -0.177 & 0.264 & 0.503 &  & \textbf{-0.278} & \textbf{0.109} & \textbf{0.013} &  & 2.418 &  & Abscess$_2$ & 0.06 & 0.04 & 0.16 &  & \textbf{0.05} & \textbf{0.01} & \textbf{0.00} &  & 2.68 \\ 
  Charlson Score$\times A_1$ & 0.136 & 0.391 & 0.728 &  & 0.195 & 0.178 & 0.276 &  & 2.194 &  & Fistula$_2$ & 0.02 & 0.05 & 0.67 &  & 0.01 & 0.02 & 0.62 &  & 2.33 \\ 
  Perianal$\times A_1$ & -0.113 & 0.226 & 0.618 &  & -0.019 & 0.08 & 0.808 &  & 2.838 &  & Female$\times A_1$ & 0.13 & 0.38 & 0.74 &  & 0.17 & 0.16 & 0.30 &  & 2.37 \\ 
  Bleeding$\times A_1$ & 0.262 & 0.364 & 0.474 &  & 0.127 & 0.161 & 0.431 &  & 2.267 &  & Age$\times A_1$ & -0.02 & 0.12 & 0.88 &  & -0.09 & 0.06 & 0.17 &  & 1.94 \\ 
   &  &  &  &  &  &  &  &  &  &  & Charlson Score$\times A_1$ & -0.02 & 0.16 & 0.89 &  & 0.04 & 0.07 & 0.55 &  & 2.19 \\ 
   &  &  &  &  &  &  &  &  &  &  & Perianal$\times A_1$ & -0.14 & 0.09 & 0.15 &  & \textbf{-0.17} & \textbf{0.04} & \textbf{0.00} &  & 2.34 \\ 
   &  &  &  &  &  &  &  &  &  &  & Bleeding$\times A_1$ & 0.13 & 0.20 & 0.51 &  & 0.03 & 0.09 & 0.76 &  & 2.17 \\ 
   &  &  &  &  &  &  &  &  &  &  & A2 & 0.07 & 0.17 & 0.69 &  & \textbf{0.22} & \textbf{0.07} & \textbf{0.00} &  & 2.55 \\ 
   &  &  &  &  &  &  &  &  &  &  & Female$\times A_2$ & -0.39 & 0.28 & 0.16 &  & \textbf{-0.51} & \textbf{0.11} & \textbf{0.00} &  & 2.53 \\ 
   &  &  &  &  &  &  &  &  &  &  & Age$\times A_2$ & 0.09 & 0.10 & 0.40 &  & \textbf{0.15} & \textbf{0.04} & \textbf{0.00} &  & 2.27 \\ 
   &  &  &  &  &  &  &  &  &  &  & Charlson Score$\times A_2$ & 0.01 & 0.07 & 0.84 &  & -0.03 & 0.03 & 0.42 &  & 2.08 \\ 
   &  &  &  &  &  &  &  &  &  &  & Perianal$\times A_2$ & \textbf{0.20} & \textbf{0.09} &\textbf{ 0.04} &  & \textbf{0.23} & \textbf{0.04} &\textbf{ 0.00} &  & 2.23 \\ 
   &  &  &  &  &  &  &  &  &  &  & Bleeding$\times A_2$ & 0.03 & 0.08 & 0.77 &  & 0.02 & 0.04 & 0.49 &  & 2.34 \\ 
   &  &  &  &  &  &  &  &  &  &  & Abscess$_2\times A_2$ & -0.13 & 0.07 & 0.06 &  & \textbf{-0.09} & \textbf{0.03} & \textbf{0.00} &  & 2.31 \\ 
   &  &  &  &  &  &  &  &  &  &  & Fistula$_2\times A_2$ & -0.04 & 0.06 & 0.56 &  & -0.03 & 0.03 & 0.36 &  & 2.17 \\  \bottomrule
\end{tabular}}
\caption{Results of Inflammatory Bowel Disease data set, for first and second stage regressions. Fully supervised $Q$-learning is shown on the left and semi-supervised is shown on the right. Last columns in the panels show relative efficiency (RE) defined as the ratio of standard errors of the semi-supervised vs. supervised method, RE greater than one favors semi-supervised. Significant coefficients at the 0.05 level are in bold.}\label{Table: IBD results}
\end{table}

\begin{table}[ht]
\centering
\begin{tabular}{lll}
                         & Estimate & SE    \\ \hline
$\Vhat\subSUPDR$  & 0.851    & 0.486 \\
$\Vhat\subSSLDR$ & 0.871     & 0.397 \\ \hline
\end{tabular}\caption{Value function estimates for Inflammatory Bowel Disease data set, the first row has the estimate for treatment rule learned using $\mathcal{U}$ and its respective value function, the second row shows the same for a rule estimated using $\mathcal{L}$ and its estimated value.}\label{Table: V IBD results}
\end{table}
\section{Discussion}\label{section: discussion}    

We have proposed an efficient and robust strategy for estimating optimal dynamic treatment rules and their value function, in a setting where patient outcomes are scarce. In particular, we developed a two step estimation procedure amenable to non-parametric imputation of the missing outcomes. This helped us establish $\sqrt n$-consistency and asymptotic normality for both the $Q$ function parameters $\bthetahat$ and the doubly robust value function estimator $\Vhat\subSSLDR$.  We additionally provided theoretical results which illustrate if and when the outcome-surrogates $\bW$ contribute towards efficiency gain in estimation of $\bthetahat\subSSL$ and $\Vhat\subSSLDR$. This lets us conclude that our procedure is always preferable to using the labeled data only: since estimation is robust to mis-specification of the imputation models, our approach is safe to use and will be at least as efficient as the supervised methods.

We focused on the 2-time point, binary action setting for simplicity but all our theoretical results and algorithms can be easily extended to a higher finite time horizon, and multiple actions with careful bookkeeping of notation. In practice, one would need to be careful with the variability of the IPW-value function which increases substantially with time. However, the SSL approach would come in handy to estimate propensity scores, providing an efficiency gain that would help stabilize the IPW in longer horizons. 

We are interested in extending this framework to handle missing at random (MAR) sampling mechanisms. In the EHR setting, it is feasible to sample a subset of the data completely at random in order to annotate the records. Hence, we argue the MCAR assumption is true by design in our context. However, the MAR context allows us to leverage different data sources for $\Lsc$ and $\Usc$. For example, we could use an annotated EHR data cohort and a large unlabeled registry data repository for our inference, ultimately making the policies and value estimation more efficient and robust. We believe this line of work has the potential to leverage massive observational cohorts, which will help to improve personalized clinical care for a wide range of diseases.

\newpage

\bibliography{references}

\begin{thebibliography}{44}
\providecommand{\natexlab}[1]{#1}
\providecommand{\url}[1]{\texttt{#1}}
\expandafter\ifx\csname urlstyle\endcsname\relax
  \providecommand{\doi}[1]{doi: #1}\else
  \providecommand{\doi}{doi: \begingroup \urlstyle{rm}\Url}\fi

\bibitem[Ananthakrishnan et~al.(2012)Ananthakrishnan, Cai, Cheng, Chen, Savova,
  Perez, Gainer, Murphy, Szolovits, Liao, Karlson, Churchill, Kohane, and
  Plenge]{ashwin2012}
An~Ananthakrishnan, Tianxi Cai, SC~Cheng, Pj~Chen, G~Savova, RG~Perez,
  Vs~Gainer, Sn~Murphy, P~Szolovits, K~Liao, Ew~Karlson, S~Churchill, I~Kohane,
  and RM~Plenge.
\newblock Improving case definition of crohn's disease and ulcerative colitis
  in electronic medical records using natural language processing - a novel
  informatics approach.
\newblock \emph{Gastroenterology}, 142\penalty0 (5):\penalty0 S791--S791, 2012.
\newblock ISSN 0016-5085.

\bibitem[Ananthakrishnan et~al.(2016)Ananthakrishnan, Cagan, Cai, Gainer, Shaw,
  Churchill, Karlson, Kohane, Liao, and Murphy]{ashwin2016}
An~Ananthakrishnan, A~Cagan, Tianxi Cai, Vs~Gainer, S~Shaw, S~Churchill,
  E~Karlson, I~Kohane, K~Liao, and S~Murphy.
\newblock Comparative effectiveness of infliximab and adalimumab in crohn's
  disease and ulcerative colitis.
\newblock \emph{Gastroenterology}, 150\penalty0 (4):\penalty0 S979--S979, 2016.
\newblock ISSN 0016-5085.

\bibitem[Biau et~al.(2008)Biau, Devroye, and Lugosi]{biau2008}
G~Biau, L~Devroye, and G~Lugosi.
\newblock Consistency of random forests and other averaging classifiers.
\newblock \emph{Journal Of Machine Learning Research}, 9:\penalty0 2015--2033,
  2008.
\newblock ISSN 1532-4435.

\bibitem[Blitzer and Zhu(2008)]{BlitzerZ08}
John Blitzer and Xiaojin Zhu.
\newblock Semi-supervised learning for natural language processing.
\newblock In \emph{ACL (Tutorial Abstracts)}, page~3, 2008.
\newblock URL \url{http://www.aclweb.org/anthology/P08-5003}.

\bibitem[Chakrabortty(2016)]{ChakraborttyThesis}
Abhishek Chakrabortty.
\newblock Robust semi-parametric inference in semi-supervised settings, 2016.

\bibitem[Chakrabortty et~al.(2018)Chakrabortty, Cai, et~al.]{chakrabortty}
Abhishek Chakrabortty, Tianxi Cai, et~al.
\newblock Efficient and adaptive linear regression in semi-supervised settings.
\newblock \emph{The Annals of Statistics}, 46\penalty0 (4):\penalty0
  1541--1572, 2018.

\bibitem[Chakraborty and Moodie(2013)]{DTRbook}
Bibhas Chakraborty and Erica~E.M Moodie.
\newblock \emph{Statistical Methods for Dynamic Treatment Regimes:
  Reinforcement Learning, Causal Inference, and Personalized Medicine}.
\newblock Statistics for Biology and Health. Springer New York, New York, NY,
  2013 edition, 2013.
\newblock ISBN 9781461474272.

\bibitem[Chapelle et~al.(2006)Chapelle, Sch{\"o}lkopf, and Zien]{Chapelle2006}
Olivier Chapelle, Bernhard Sch{\"o}lkopf, and Alexander Zien.
\newblock \emph{Semi-supervised learning}.
\newblock Adaptive computation and machine learning. MIT Press, Cambridge,
  Mass., 2006.

\bibitem[Charlson et~al.(1987)Charlson, Pompei, Ales, and Mackenzie]{charlson}
Mary~E Charlson, Peter Pompei, Kathy~L Ales, and C.Ronald Mackenzie.
\newblock A new method of classifying prognostic comorbidity in longitudinal
  studies: Development and validation.
\newblock \emph{Journal of Chronic Diseases}, 40\penalty0 (5):\penalty0
  373--383, 1987.
\newblock ISSN 0021-9681.

\bibitem[Cheng et~al.(2020)Cheng, Ananthakrishnan, and Cai]{cheng2020robust}
David Cheng, Ashwin~N Ananthakrishnan, and Tianxi Cai.
\newblock Robust and efficient semi-supervised estimation of average treatment
  effects with application to electronic health records data.
\newblock \emph{Biometrics}, 2020.

\bibitem[Dudley(1979)]{Dudley}
R.M Dudley.
\newblock Balls in rk do not cut all subsets of k + 2 points.
\newblock \emph{Advances in mathematics (New York. 1965)}, 31\penalty0
  (3):\penalty0 306--308, 1979.
\newblock ISSN 0001-8708.

\bibitem[Finn et~al.(2016)Finn, Yu, Fu, Abbeel, and Levine]{Finn2016}
Chelsea Finn, Tianhe Yu, Justin Fu, Pieter Abbeel, and Sergey Levine.
\newblock Generalizing skills with semi-supervised reinforcement learning.
\newblock 2016.

\bibitem[Hastie(1992)]{HastieT}
T.J Hastie.
\newblock \emph{Statistical Models in S}.
\newblock CRC Press, 1 edition, 1992.
\newblock ISBN 041283040X.

\bibitem[Hong et~al.(2019)Hong, Liao, and Cai]{hong2019semi}
Chuan Hong, Katherine~P Liao, and Tianxi Cai.
\newblock Semi-supervised validation of multiple surrogate outcomes with
  application to electronic medical records phenotyping.
\newblock \emph{Biometrics}, 75\penalty0 (1):\penalty0 78--89, 2019.

\bibitem[Inokuchi et~al.(2019)Inokuchi, Takahashi, Hiraoka, Toyokawa, Takagi,
  Takemoto, Miyaike, Fujimoto, Higashi, Morito, et~al.]{inokuchi2019long}
Toshihiro Inokuchi, Sakuma Takahashi, Sakiko Hiraoka, Tatsuya Toyokawa,
  Shinjiro Takagi, Koji Takemoto, Jiro Miyaike, Tsuyoshi Fujimoto, Reiji
  Higashi, Yuki Morito, et~al.
\newblock Long-term outcomes of patients with crohn's disease who received
  infliximab or adalimumab as the first-line biologics.
\newblock \emph{Journal of gastroenterology and hepatology}, 34\penalty0
  (8):\penalty0 1329--1336, 2019.

\bibitem[Jiang and Li(2016)]{Jiang2016}
Nan Jiang and Lihong Li.
\newblock Doubly robust off-policy value evaluation for reinforcement learning.
\newblock \emph{arXiv.org}, 2016.
\newblock URL \url{http://search.proquest.com/docview/2080150644/}.

\bibitem[Kallus and Mao(2020)]{kallus2020role}
Nathan Kallus and Xiaojie Mao.
\newblock On the role of surrogates in the efficient estimation of treatment
  effects with limited outcome data.
\newblock \emph{arXiv preprint arXiv:2003.12408}, 2020.

\bibitem[Kosorok and Laber(2019)]{2019PM}
Michael~R. Kosorok and Eric~B. Laber.
\newblock Precision medicine.
\newblock 6\penalty0 (1):\penalty0 263--286, 2019.
\newblock ISSN 2326-8298.

\bibitem[Laber et~al.(2014)Laber, Lizotte, Qian, Pelham, and Murphy]{laber2014}
Eric~B Laber, Daniel~J Lizotte, Min Qian, William~E Pelham, and Susan~A Murphy.
\newblock Dynamic treatment regimes: technical challenges and applications.
\newblock \emph{Electronic journal of statistics}, 8\penalty0 (1):\penalty0
  1225--1272, 2014.
\newblock ISSN 1935-7524.
\newblock URL \url{http://search.proquest.com/docview/1826600138/}.

\bibitem[Lee et~al.(2019)Lee, Cheon, Park, Park, Kim, and
  Kim]{lee2019comparison}
Yongil Lee, Jae~Hee Cheon, Yehyun Park, Soo~Jung Park, Tae~Il Kim, and Won~Ho
  Kim.
\newblock Comparison of long-term outcomes between infliximab and adalimumab in
  biologic-naive patients with ulcerative colitis.
\newblock \emph{Gut \& Liver}, 13, 2019.

\bibitem[Murphy(2003)]{MurphyAlearning}
S.~A. Murphy.
\newblock Optimal dynamic treatment regimes.
\newblock \emph{Journal of the Royal Statistical Society: Series B (Statistical
  Methodology)}, 65\penalty0 (2):\penalty0 331--355, 2003.
\newblock \doi{10.1111/1467-9868.00389}.
\newblock URL
  \url{https://rss.onlinelibrary.wiley.com/doi/abs/10.1111/1467-9868.00389}.

\bibitem[Murphy(2005)]{murphy2005}
SA~Murphy.
\newblock A generalization error for q-learning.
\newblock \emph{Journal Of Machine Learning Research}, 6:\penalty0 1073--1097,
  2005.
\newblock ISSN 1532-4435.

\bibitem[Osterman and Lichtenstein(2017)]{osterman2017infliximab}
Mark~T Osterman and Gary~R Lichtenstein.
\newblock Infliximab vs adalimumab for uc: Is there a difference?
\newblock \emph{Clinical Gastroenterology and Hepatology}, 15\penalty0
  (8):\penalty0 1197--1199, 2017.

\bibitem[Peyrin-Biroulet(2010)]{peyrin2010anti}
L~Peyrin-Biroulet.
\newblock Anti-tnf therapy in inflammatory bowel diseases: a huge review.
\newblock \emph{Minerva gastroenterologica e dietologica}, 56\penalty0
  (2):\penalty0 233, 2010.

\bibitem[Qiao et~al.(2018)Qiao, Shen, Zhang, Wang, and Yuille]{qiao2018}
Siyuan Qiao, Wei Shen, Zhishuai Zhang, Bo~Wang, and Alan Yuille.
\newblock Deep co-training for semi-supervised image recognition, 2018.

\bibitem[Robins(1997)]{robins1997}
J.~Robins.
\newblock Causal inference from complex longitudinal data.
\newblock \emph{Latent Variable Modeling and Applications to Causality}, pages
  69---117, 1997.

\bibitem[Robins(2004)]{robins2004}
James~M. Robins.
\newblock \emph{Optimal Structural Nested Models for Optimal Sequential
  Decisions}, pages 189--326.
\newblock Springer New York, New York, NY, 2004.
\newblock ISBN 978-1-4419-9076-1.
\newblock \doi{10.1007/978-1-4419-9076-1_11}.
\newblock URL \url{https://doi.org/10.1007/978-1-4419-9076-1_11}.

\bibitem[Sands et~al.(2019)Sands, Peyrin-Biroulet, Loftus~Jr, Danese, Colombel,
  T{\"o}r{\"u}ner, Jonaitis, Abhyankar, Chen, Rogers,
  et~al.]{sands2019vedolizumab}
Bruce~E Sands, Laurent Peyrin-Biroulet, Edward~V Loftus~Jr, Silvio Danese,
  Jean-Fr{\'e}d{\'e}ric Colombel, Murat T{\"o}r{\"u}ner, Laimas Jonaitis,
  Brihad Abhyankar, Jingjing Chen, Raquel Rogers, et~al.
\newblock Vedolizumab versus adalimumab for moderate-to-severe ulcerative
  colitis.
\newblock \emph{New England Journal of Medicine}, 381\penalty0 (13):\penalty0
  1215--1226, 2019.

\bibitem[Schulte et~al.(2014)Schulte, Tsiatis, Laber, and
  Davidian]{schulte2014}
Phillip~J. Schulte, Anastasios~A. Tsiatis, Eric~B. Laber, and Marie Davidian.
\newblock $\mathbf{Q}$- and $\mathbf{A}$-learning methods for estimating
  optimal dynamic treatment regimes.
\newblock \emph{Statist. Sci.}, 29\penalty0 (4):\penalty0 640--661, 11 2014.
\newblock \doi{10.1214/13-STS450}.
\newblock URL \url{https://doi.org/10.1214/13-STS450}.

\bibitem[Scornet et~al.(2015)Scornet, Biau, and Vert]{scornet2015}
Erwan Scornet, G{\'e}rard Biau, and Jean-Philippe Vert.
\newblock Consistency of random forests.
\newblock \emph{Annals of Statistics}, 43\penalty0 (4):\penalty0 1716, 2015.
\newblock ISSN 00905364.
\newblock URL \url{http://search.proquest.com/docview/1787036058/}.

\bibitem[Sutton(2018)]{sutton2018}
Richard~S. Sutton.
\newblock \emph{Reinforcement learning : an introduction}.
\newblock Adaptive computation and machine learning. The MIT Press, Cambridge,
  Massachusetts ; London, England, second edition. edition, 2018.
\newblock ISBN 9780262039246.

\bibitem[Thomas and Brunskill(2016)]{WDR}
Philip~S. Thomas and Emma Brunskill.
\newblock Data-efficient off-policy policy evaluation for reinforcement
  learning.
\newblock 2016.

\bibitem[Tsiatis(2006)]{tsiatis2006}
Anastasios~A Tsiatis.
\newblock \emph{Semiparametric Theory and Missing Data}.
\newblock Springer Series in Statistics. Springer New York, New York, NY, 2006.
\newblock ISBN 9780387324487.

\bibitem[Tsybakov(2009)]{tsybakov2009}
Alexandre~B. Tsybakov.
\newblock \emph{Introduction to Nonparametric Estimation}.
\newblock Springer Series in Statistics. Springer New York, New York, NY, 2009.
\newblock ISBN 978-0-387-79051-0.

\bibitem[Vaart(1998)]{vaart_donsker}
A.~W. van~der Vaart.
\newblock \emph{Asymptotic statistics}.
\newblock Cambridge series on statistical and probabilistic mathematics.
  Cambridge University Press, Cambridge, UK ; New York, NY, USA, 1998.
\newblock ISBN 0521496039.

\bibitem[van~der Vaart and Wellner(1996)]{vanderVaartAadW1996WCaE}
Aad~W van~der Vaart and Jon~A Wellner.
\newblock \emph{Weak Convergence and Empirical Processes: With Applications to
  Statistics}.
\newblock Springer Series in Statistics. Springer New York, New York, 1996.
\newblock ISBN 9781475725476.

\bibitem[Van Der~Vaart and Wellner(2007)]{wellner_emp}
Aad~W. Van Der~Vaart and Jon~A. Wellner.
\newblock Empirical processes indexed by estimated functions.
\newblock \emph{Lecture Notes-Monograph Series}, 55:\penalty0 234--252, 2007.
\newblock ISSN 07492170.

\bibitem[Wasserman and Lafferty(2008)]{Wasserman2007}
Larry Wasserman and John~D. Lafferty.
\newblock Statistical analysis of semi-supervised regression.
\newblock In J.~C. Platt, D.~Koller, Y.~Singer, and S.~T. Roweis, editors,
  \emph{Advances in Neural Information Processing Systems 20}, pages 801--808.
  Curran Associates, Inc., 2008.
\newblock URL
  \url{http://papers.nips.cc/paper/3376-statistical-analysis-of-semi-supervised-regression.pdf}.

\bibitem[Watkins(1989)]{Watkins1989}
Christopher John Cornish~Hellaby Watkins.
\newblock Learning from delayed rewards, 1989.

\bibitem[Zhang et~al.(2019)Zhang, Cai, Yu, Cho, Hong, Sun, Huang, Ho,
  Ananthakrishnan, Xia, et~al.]{zhang2019high}
Yichi Zhang, Tianrun Cai, Sheng Yu, Kelly Cho, Chuan Hong, Jiehuan Sun, Jie
  Huang, Yuk-Lam Ho, Ashwin~N Ananthakrishnan, Zongqi Xia, et~al.
\newblock High-throughput phenotyping with electronic medical record data using
  a common semi-supervised approach (phecap).
\newblock \emph{Nature Protocols}, 14\penalty0 (12):\penalty0 3426--3444, 2019.

\bibitem[Zhao et~al.(2015)Zhao, Zeng, Laber, and Kosorok]{Zhao2015}
Ying-Qi Zhao, Donglin Zeng, Eric~B Laber, and Michael~R Kosorok.
\newblock New statistical learning methods for estimating optimal dynamic
  treatment regimes.
\newblock \emph{Journal of the American Statistical Association}, 110\penalty0
  (510):\penalty0 583--598, 2015.
\newblock ISSN 0162-1459.
\newblock URL
  \url{http://www.tandfonline.com/doi/abs/10.1080/01621459.2014.937488}.

\bibitem[Zhixing and Shaohong(2011)]{Wang2011}
Wang Zhixing and Chen Shaohong.
\newblock Web page classification based on semi-supervised na{\"\i}ve bayesian
  em algorithm.
\newblock In \emph{2011 IEEE 3rd International Conference on Communication
  Software and Networks}, pages 242--245. IEEE, 2011.
\newblock ISBN 9781612844855.

\bibitem[Zhu et~al.(2019)Zhu, Zeng, and Song]{zhu2019}
Wensheng Zhu, Donglin Zeng, and Rui Song.
\newblock Proper inference for value function in high-dimensional q-learning
  for dynamic treatment regimes.
\newblock \emph{Journal of the American Statistical Association}, 114\penalty0
  (527):\penalty0 1404--1417, 2019.
\newblock ISSN 0162-1459.
\newblock URL
  \url{http://www.tandfonline.com/doi/abs/10.1080/01621459.2018.1506341}.

\bibitem[Zhu(2008)]{zhu05}
Xiaojin Zhu.
\newblock Semi-supervised learning literature survey.
\newblock Technical Report 1530, Computer Sciences, University of
  Wisconsin-Madison, 2008.

\end{thebibliography}
\newpage

\appendix
\counterwithin{figure}{section}
\counterwithin{table}{section}
\section{Simulation Results for Alternative Settings}\label{appendix: alt sims}
In this Section we provide additional results for data generating scenarios described in Section \ref{section: sims and application}. Tables \ref{tab:simtheta small n} and \ref{tab:simtheta small n} contain results for estimation of $Q$ function parameters for the EHR simulation setting for small and large sample sizes respectively. Table \ref{tab:simtheta} contains the complete parameter results for the continuous data generating setting for both small and large samples.

\begin{table}[ht]
\centering
\centerline{(a) $n=135$ and $N=1272$}
\scalebox{0.8}{
\begin{tabular}{@{}lccccccclcccccl@{}}
\toprule
 & \multicolumn{2}{c}{Supervised} &  & \multicolumn{11}{c}{Semi-Supervised} \\ \cmidrule(lr){2-3} \cmidrule(lr){5-15} 
 & \multicolumn{2}{c}{} &  & \multicolumn{5}{c}{Random Forests} &  & \multicolumn{5}{c}{Basis Expansion} \\ \cmidrule(l){5-9} \cmidrule(l){11-15} 
Parameter & Bias & ESE &  & Bias & ESE & ASE & CovP & RE &  & Bias & ESE & ASE & CovP & RE \\ \cmidrule(l){1-3} \cmidrule(l){5-9} \cmidrule(l){11-15} 
$\beta_{11}$=1.2 & 0.05 & 0.09 &  & 0.03 & 0.06 & 0.05 & 0.88 & 1.65 &  & 0.03 & 0.06 & 0.05 & 0.89 & 1.60 \\ 
  $\beta_{12}$=0 & 0.00 & 0.06 &  & 0.00 & 0.04 & 0.04 & 0.90 & 1.57 &  & 0.00 & 0.04 & 0.04 & 0.91 & 1.62 \\ 
  $\beta_{13}$=-0.4 & 0& 0.07 &  & -0.01 & 0.05 & 0.04 & 0.92 & 1.53 &  & 0& 0.05 & 0.05 & 0.93 & 1.56 \\ 
  $\beta_{14}$=-0.3 & 0.00 & 0.07 &  & -0.01 & 0.04 & 0.04 & 0.93 & 1.67 &  & 0& 0.04 & 0.04 & 0.93 & 1.64 \\ 
  $\beta_{15}$=0 & 0.00 & 0.08 &  & 0.00 & 0.04 & 0.04 & 0.93 & 1.69 &  & 0.00 & 0.04 & 0.04 & 0.92 & 1.69 \\ 
  $\beta_{16}$=0 & 0& 0.07 &  & 0.00 & 0.04 & 0.04 & 0.93 & 1.67 &  & 0.00 & 0.04 & 0.04 & 0.93 & 1.74 \\ 
  $\beta_{17}$=0 & 0.00 & 0.08 &  & 0.00 & 0.05 & 0.04 & 0.92 & 1.62 &  & 0.00 & 0.05 & 0.04 & 0.92 & 1.62 \\ 
  $\gamma_{11}$=0.1 & -0.01 & 0.14 &  & 0.00 & 0.09 & 0.08 & 0.91 & 1.55 &  & 0& 0.09 & 0.07 & 0.89 & 1.55 \\ 
  $\gamma_{12}$=0 & -0.01 & 0.09 &  & -0.01 & 0.06 & 0.05 & 0.92 & 1.53 &  & -0.01 & 0.06 & 0.06 & 0.93 & 1.51 \\ 
  $\gamma_{13}$=0 & 0& 0.08 &  & 0& 0.05 & 0.05 & 0.93 & 1.58 &  & 0 & 0.05 & 0.05 & 0.94 & 1.58 \\ 
  $\gamma_{14}$=0 & 0 & 0.08 &  & 0.00 & 0.05 & 0.05 & 0.93 & 1.58 &  & 0 & 0.05 & 0.05 & 0.93 & 1.58 \\ 
  $\gamma_{15}$=0 & 0.00 & 0.09 &  & 0.00 & 0.05 & 0.05 & 0.92 & 1.59 &  & 0 & 0.05 & 0.05 & 0.95 & 1.65 \\ 
  $\gamma_{16}$=-0.1 & 0 & 0.09 &  & 0 & 0.06 & 0.05 & 0.92 & 1.52 &  & 0 & 0.06 & 0.05 & 0.93 & 1.49 \\ 
  $\beta_{21}$=0.1 & 0.00 & 0.10 &  & -0.01 & 0.15 & 0.13 & 0.91 & 0.71 &  & 0 & 0.14 & 0.13 & 0.93 & 0.75 \\ 
  $\beta_{22}$=0.6 & 0 & 0.13 &  & 0.01 & 0.11 & 0.10 & 0.91 & 1.16 &  & 0 & 0.11 & 0.11 & 0.94 & 1.18 \\ 
  $\beta_{23}$=0 & 0.00 & 0.06 &  & 0.00 & 0.04 & 0.04 & 0.93 & 1.44 &  & 0.00 & 0.04 & 0.04 & 0.93 & 1.47 \\ 
  $\beta_{24}$=-0.2 & 0.00 & 0.06 &  & 0 & 0.05 & 0.04 & 0.89 & 1.16 &  & 0 & 0.05 & 0.05 & 0.93 & 1.20 \\ 
  $\beta_{25}$=-0.2 & 0.00 & 0.05 &  & 0 & 0.05 & 0.04 & 0.90 & 1.13 &  & 0 & 0.04 & 0.04 & 0.92 & 1.18 \\ 
  $\beta_{26}$=0 & 0.00 & 0.04 &  & 0.00 & 0.02 & 0.02 & 0.94 & 1.50 &  & 0.00 & 0.02 & 0.02 & 0.94 & 1.50 \\ 
  $\beta_{27}$=0 & 0.00 & 0.04 &  & 0.00 & 0.03 & 0.02 & 0.94 & 1.52 &  & 0.00 & 0.02 & 0.02 & 0.94 & 1.58 \\ 
  $\beta_{28}$=0 & 0.00 & 0.05 &  & 0.00 & 0.04 & 0.03 & 0.92 & 1.49 &  & 0.00 & 0.04 & 0.03 & 0.92 & 1.49 \\ 
  $\beta_{29}$=0 & 0 & 0.12 &  & 0.00 & 0.08 & 0.07 & 0.91 & 1.49 &  & 0.00 & 0.08 & 0.08 & 0.93 & 1.52 \\ 
  $\beta_{210}$=-0.2 & 0.00 & 0.11 &  & 0 & 0.07 & 0.07 & 0.94 & 1.54 &  & 0.00 & 0.07 & 0.07 & 0.94 & 1.57 \\ 
  $\beta_{211}$=-0.1 & 0.01 & 0.11 &  & 0.00 & 0.07 & 0.07 & 0.94 & 1.54 &  & 0.00 & 0.07 & 0.07 & 0.93 & 1.56 \\ 
  $\gamma_{21}$=0.1 & 0.01 & 0.16 &  & 0.01 & 0.11 & 0.10 & 0.92 & 1.47 &  & 0.01 & 0.11 & 0.10 & 0.94 & 1.51 \\ 
  $\gamma_{22}$=0 & 0.00 & 0.08 &  & 0.00 & 0.06 & 0.05 & 0.94 & 1.47 &  & 0.00 & 0.06 & 0.06 & 0.93 & 1.50 \\ 
  $\gamma_{23}$=0 & 0 & 0.08 &  & 0.00 & 0.06 & 0.05 & 0.94 & 1.45 &  & 0.00 & 0.05 & 0.05 & 0.94 & 1.48 \\ 
  $\gamma_{24}$=0 & 0 & 0.07 &  & 0.00 & 0.05 & 0.05 & 0.93 & 1.43 &  & 0.00 & 0.05 & 0.05 & 0.94 & 1.46 \\ 
  $\gamma_{25}$=0 & 0 & 0.07 &  & 0 & 0.05 & 0.05 & 0.94 & 1.48 &  & 0 & 0.05 & 0.05 & 0.94 & 1.48 \\ 
  $\gamma_{26}$=0 & 0 & 0.18 &  & 0 & 0.12 & 0.11 & 0.92 & 1.45 &  & 0 & 0.12 & 0.11 & 0.94 & 1.52 \\ 
  $\gamma_{27}$=-0.2 & -0.01 & 0.16 &  & -0.01 & 0.11 & 0.10 & 0.93 & 1.47 &  & -0.01 & 0.11 & 0.10 & 0.94 & 1.48 \\ 
  $\gamma_{28}$=-0.1 & -0.01 & 0.15 &  & -0.01 & 0.10 & 0.10 & 0.94 & 1.54 &  & -0.01 & 0.10 & 0.10 & 0.94 & 1.57 \\  \bottomrule
\end{tabular}}
\caption{Bias, empirical standard error (ESE) of the supervised and the SSL estimators with either random forest imputation or basis expansion imputation strategies for $\bthetabar$ when (a) $n=135$ and $N=1272$ under the EHR simulation setting. For the SSL estimators, we also obtain the average of the estimated standard errors (ASE) as well as the empirical coverage probabilities (CovP) of the 95\% confidence intervals.}
\label{tab:simtheta small n}
\end{table}

\begin{table}[ht]
\centering
\centerline{(b) $n=500$ and $N=10,000$}
\scalebox{0.8}{
\begin{tabular}{@{}lccccccclcccccl@{}}
\toprule
 & \multicolumn{2}{c}{Supervised} &  & \multicolumn{11}{c}{Semi-Supervised} \\ \cmidrule(lr){2-3} \cmidrule(lr){5-15} 
 & \multicolumn{2}{c}{} &  & \multicolumn{5}{c}{Random Forests} &  & \multicolumn{5}{c}{Basis Expansion} \\ \cmidrule(l){5-9} \cmidrule(l){11-15} 
Parameter & Bias & ESE &  & Bias & ESE & ASE & CovP & RE &  & Bias & ESE & ASE & CovP & RE \\ \cmidrule(l){1-3} \cmidrule(l){5-9} \cmidrule(l){11-15} 
$\beta_{11}$=1.2 & 0.01 & 0.05 &  & 0.00 & 0.02 & 0.02 & 0.91 & 2.09 &  & 0.00 & 0.02 & 0.02 & 0.92 & 2.00 \\ 
  $\beta_{12}$=0 & 0.00 & 0.03 &  & 0.00 & 0.01 & 0.01 & 0.91 & 2.07 &  & 0.00 & 0.01 & 0.01 & 0.92 & 2.07 \\ 
  $\beta_{13}$=-0.4 & 0.00 & 0.04 &  & 0 & 0.02 & 0.02 & 0.92 & 2.05 &  & 0 & 0.02 & 0.02 & 0.92 & 2.05 \\ 
  $\beta_{14}$=-0.3 & 0 & 0.04 &  & 0 & 0.02 & 0.01 & 0.92 & 2.06 &  & 0 & 0.02 & 0.02 & 0.92 & 2.06 \\ 
  $\beta_{15}$=0 & 0.00 & 0.04 &  & 0 & 0.02 & 0.02 & 0.94 & 2.18 &  & 0 & 0.02 & 0.02 & 0.94 & 2.06 \\ 
  $\beta_{16}$=0 & 0 & 0.04 &  & 0.00 & 0.02 & 0.02 & 0.94 & 2.18 &  & 0.00 & 0.02 & 0.02 & 0.94 & 2.18 \\ 
  $\beta_{17}$=0 & 0.00 & 0.04 &  & 0.00 & 0.02 & 0.02 & 0.93 & 2.06 &  & 0.00 & 0.02 & 0.02 & 0.94 & 2.06 \\ 
  $\gamma_{11}$=0.1 & 0 & 0.07 &  & 0 & 0.03 & 0.03 & 0.91 & 2.00 &  & 0 & 0.03 & 0.03 & 0.91 & 2.00 \\ 
  $\gamma_{12}$=0 & -0.01 & 0.05 &  & 0 & 0.02 & 0.02 & 0.90 & 2.00 &  & 0 & 0.02 & 0.02 & 0.89 & 2.00 \\ 
  $\gamma_{13}$=0 & 0.00 & 0.04 &  & 0.00 & 0.02 & 0.02 & 0.92 & 2.00 &  & 0.00 & 0.02 & 0.02 & 0.91 & 1.90 \\ 
  $\gamma_{14}$=0 & 0 & 0.04 &  & 0.00 & 0.02 & 0.02 & 0.94 & 2.00 &  & 0.00 & 0.02 & 0.02 & 0.94 & 1.90 \\ 
  $\gamma_{15}$=0 & 0.00 & 0.04 &  & 0.00 & 0.02 & 0.02 & 0.94 & 2.16 &  & 0.00 & 0.02 & 0.02 & 0.94 & 2.05 \\ 
  $\gamma_{16}$=-0.1 & 0 & 0.04 &  & 0 & 0.02 & 0.02 & 0.93 & 2.05 &  & 0 & 0.02 & 0.02 & 0.92 & 1.95 \\ 
  $\beta_{21}$=0.1 & 0.00 & 0.05 &  & 0.00 & 0.04 & 0.04 & 0.95 & 1.16 &  & 0.00 & 0.04 & 0.05 & 0.96 & 1.13 \\ 
  $\beta_{22}$=0.6 & 0 & 0.07 &  & 0 & 0.04 & 0.04 & 0.95 & 1.74 &  & 0 & 0.04 & 0.04 & 0.96 & 1.69 \\ 
  $\beta_{23}$=0 & 0.00 & 0.03 &  & 0.00 & 0.01 & 0.01 & 0.94 & 1.87 &  & 0.00 & 0.01 & 0.01 & 0.94 & 1.87 \\ 
  $\beta_{24}$=-0.2 & 0.00 & 0.03 &  & 0.00 & 0.02 & 0.02 & 0.94 & 1.71 &  & 0.00 & 0.02 & 0.02 & 0.95 & 1.71 \\ 
  $\beta_{25}$=-0.2 & 0.00 & 0.02 &  & 0 & 0.01 & 0.01 & 0.94 & 1.60 &  & 0 & 0.01 & 0.01 & 0.95 & 1.60 \\ 
  $\beta_{26}$=0 & 0.00 & 0.02 &  & 0.00 & 0.01 & 0.01 & 0.92 & 1.90 &  & 0.00 & 0.01 & 0.01 & 0.93 & 1.90 \\ 
  $\beta_{27}$=0 & 0.00 & 0.02 &  & 0.00 & 0.01 & 0.01 & 0.94 & 1.89 &  & 0.00 & 0.01 & 0.01 & 0.94 & 1.89 \\ 
  $\beta_{28}$=0 & 0.00 & 0.03 &  & 0.00 & 0.01 & 0.01 & 0.94 & 1.92 &  & 0.00 & 0.01 & 0.01 & 0.94 & 1.92 \\ 
  $\beta_{29}$=0 & 0.00 & 0.06 &  & 0.00 & 0.03 & 0.03 & 0.92 & 1.94 &  & 0.00 & 0.03 & 0.03 & 0.93 & 1.88 \\ 
  $\beta_{210}$=-0.2 & 0 & 0.05 &  & 0 & 0.03 & 0.03 & 0.94 & 2.00 &  & 0.00 & 0.03 & 0.03 & 0.94 & 2.00 \\ 
  $\beta_{211}$=-0.1 & 0.00 & 0.06 &  & 0.00 & 0.03 & 0.03 & 0.94 & 2.00 &  & 0.00 & 0.03 & 0.03 & 0.94 & 2.00 \\ 
  $\gamma_{21}$=0.1 & 0 & 0.08 &  & 0.00 & 0.04 & 0.04 & 0.94 & 1.98 &  & 0.00 & 0.04 & 0.04 & 0.94 & 1.98 \\ 
  $\gamma_{22}$=0 & 0.00 & 0.04 &  & 0.00 & 0.02 & 0.02 & 0.93 & 1.95 &  & 0.00 & 0.02 & 0.02 & 0.93 & 1.86 \\ 
  $\gamma_{23}$=0 & 0 & 0.04 &  & 0 & 0.02 & 0.02 & 0.94 & 1.81 &  & 0 & 0.02 & 0.02 & 0.93 & 1.90 \\ 
  $\gamma_{24}$=0 & 0 & 0.03 &  & 0.00 & 0.02 & 0.02 & 0.94 & 1.83 &  & 0.00 & 0.02 & 0.02 & 0.95 & 1.83 \\ 
  $\gamma_{25}$=0 & 0 & 0.04 &  & 0 & 0.02 & 0.02 & 0.94 & 1.84 &  & 0 & 0.02 & 0.02 & 0.94 & 1.84 \\ 
  $\gamma_{26}$=0 & -0.01 & 0.09 &  & 0 & 0.04 & 0.04 & 0.93 & 2.00 &  & 0 & 0.04 & 0.04 & 0.93 & 2.00 \\ 
  $\gamma_{27}$=-0.2 & 0.01 & 0.08 &  & 0.00 & 0.04 & 0.04 & 0.94 & 1.98 &  & 0.00 & 0.04 & 0.04 & 0.94 & 1.98 \\ 
  $\gamma_{28}$=-0.1 & 0.00 & 0.08 &  & 0.00 & 0.04 & 0.04 & 0.94 & 1.95 &  & 0.00 & 0.04 & 0.04 & 0.94 & 1.95 \\  \bottomrule
\end{tabular}}
\caption{Bias, empirical standard error (ESE) of the supervised and the SSL estimators with either random forest imputation or basis expansion imputation strategies for $\bthetabar$ when (b) $n=500$ and $N=10,000$ under the EHR simulation setting. For the SSL estimators, we also obtain the average of the estimated standard errors (ASE) as well as the empirical coverage probabilities (CovP) of the 95\% confidence intervals.}
\label{tab:simtheta large n}
\end{table}

\begin{table}[ht]
\centering
\centerline{(a) $n=135$ and $N=1272$}
\scalebox{0.8}{
\begin{tabular}{@{}lccccccclcccccl@{}}
\toprule
 & \multicolumn{2}{c}{Supervised} &  & \multicolumn{11}{c}{Semi-Supervised} \\ \cmidrule(lr){2-3} \cmidrule(lr){5-15} 
 & \multicolumn{2}{c}{} &  & \multicolumn{5}{c}{Random Forests} &  & \multicolumn{5}{c}{Basis Expansion} \\ \cmidrule(l){5-9} \cmidrule(l){11-15} 
Parameter & Bias & ESE &  & Bias & ESE & ASE & CovP & RE &  & Bias & ESE & ASE & CovP & RE \\ \cmidrule(l){1-3} \cmidrule(l){5-9} \cmidrule(l){11-15} 
$\beta_{11}$=4.9 & 0.04 & 0.34 &  & 0.01 & 0.22 & 0.18 & 0.91 & 1.58 &  & 0.01 & 0.20 & 0.17 & 0.90 & 1.70 \\ 
  $\beta_{12}$=1.1 & -0.03 & 0.42 &  & 0.00 & 0.26 & 0.24 & 0.94 & 1.61 &  & 0.01 & 0.25 & 0.23 & 0.92 & 1.68 \\ 
  $\gamma_{11}$=1.4 & -0.03 & 0.41 &  & 0.00 & 0.26 & 0.24 & 0.93 & 1.57 &  & 0.00 & 0.24 & 0.23 & 0.93 & 1.68 \\ 
  $\gamma_{12}$=-2.6 & 0.04 & 0.58 &  & -0.01 & 0.36 & 0.34 & 0.94 & 1.61 &  & -0.02 & 0.35 & 0.31 & 0.90 & 1.69 \\ 
  $\beta_{21}$=0.1 & 0.00 & 0.10 &  & 0.00 & 0.13 & 0.12 & 0.94 & 0.82 &  & 0.00 & 0.16 & 0.17 & 0.94 & 0.64 \\ 
  $\beta_{22}$=3 & 0.00 & 0.33 &  & 0.00 & 0.24 & 0.23 & 0.93 & 1.39 &  & 0 & 0.26 & 0.25 & 0.93 & 1.30 \\ 
  $\beta_{23}$=0 & -0.01 & 0.34 &  & -0.01 & 0.24 & 0.22 & 0.93 & 1.43 &  & -0.01 & 0.24 & 0.24 & 0.94 & 1.39 \\ 
  $\beta_{24}$=0.1 & 0 & 0.43 &  & 0 & 0.29 & 0.28 & 0.94 & 1.49 &  & 0 & 0.30 & 0.29 & 0.94 & 1.46 \\ 
  $\beta_{25}$=-0.5 & 0.01 & 0.15 &  & 0 & 0.09 & 0.09 & 0.93 & 1.62 &  & 0.00 & 0.09 & 0.09 & 0.93 & 1.71 \\ 
  $\beta_{26}$=-0.4 & 0.03 & 0.48 &  & 0.01 & 0.37 & 0.35 & 0.93 & 1.29 &  & 0.01 & 0.41 & 0.40 & 0.94 & 1.16 \\ 
  $\gamma_{21}$=0.8 & 0.00 & 0.34 &  & 0.01 & 0.21 & 0.20 & 0.93 & 1.61 &  & 0.00 & 0.20 & 0.19 & 0.94 & 1.71 \\ 
  $\gamma_{22}$=0.2 & -0.02 & 0.45 &  & -0.01 & 0.28 & 0.28 & 0.95 & 1.60 &  & -0.01 & 0.27 & 0.26 & 0.94 & 1.70 \\ 
  $\gamma_{23}$=0.5 & 0 & 0.18 &  & 0.01 & 0.11 & 0.11 & 0.94 & 1.59 &  & 0.00 & 0.11 & 0.11 & 0.94 & 1.68 \\ \bottomrule
\end{tabular}}\vspace{.1in}
\centerline{(b) $n=500$ and $N=10,000$}
\scalebox{0.8}{
\begin{tabular}{@{}lccccccclcccccl@{}}
\toprule
 & \multicolumn{2}{c}{Supervised} &  & \multicolumn{11}{c}{Semi-Supervised} \\ \cmidrule(lr){2-3} \cmidrule(lr){5-15} 
 & \multicolumn{2}{c}{} &  & \multicolumn{5}{c}{Random Forests} &  & \multicolumn{5}{c}{Basis Expansion} \\ \cmidrule(l){5-9} \cmidrule(l){11-15} 
Parameter & Bias & ESE &  & Bias & ESE & ASE & CovP & RE &  & Bias & ESE & ASE & CovP & RE \\ \cmidrule(l){1-3} \cmidrule(l){5-9} \cmidrule(l){11-15} 
$\beta_{11}$=4.9 & 0.00 & 0.17 &  & 0 & 0.10 & 0.09 & 0.91 & 1.72 &  & 0 & 0.10 & 0.08 & 0.92 & 1.79 \\ 
  $\beta_{12}$=1.1 & 0 & 0.22 &  & 0.00 & 0.12 & 0.11 & 0.93 & 1.80 &  & 0.00 & 0.12 & 0.11 & 0.93 & 1.86 \\ 
  $\gamma_{11}$=1.4 & 0.01 & 0.22 &  & 0.01 & 0.12 & 0.11 & 0.92 & 1.76 &  & 0.01 & 0.12 & 0.11 & 0.92 & 1.80 \\ 
  $\gamma_{12}$=-2.6 & 0 & 0.29 &  & 0 & 0.17 & 0.16 & 0.93 & 1.73 &  & -0.01 & 0.16 & 0.15 & 0.93 & 1.80 \\ 
  $\beta_{21}$=0.1 & -0.01 & 0.05 &  & 0 & 0.05 & 0.05 & 0.94 & 1.06 &  & 0 & 0.07 & 0.08 & 0.95 & 0.74 \\ 
  $\beta_{22}$=3 & 0.00 & 0.17 &  & 0.00 & 0.11 & 0.10 & 0.93 & 1.60 &  & 0.00 & 0.12 & 0.11 & 0.94 & 1.45 \\ 
  $\beta_{23}$=0 & 0.00 & 0.17 &  & 0.00 & 0.10 & 0.10 & 0.95 & 1.66 &  & 0.00 & 0.11 & 0.11 & 0.95 & 1.54 \\ 
  $\beta_{24}$=0.1 & 0.02 & 0.23 &  & 0.01 & 0.13 & 0.12 & 0.94 & 1.77 &  & 0.01 & 0.14 & 0.13 & 0.94 & 1.68 \\ 
  $\beta_{25}$=-0.5 & 0.00 & 0.07 &  & 0.00 & 0.04 & 0.04 & 0.93 & 1.74 &  & 0.00 & 0.04 & 0.04 & 0.94 & 1.78 \\ 
  $\beta_{26}$=-0.4 & -0.01 & 0.25 &  & -0.01 & 0.17 & 0.15 & 0.93 & 1.51 &  & -0.01 & 0.19 & 0.18 & 0.94 & 1.31 \\ 
  $\gamma_{21}$=0.8 & 0.00 & 0.17 &  & 0.00 & 0.10 & 0.09 & 0.93 & 1.80 &  & 0.00 & 0.09 & 0.09 & 0.93 & 1.86 \\ 
  $\gamma_{22}$=0.2 & -0.01 & 0.23 &  & 0 & 0.13 & 0.12 & 0.93 & 1.81 &  & 0 & 0.13 & 0.12 & 0.94 & 1.83 \\ 
  $\gamma_{23}$=0.5 & 0.00 & 0.09 &  & 0.00 & 0.05 & 0.05 & 0.94 & 1.78 &  & 0.00 & 0.05 & 0.05 & 0.95 & 1.81 \\ \bottomrule
\end{tabular}}
\caption{Bias, empirical standard error (ESE) of the supervised and the SSL estimators with either random forest imputation or basis expansion imputation strategies for $\bthetabar$ when (a) $n=135$ and $N=1272$ and (b) $n=500$ and $N=10,000$ under the continuous outcome simulation setting. For the SSL estimators, we also obtain the average of the estimated standard errors (ASE) as well as the empirical coverage probabilities (CovP) of the 95\% confidence intervals.}
\label{tab:simtheta}
\end{table}

\begin{figure}[ht]
	\centering
	\includegraphics[width=1\textwidth]{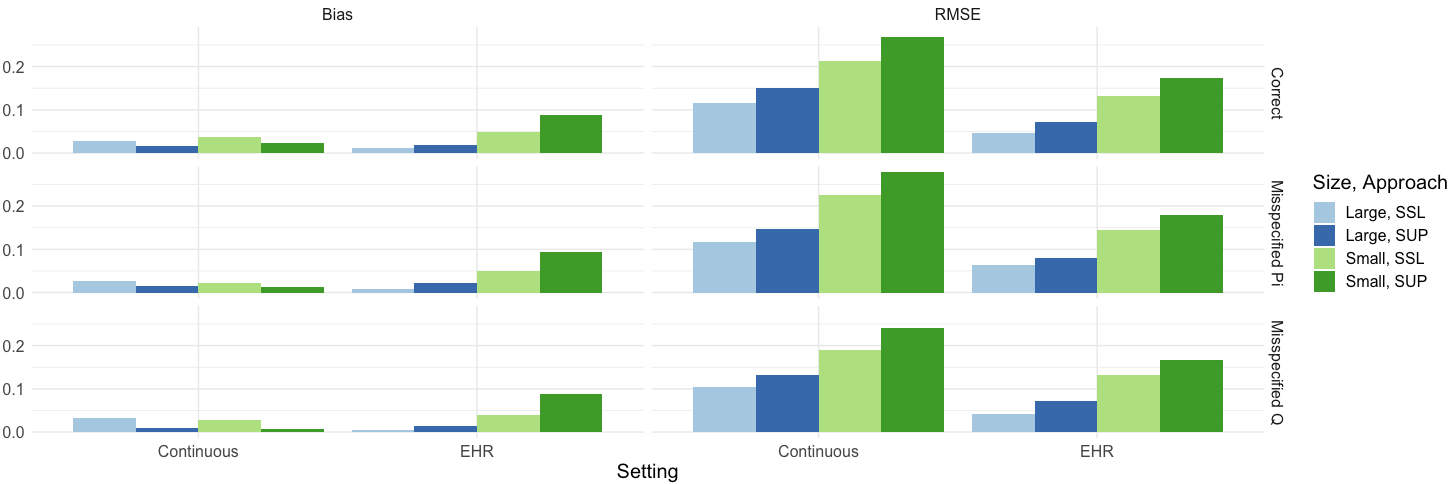}
	\caption{Monte Carlo estimates for doubly-robust value function estimation: $\Vhat\subSSLDR$, $\Vhat\subSUPDR$ under continuous, and EHR settings. Columns show bias and RMSE respectively, rows show different mis-specification scenarios. Results are shown for the large ($N=10,000$, $n=500$) and small data samples ($N=1,272$, $n=135$) for the continuous setting over 1,000 simulated datasets.
	}
	\label{fig_misspec_Qpi_sin} 
\end{figure}

\newpage

\section{Proof of Main Results}\label{sec:proof_main_results}

\subsection{Semi-supervised \texorpdfstring{$Q$}{Lg}-learning asymptotics}\label{section: Q learning result proofs}

In this section we first show the proofs for the theoretical results on the generalized semi-supervised $Q$-learning shown in section \ref{theory}. 

\subsubsection{Proofs for theoretical results for \texorpdfstring{$Q$}{Lg}-learning in section \ref{theory}}\label{unbiased_prfs} 
We first define $\btheta_{2-}\equiv(\bbeta_{22}\trans,\bgamma_2\trans)\trans$, and $\hat\Delta_s\supnk(\bUvec)\equiv\mhat_s\supnk(\bUvec)-m_s(\bUvec)$, $s\in\{2,3,22,23\}$, and note that from Assumptions \ref{assumption: covariates}, \ref{assumption: Q imputation} \& \ref{assumption SS linear model} it follows that:

\begin{equation}\label{deltas bound}
\begin{split}
\sum\limits_{k=1}^{K}& \sup\limits_{\bUvec}\left|\hat\Delta_{2t}\supnk(\bUvec)\right|=o_\Pbb(1)\text{ for }t=2,3,\\
\sum\limits_{k=1}^{K}& \sup\limits_{\bXvec,\bUvec}\|\bXvec\hat\Delta_2\supnk(\bUvec)\|=o_\Pbb(1),\\
\sum\limits_{k=1}^{K}& \sup\limits_{,\bX_2\bUvec}\|\bX_2\hat\Delta_3\supnk(\bUvec)\|=o_\Pbb(1),
\end{split}
\end{equation}

Next we remind that, to ensure the validity of the SSL algorithm from the refitted imputation model, the final imputation models for $\{Y_t, Y_{2t}, t=2,3\}$, denoted by $\{\mubar_t(\bUvec), \mubar_{2t}, t=2,3\}$, need to satisfy the constraints shown in Section \ref{sec: ssQ}:
\begin{equation}\label{impute_constraint}
\begin{alignedat}{3}
 \Ebb\left[\bXvec\{Y_2 - \mubar_2(\bUvec)\}\right] = \bzero, &\quad&
\Ebb\left\{Y_2^2 - \mubar_{22}(\bUvec)\right\} =&0 &, \\
\Ebb\left[\bX_2\{Y_3 - \mubar_3(\bUvec)\}\right] = \bzero, &\quad&
\Ebb\left\{Y_2Y_3 - \mubar_{23}(\bUvec)\right\} =&0 &.
\end{alignedat}
\end{equation}
where $\bXvec = (1,\bX_1\trans,\bX_2\trans)\trans$.
\begin{proof}[Proof of Theorem \ref{theorem: unbiased theta2}]

Recall the estimating equation for stage 2 regression in Section \ref{section: SSQL} is 
    \begin{align*}
    \begin{split}
        \Pbb_N&
        \begin{bmatrix}
        \muhat_{23}(\bUvec)-\hat\beta_{21}
        \muhat_{22}(\bUvec)-\muhat_2(\bUvec)\bX_2\trans\bthetahat_{2-}\\
        \bX_2\left\{\muhat_3(\bUvec)-\hat\beta_{21}
        \muhat_2(\bUvec)-\bX_2\trans\bthetahat_{2-}\right\}
        \end{bmatrix}
        =\textbf{0}.
    \end{split}
    \end{align*}
Centering the above at $\bthetabar_2$ we get
\begin{align}\label{stage2_expansion}
\Pbb_N
\begin{bmatrix}
\muhat_{22}(\bUvec),\muhat_2(\bUvec)\bX_2\trans\\
\bX_2
\muhat_2(\bUvec),\bX_2\bX_2\trans
\end{bmatrix}
(\bthetahat_2-\bthetabar_2)
=
\Pbb_N&
\begin{bmatrix}
\muhat_{23}(\bUvec)-\bar\beta_{21}
        \muhat_{22}(\bUvec)-\muhat_2(\bUvec)\bX_2\trans\bthetabar_{2-}\\
\bX_2\left\{\muhat_3(\bUvec)-\bar\beta_{21}
\muhat_2(\bUvec)-\bX_2\trans\bthetabar_{2-}\right\}
\end{bmatrix}.
\end{align}

Define
\begin{align*}
\mathcal{R}_{\mathcal{U}}=&
\Pbb_N
\begin{bmatrix}
\mubar_{23}(\bUvec)-
\bar\beta_{21}\mubar_{22}(\bUvec)-\mubar_2(\bUvec)\bX_2\trans\bthetabar_{2-}\\
\bX_2\left\{
\mubar_3(\bUvec)-\bar\beta_{21}\mubar_2(\bUvec)-\bX_2\trans\bthetabar_{2-}\right\}
\end{bmatrix},
\\
\hat{\mathcal{R}}_\mathbb{S}^{(K)}
=&
\Pbb_N
\begin{bmatrix}
\left\{\muhat_{23}(\bUvec)-\mubar_{23}(\bUvec)\right\}
-\bar\beta_{21}\left\{
\muhat_{22}(\bUvec)-\mubar_{22}(\bUvec)\right\}-\left\{\muhat_2(\bUvec)-\mubar_2(\bUvec)\right\}\bX_2\trans\bthetabar_{2-}\\
\bX_2\left\{\muhat_3(\bUvec)-\mubar_3(\bUvec)\right\}-\bar\beta_{21}\bX_2\left\{
\muhat_2(\bUvec)-\mubar_2(\bUvec)\right\}
\end{bmatrix},\\
\Gamma_{\mathcal{U}}
=&
\Pbb_N
\begin{bmatrix}
\mubar_{22}(\bUvec)&\mubar_2(\bUvec)\bX_2\trans\\
\mubar_2(\bUvec)\bX_2&\bX_2\bX_2\trans
\end{bmatrix},\\
\hat\Gamma_\mathbb{S}^{(K)}
=&
\Pbb_N
\begin{bmatrix}
\muhat_{22}(\bUvec)-\mubar_{22}(\bUvec)&\left\{\muhat_2(\bUvec)-\mubar_2(\bUvec)\right\}\bX_2\trans\\
\left\{\muhat_2(\bUvec)-\mubar_2(\bUvec)\right\}\bX_2&\bzero\\
\end{bmatrix},
\end{align*}

with these we can re-write equation \eqref{stage2_expansion} as
$\left(\Gamma_{\mathcal{U}}+\hat\Gamma_\mathbb{S}^{(K)}\right)
(\bthetahat_2-\bthetabar_2)
=
\mathcal{R}_{\mathcal{U}}+
\hat{\mathcal{R}}_\mathbb{S}^{(K)}$. We next deal with each term.\\

(I) We first consider $\hat{\mathcal{R}}_\mathbb{S}^{(K)}$, let 
\begin{align*}
\hat{\mathcal{S}}^{\eta}_\mathbb{S}=&
\Pbb_N
\begin{bmatrix}
\left(\etahat_{23}-\eta_{23}\right)
-\bar\beta_{21}(
\etahat_{22}-\eta_{22})-
(\bEtahat_2-\bEta_2)\trans\bX_2\bX_2\trans\bthetabar_{2-}\\
\bX_2\bX_2\trans
\left\{\left(\bEtahat_3-\bEta_3\right)-\bar\beta_{21}
\left(\bEtahat_2-\bEta_2\right)\right\}
\end{bmatrix}
\\
\hat{\mathcal{S}}^{(K)}_\mathbb{S}=&
\frac{1}{K}\sum_{k=1}^K
\Pbb_N
\begin{bmatrix}
\hat\Delta_{23}\supnk(\bUvec)
-\bar\beta_{21}\hat\Delta_{22}\supnk(\bUvec)-\hat\Delta_2\supnk(\bUvec)\bX_2\trans\bthetabar_{2-}\\
\bX_2\left\{\hat\Delta_3\supnk(\bUvec)-\bar\beta_{21}
\hat\Delta_2\supnk(\bUvec)\right\}
\end{bmatrix}
\\
\bar{\mathcal{S}}_k=&
\mathbb{E}_\mathcal{L}
\begin{bmatrix}
\hat\Delta_{23}\supnk(\bUvec)
-\bar\beta_{21}\hat\Delta_{22}\supnk(\bUvec)-\hat\Delta_2\supnk(\bUvec)\bX_2\trans\bthetabar_{2-}\\
\bX_2\left\{\hat\Delta_3\supnk(\bUvec)-\bar\beta_{21}
\hat\Delta_2\supnk(\bUvec)\right\}
\end{bmatrix}
\:\text{ for } k \in\{1,\dots,K\}.
\end{align*}
From \eqref{eta_EE_Q1} it follows that 
$\hat{\mathcal{R}}_\mathbb{S}^{(K)}=\hat{\mathcal{S}}^{\eta}_\mathbb{S}+\hat{\mathcal{S}}^{(K)}_\mathbb{S}.$
Next using \eqref{deltas bound}, Assumption \ref{assumption: Q imputation}, and Lemma \ref{lemm_chakrabortty} it follows that $\hat{\mathcal{S}}^{(K)}_\mathbb{S}=
\frac{1}{K}\sum_k\bar{\mathcal{S}}_k+O_\Pbb\left(N^{-\frac{1}{2}}\right)$, which lets us write $\hat{\mathcal{R}}_\mathbb{S}^{(K)}=\hat{\mathcal{S}}^{\eta}_\mathbb{S}+\frac{1}{K}\sum_k\bar{\mathcal{S}}_k+O_\Pbb\left(N^{-\frac{1}{2}}\right)$.\\ 
Now consider $\hat{\mathcal{S}}^{\eta}_\mathbb{S}$, note that by the central limit theorem (CLT) $\Pbb_n\bX_2\bX_2=\Ebb\bX_2\bX_2+O_\Pbb\left(n^{-\frac{1}{2}}\right)$. Thus using this, Slutsky's theorem and Assumption \ref{assumption: covariates}
\[
(\Pbb_n\bX_2\bX_2)^{-1}(\Pbb_N\bX_2\bX_2)=I+O_\Pbb\left(n^{-\frac{1}{2}}\right),
\]
then using \eqref{impute_constraint}, \eqref{eta_EE_Q1} and Assumption \ref{assumption: Q imputation} we can write
\begin{align*}
&\Pbb_N
\left\{(\bEtahat_2-\bEta_2)\trans\bX_2\bX_2\trans\bthetabar_{2-}\right\}\\
=&
\left[
\left(\Pbb_n\bX_2\bX_2\trans\right)^{-1}\frac{1}{n}\sum\limits_{k=1}^{K}\sum\limits_{i\in\mathcal{I}_k}\bX_{2i}\left\{Y_{2i}-\mubar_2(\bUvec_i)+m_2(\bUvec_i)-m_2\supnk(\bUvec_i)\right\}\right]
\trans
\Pbb_N(\bX_2\bX_2\trans)\btheta_{2-}\\
=&
\left[\Pbb_n\bX_2\trans\left\{Y_2-\mubar_2(\bUvec)\right\}+\frac{1}{n}\sum\limits_{k=1}^{K}\sum\limits_{i\in\mathcal{I}_k}\bX_{2i}\trans\hat\Delta_2\supnk(\bUvec_i)\right]
\left(\Pbb_n\bX_2\bX_2\trans\right)^{-1}\Pbb_N(\bX_2\bX_2\trans)\btheta_{2-}\\
=&
\Pbb_n\bX_2\btheta_{2-}\trans\left\{Y_2-\mubar_2(\bUvec)\right\}+\frac{1}{n}\sum\limits_{k=1}^{K}\sum\limits_{i\in\mathcal{I}_k}\bX_{2i}\trans\btheta_{2-}\hat\Delta_2\supnk(\bUvec_i)\\
+&
O_\Pbb\left(n^{-\frac{1}{2}}\right)\left[\Pbb_n\bX_2\btheta_{2-}\trans\left\{Y_2-\mubar_2(\bUvec)\right\}+\frac{1}{n}\sum\limits_{k=1}^{K}\sum\limits_{i\in\mathcal{I}_k}\bX_{2i}\trans\btheta_{2-}\hat\Delta_2\supnk(\bUvec_i)\right]\\
=&
\Pbb_n\bX_2\btheta_{2-}\trans\left\{Y_2-\mubar_2(\bUvec)\right\}+\frac{1}{n}\sum\limits_{k=1}^{K}\sum\limits_{i\in\mathcal{I}_k}\bX_{2i}\trans\btheta_{2-}\hat\Delta_2\supnk(\bUvec_i)+O_\Pbb\left(n^{-1}\right)+O_\Pbb\left(n^{-\frac{1}{2}}\right)o_\Pbb(1).
\end{align*} 
Analogous derivations for all terms in $\hat{\mathcal{S}}^{\eta}_\mathbb{S}$ gives us
\[
\hat{\mathcal{S}}^{\eta}_\mathbb{S}=\mathbb{T}_{\mathcal{L}}-\mathbb{T}_\mathcal{L}^{(K)}+O_\Pbb\left(n^{-1}\right)+O_\Pbb\left(n^{-\frac{1}{2}}\right)o_\Pbb(1),
\]

where

\begin{align*}
\mathbb{T}_{\mathcal{L}}=&
\Pbb_n
\begin{bmatrix}
\left\{Y_2Y_3-\mubar_{23}(\bUvec)\right\}-\bar\beta_{21}
\left\{Y_2^2-\mubar_{22}(\bUvec)\right\}-
\left\{Y_2-\mubar_2(\bUvec)\right\}\bX_2\trans\bthetabar_{2-}\\
\bX_2\left\{Y_3-\mubar_3(\bUvec)\right\}-\bar\beta_{21}
\bX_2\left\{Y_2-\mubar_2(\bUvec)\right\}
\end{bmatrix},
\\
\mathbb{T}_\mathcal{L}^{(K)}=&
\frac{1}{n}\sum_{k=1}^{K}\sum_{i\in\mathcal{I}_k}
\begin{bmatrix}
\hat\Delta_{23}\supnk(\bUvec_i)-\bar\beta_{21}
\hat\Delta_{22}\supnk(\bUvec_i)-
\hat\Delta_2\supnk(\bUvec_i)\bX_{2i}\trans\bthetabar_{2-}\\
\bX_{2i}\left\{\hat\Delta_3\supnk(\bUvec_i)-\bar\beta_{21}
\hat\Delta_2\supnk(\bUvec_i)
\right\}
\end{bmatrix}.
\end{align*}


From the above it follows that $\hat{\mathcal{R}}_\mathbb{S}^{(K)}=\mathbb{T}_\mathcal{L}-\mathbb{T}_\mathcal{L}^{(K)}+\frac{1}{K}\sum_k\bar{\mathcal{S}}_k+O_\Pbb\left(n^{-1}\right)+O_\Pbb\left(n^{-\frac{1}{2}}\right)o_\Pbb(1)$. Next by Assumption \ref{assumption: Q imputation} and using Lemma \ref{lemm_deltas} with $\hat C_{n,N}=1$, and setting functions $\hat l_n(\cdot)$, $\hat \pi_n(\cdot)$ to be the constant 1, and $f(\bX_2)=\bX_2$ to be the identity function, we have $\sqrt n\left(\mathbb{T}^{(K)}_\mathcal{L}-\frac{1}{K}\sum_k\bar{\mathcal{S}}_k\right)=O_\Pbb\left(c_{n^-_K}\right)$. Therefore $\hat{\mathcal{R}}_\mathbb{S}^{(K)}=\mathbb{T}_{\mathcal{L}}+O_\Pbb\left(n^{-\frac{1}{2}}c_{n^-_K}\right)$.
 
(II) Now we consider $\mathcal{R}_{\mathcal{U}}$, from the CLT, assuming working model \eqref{linear_Qs}, as constraints \eqref{impute_constraint} are satisfied it follows that
\[
\mathcal{R}_{\mathcal{U}}=\mathbb{E}
\begin{bmatrix}
\mubar_{23}(\bUvec)-\bar\beta_{21}\mubar_{22}(\bUvec)-\mubar_2(\bUvec)\bX_2\trans\bthetabar_{2-}\\
\bX_2\{\mubar_3(\bUvec)-\bar\beta_{21}\mubar_2(\bUvec)-\bX_2\trans\bthetabar_{2-}\}
\end{bmatrix}
+O_\Pbb\left(N^{-\frac{1}{2}}\right)=\pmb 1O_\Pbb\left(N^{-\frac{1}{2}}\right).
\]
(III) Next we focus on $\hat\Gamma_\mathbb{S}^{(K)}$, we use a similar expansion to (I) and define 
\begin{align*}
\hat{\mathcal{F}}_\mathbb{S}^{\eta}=&
\begin{bmatrix}
\etahat_{22}-\eta_{22}&(\etahat_2-\eta_2)\bX_2\trans\\
\left(\etahat_2-\eta_2\right)
\bX_2&\bzero\\
\end{bmatrix},\\
\hat{\mathcal{F}}_\mathbb{S}^{(K)}=&
\frac{1}{K}\sum_{k=1}^K\Pbb_N
\begin{bmatrix}
\hat\Delta_{22}\supnk(\bUvec)&\hat\Delta_2\supnk(\bUvec)\bX_2\trans&\\
\hat\Delta_2\supnk(\bUvec)\bX_2&\bzero
\end{bmatrix},\\
\bar{\mathcal{F}}_k=&
\mathbb{E}_\mathcal{L}
\begin{bmatrix}
\hat\Delta_{22}\supnk(\bUvec)&\hat\Delta_2\supnk(\bUvec)\bX_2\trans&\\
\hat\Delta_2\supnk(\bUvec)\bX_2&\bzero
\end{bmatrix}
\:\forall k\in\{1,\dots,K\},
\end{align*}
We argue as in (I), that from \eqref{eta_EE_Q1} it follows that 
$\hat{\Gamma}^{(K)}_\mathbb{S}=\hat{\mathcal{F}}^{\eta}_\mathbb{S}+\hat{\mathcal{F}}^{(K)}_\mathbb{S}.$ Using \eqref{deltas bound}, Assumptions \ref{assumption: Q imputation} and Lemma \ref{lemm_chakrabortty} $\hat{\mathcal{F}}^{(K)}_\mathbb{S}-\frac{1}{K}\sum_k\bar{\mathcal{F}}_k=O_\Pbb\left(N^{-\frac{1}{2}}\right)$, therefore $\hat{\Gamma}^{(K)}_\mathbb{S}=\hat{\mathcal{F}}^{\eta}_\mathbb{S}+\frac{1}{K}\sum_k\bar{\mathcal{F}}_k+O_\Pbb\left(N^{-\frac{1}{2}}\right)$. Next we follow the same decomposition for $\hat{\mathcal{F}}^{\eta}_\mathbb{S}$ as we did in (I) for $\hat{\mathcal{S}}^{\eta}_\mathbb{S}$, it follows that
\begin{align*}
\hat{\Gamma}^{(K)}_\mathbb{S}
=&
\Pbb_n
\begin{bmatrix}
Y_2^2-\mubar_{22}(\bUvec)&\left\{Y_2-\mubar_2(\bUvec)\right\}\bX_2\trans\\
\left\{Y_2-\mubar_2(\bUvec)\right\}\bX_2&\bzero
\end{bmatrix}\\
-&
\frac{1}{n}\sum_{k=1}^{K}\sum_{i\in\mathcal{I}_k}
\begin{bmatrix}
\hat\Delta_{22}\supnk(\bUvec)&\hat\Delta_2\supnk(\bUvec)\bX_2\trans&\\
\hat\Delta_2\supnk(\bUvec)\bX_2&\bzero
\end{bmatrix}
+
\frac{1}{K}\sum_k\bar{\mathcal{F}}_k
+O_\Pbb\left(n^{-1}\right)+O_\Pbb\left(n^{-\frac{1}{2}}\right)o_\Pbb(1).
\end{align*}
The first term in the right hand side is $O_\Pbb\left(n^{-\frac{1}{2}}\right)$ by he CLT, the next two terms together are $O_\Pbb\left(n^{-\frac{1}{2}}c_{n^-_K}\right)$ by Lemma \ref{lemm_deltas}, thus $\hat{\Gamma}^{(K)}_\mathbb{S}=O_\Pbb\left(n^{-\frac{1}{2}}c_{n^-_K}\right).$\\

(IV) Finally we consider $\Gamma_{\mathcal{U}}$. By central limit theorem and \eqref{impute_constraint} it follows that
\begin{align*}
\Gamma_{\mathcal{U}}
=
\mathbb{E}
\begin{bmatrix}
\mubar_{22}(\bUvec)&\mubar_2(\bUvec)\bX_2\trans\\
\mubar_2(\bUvec)\bX_2&\bX_2\bX_2\trans\\
\end{bmatrix}
+O_\Pbb\left(N^{-\frac{1}{2}}\right)=\mathbb{E}[\bXcheck_2\bXcheck_2\trans]+O_\Pbb\left(N^{-\frac{1}{2}}\right).
\end{align*}
From (I)-(IV) we can write \eqref{stage2_expansion} as $(\bthetahat_2-\bthetabar_2)=\mathbb{E}\left[\bXcheck_2\bXcheck_2\trans\right]^{-1}\mathbb{T}_{\mathcal{L}}+O_\Pbb\left(n^{-\frac{1}{2}}c_{n^-_K}\right)$, it follows that

\begin{align*}
&\sqrt n
(\bthetahat_2-\bthetabar_2)\\
=&
\mathbb{E}\left[\bXcheck_2\bXcheck_2\trans\right]^{-1}\frac{1}{\sqrt n}\sum_{i=1}^n
\begin{bmatrix}
\left\{Y_{2i}Y_{3i}-\mubar_{23}(\bUvec_i)\right\}-\bar\beta_{21}
\left\{Y_{2i}^2-\mubar_{22}(\bUvec_i)\right\}-
\bXcheck_{2i}\trans\bthetabar_{2-}
\left\{Y_{2i}-\mubar_2(\bUvec_i)\right\}\\
\bX_{2i}\left\{Y_{3i}-\mubar_3(\bUvec_i)\right\}-\bar\beta_{21}
\bX_{2i}\left\{Y_{2i}-\mubar_2(\bUvec_i)\right\}
\end{bmatrix}\\
+&o_\Pbb\left(1\right).
\end{align*}
\end{proof}

\begin{proof}[Proof of Theorem \ref{theorem: unbiased theta1}] The solution to stage 1 estimating equation $\btheta_1$ in Section \ref{section: SSQL} satisfies 
\[
\Pbb_N
\left[\bX_1\left\{
\muhat_2(\bUvec)+\hat\beta_{21}\muhat_2(\bUvec)+\bH_{20}\trans\bbetahat_{22}+\left[\bH_{21}\trans\bgammahat_2\right]_+-\bX_1\trans\bthetahat_1
\right\}\right]=\bzero.
\]
We center the above at $\bthetabar_1$ and get 
\begin{align}\label{eq: centered theta1}
\Pbb_N\left[\bX_1\bX_1\trans\right]\left(\bthetahat_1-\bthetabar_1\right)
=
\Pbb_N
\left[\bX_1\left\{
\mubar_2(\bUvec)+\hat\beta_{21}\mubar_2(\bUvec)+\bH_{20}\trans\bbetahat_{22}+\left[\bH_{21}\trans\bgammahat_2\right]_+-\bX_1\trans\bthetabar_1
\right\}\right].
\end{align}
Next, with the following definitions
\begin{align*}
\hat\Sigma_\Usc=&\Pbb_N\left[\bX_1\bX_1\trans\right],\quad
\hat\Sigma_\Lsc=\Pbb_n\left[\bX_1\bX_1\trans\right],\\
\mathcal{R}^{(1)}=&\Pbb_N
\left[\bX_1\left\{
\mubar_2(\bUvec)+\hat\beta_{21}\mubar_2(\bUvec)+\bH_{20}\trans\bbetahat_{22}+\left[\bH_{21}\trans\bgammahat_2\right]_+-\bX_1\trans\bthetabar_1
\right\}\right],\\
\hat{\mathcal{R}}^{(1K)}_\mathbb{S}=&\Pbb_N\left[
\bX_1\left\{\muhat_2(\bUvec)-\mubar_2(\bUvec)\right\}
\right],
\end{align*}
we can write \eqref{eq: centered theta1} as $\hat\Sigma_{\mathcal{U}}(\bthetahat_1-\bthetabar_1)=\mathcal{R}^{(1)}+(1+\hat\beta_{21})\hat{\mathcal{R}}^{(1K)}_\mathbb{S}$. We now analyze both terms $\mathcal{R}^{(1)},$ and $(1+\hat\beta_{21})\hat{\mathcal{R}}^{(1K)}_\mathbb{S}$.\\

I) First we consider $(1+\hat\beta_{21})\hat{\mathcal{R}}^{(1K)}_\mathbb{S}$, define 
\begin{align*}
\hat{\mathcal{S}}^{(1\eta)}_\mathbb{S}=&\hat\Sigma_\Usc\left(\bEtahat_2-\bEta_2\right),
\\
\hat{\mathcal{S}}^{(1\mathbb K)}_\mathbb{S}=&\frac{1}{K}\sum_{k=1}^K\Pbb_N\left[\bX_1\hat\Delta_2\supnk(\bUvec)\right],
\\
\bar{\mathcal{S}}^{(1)}_k=&\mathbb{E}\left[\bX_1\hat\Delta_2\supnk(\bUvec)\right],
\end{align*}
from \eqref{eta_EE_Q1} it follows that $\hat{\mathcal{R}}^{(1K)}_\mathbb{S}=\hat{\mathcal{S}}^{(1\eta)}_\mathbb{S}+\hat{\mathcal{S}}^{(1\mathbb K)}_\mathbb{S}$, next from Assumptions \ref{assumption: covariates}, \ref{assumption: Q imputation}, we get $\sum_{k=1}^{K} \sup\limits_{\bX_1,\bUvec}\| \bX_1\hat\Delta_2\supnk(\bUvec)\|=o_\Pbb(1)$, thus by Lemma \ref{lemm_chakrabortty} $\hat{\mathcal{S}}^{(1\mathbb K)}_\mathbb{S}=\bar{\mathcal{S}}^{(1)}_k+\left(N^{-\frac{1}{2}}\right)$. Using \eqref{eta_EE_Q1} again, and recalling $\mubar_2(\bUvec)=m_2(\bUvec)+\bX_1\trans\bEta_2$ we have
\begin{align*}
\hat{\mathcal{S}}^{(1\eta)}_\mathbb{S}
=&
\hat\Sigma_\Usc\hat\Sigma_\Lsc^{-1}
\frac{1}{n}\sum_{k=1}^K\sum_{i\in\mathcal{I}_k}\bX_{1i}\left\{Y_{2i}-\mubar_2(\bUvec_i)-\mhat_2\supnk(\bUvec_i)+m_2(\bUvec_i)\right\}\\
=&
\frac{1}{n}\sum_{i=1}^n\bX_{1i}\left\{Y_{2i}-\mubar_2(\bUvec_i)\right\}
-
\frac{1}{n}\sum_{k=1}^K\sum_{i\in\mathcal{I}_k}\bX_{1i}\hat\Delta_2\supnk(\bUvec_i)
+
O_\Pbb\left(n^{-\frac{1}{2}}\right),
\end{align*}
where the last line follows by the CLT and Assumptions  \ref{assumption: covariates} and \ref{assumption: Q imputation} as
\[
\hat\Sigma_{\mathcal{U}}\hat\Sigma_{\mathcal{L}}^{-1}=I+O_\Pbb\left(n^{-\frac{1}{2}}\right)
\] 
Now using Lemma \ref{lemm_chakrabortty} and Assumptions  \ref{assumption: covariates}, \ref{assumption: Q imputation} again, it follows that 
\[
\hat{\mathcal{S}}^{(1\mathbb K)}_\mathbb{S}=\bar{\mathcal{S}}^{(1)}_k+O_\Pbb\left(N^{-\frac{1}{2}}\right),
\]
combining the above we can write
\[
\hat{\mathcal{R}}_\mathbb{S}^{(1K)}
=
\Pbb_n\bX_1\left\{Y_2-\mubar_2(\bUvec)\right\}
-
\frac{1}{n}\sum_{k=1}^K\left\{\sum_{i\in\mathcal{I}_k}
\bX_{1i}\hat\Delta_2\supnk(\bUvec_i)-\mathbb{E}\left[\bX_1\hat\Delta_2\supnk(\bUvec)\right]\right\}
+
O_\Pbb\left(n^{-\frac{1}{2}}\right).
\]
Next by Assumption \ref{assumption: Q imputation} and Lemma \ref{lemm_deltas} we have 
\[
\frac{1}{\sqrt n}\sum_{k=1}^K\left\{\sum_{i\in\mathcal{I}_k}
\bX_{1i}\hat\Delta_2\supnk(\bUvec_i)-\mathbb{E}\left[\bX_1\hat\Delta_2\supnk(\bUvec)\right]\right\}=O_\Pbb\left(c_{n^-_K}\right),
\]
therefore $\hat{\mathcal{R}}_\mathbb{S}^{(1K)}=\Pbb_n\bX_1\left\{Y_2-\mubar_2(\bUvec)\right\}+O_\Pbb\left(n^{-\frac{1}{2}}c_{n^-_K}\right)$. Finally using Theorem \ref{theorem: unbiased theta2} we have $\hat\beta_{21}-\bar\beta_{21}=O_\Pbb\left(n^{-\frac{1}{2}}\right)$, and by CLT $\Pbb_n\bX_1\left\{Y_2-\mubar_2(\bUvec)\right\}=O_\Pbb\left(n^{-\frac{1}{2}}\right)$, thus we can write 
\[
(1+\hat\beta_{21})\hat{\mathcal{R}}^{(1K)}_\mathbb{S}=(1+\bar\beta_{21})\Pbb_n\bX_1\left\{Y_2-\mubar_2(\bUvec)\right\}+O_\Pbb\left(n^{-\frac{1}{2}}c_{n^-_K}\right).
\]
II) Next we consider $\mathcal{R}^{(1)}$ by writing
\begin{align*}
\mathcal{R}^{(1)}=
&\Pbb_N
\left[\bX_1\left\{
\mubar_2(\bUvec)+\bar\beta_{21}\mubar_2(\bUvec)+\bH_{20}\trans\bbetabar_{22}+\left[\bH_{21}\trans\bgammabar_2\right]_+-\bX_1\trans\bthetabar_1
\right\}\right]\\
+&
\Pbb_N
\left[\bX_1\left\{\mubar_2(\bUvec)(\hat\beta_{21}-\bar\beta_{21})+\bH_{20}\trans(\bbetahat_{22}-\bbetabar_{22})+\left[\bH_{21}\trans\bgammahat_2\right]_+-
\left[\bH_{21}\trans\bgammabar_2\right]_+
\right\}\right],
\end{align*}
note that under \eqref{impute_constraint} using model \eqref{linear_Qs} the first term in the right hand side is mean zero, therefore from Assumption \ref{assumption: covariates} and CLT
\[
\Pbb_N
\left\{\bX_1\left(
\mubar_2(\bUvec)+\bbetabar_{21}\mubar_2(\bUvec)+\bH_{20}\trans\bbetabar_{22}+\left[\bH_{21}\trans\bgammabar_2\right]_+-\bX_1\trans\bthetabar_1
\right)\right\}=O_\Pbb\left(N^{-\frac{1}{2}}\right).
\]
Hence, we have
\begin{align*}
\sqrt n\mathcal{R}^{(1)}
=&
\sqrt n\Pbb_N
\left[\bX_1\left\{
\mubar_2(\bUvec)(\hat\beta_{21}-\bar\beta_{21})+\bH_{20}\trans(\bbetahat_{22}-\bbetabar_{22})+\left[\bH_{21}\trans\bgammahat_2\right]_+-
\left[\bH_{21}\trans\bgammabar_2\right]_+
\right\}\right]
+O_\Pbb\left(\sqrt{\frac{n}{N}}\right)\\
=&
\Pbb_N
\left[\bX_1
\left(\mubar_2(\bUvec),\bH_{20}\trans\right)\right]
\sqrt n\left(\bbetahat_2-\bbetabar_2\right)
+
\sqrt n\Pbb_N
\left[\bX_1\left(\left[\bH_{21}\trans\bgammahat_2\right]_+-
\left[\bH_{21}\trans\bgammabar_2\right]_+
\right)\right]
+O_\Pbb\left(\sqrt{\frac{n}{N}}\right)
\\
=&
\mathbb{E}\left[\bX_1
\left(\mubar_2(\bUvec),\bH_{20}\trans\right)\right]
n^{-\frac{1}{2}}\sum_{i=1}^n\bpsi_{2i\beta}
+
\sqrt n\Pbb_N
\left[\bX_1\left(\left[\bH_{21}\trans\bgammahat_2\right]_+-
\left[\bH_{21}\trans\bgammabar_2\right]_+
\right)\right]
+O_\Pbb\left(\sqrt{\frac{n}{N}}\right),
\end{align*}
where the last inequality follows from the CLT, where $\bpsi_{2i\beta}$ is the element corresponding to $\bbetahat_2$ of the influence function $\bpsi_{2i}$ defined in Theorem \ref{theorem: unbiased theta2}.

Next by Theorem \ref{theorem: unbiased theta2} we know that 
\[
\sqrt n (\bgammahat_2-\bgammabar_2)=O_\Pbb(1),
\]
using Lemma \ref{lemm_gamma_difs} (a) we have 
\[
\Pbb\left[\sqrt{n}\Pbb_N
\left\{\bX_1\left(\left[\bH_{21}\trans\bgammahat_2\right]_+-
\left[\bH_{21}\trans\bgammabar_2\right]_+
\right)\right\}
=
\Pbb_N\bigg\{\bX_1\bH_{21}\trans I\left(\bH_{21}\trans\bgammabar_2>0\right)
\bigg\}
\sqrt{n}\left(
\bgammahat_2
-\bgammabar_2
\right)\right]\rightarrow1.
\]

Therefore, letting $\bpsi_{2i\gamma}$ be the element corresponding to $\bgammahat_2$ of the influence function $\bpsi_{2i}$ defined in Theorem \ref{theorem: unbiased theta2},
\begin{align*}
&\sqrt n\Pbb_N
\left\{\bX_1\left(\left[\bH_{21}\trans\bgammahat_2\right]_+-
\left[\bH_{21}\trans\bgammabar_2\right]_+
\right)\right\}\\
&=
\Pbb_N\left\{\bX_1\bH_{21}\trans I(\bH_{21}\trans\bgammabar_2>0) I\left(\bgammahat_2\in\mathcal{A}\right)\right\}
\sqrt{n}\left(
\bgammahat_2
-\bgammabar_2
\right)\\
&+
\sqrt{n}\Pbb_N\left\{\bX_1\left(
\left[\bH_{21}\trans\bgammahat_2\right]_+
-\left[\bH_{21}\trans\bgammabar_2\right]_+
\right)\right\}I_{\left\{\bgammahat_2\notin\mathcal{A}\right\}}\\
&=
\mathbb{E}\left[
\bX_1\bH_{21}\trans|\bH_{21}\trans\bgammabar_2>0,\bgammahat_2\in\mathcal{A}\right]\Pbb\left(\bH_{21}\trans\bgammabar_2>0\right)\Pbb\left(\bgammahat_2\in\mathcal{A}\right)
\frac{1}{\sqrt n}\sum_{i=1}^n\psi_{2\gamma_2i}
+O_\Pbb\left(c_{n^-_K}\right)+o_\Pbb\left(1\right)\\
&=
\mathbb{E}\left[\bX_1\bH_{21}\trans|\bH_{21}\trans\bgammabar_2>0\right]\Pbb\left(\bH_{21}\trans\bgammabar_2>0\right)
\frac{1}{\sqrt n}\sum_{i=1}^n\psi_{2\gamma_2i}+o_\Pbb\left(1\right),\\
\end{align*}
combining all terms
\begin{align*}
\sqrt n\mathcal{R}^{(1)}
=&
\mathbb{E}\left[\bX_1
\left(\mubar_2(\bUvec),\bH_{20}\trans\right)\right]
n^{-\frac{1}{2}}\sum_{i=1}^n\bpsi_{2i(\beta )}\\
+&
\mathbb{E}\left[\bX_1\bH_{21}\trans|\bH_{21}\trans\bgammabar_2>0\right]\Pbb\left(\bH_{21}\trans\bgammabar_2>0\right)
\frac{1}{\sqrt n}\sum_{i=1}^n\bpsi_{2i(\gamma)}\\
+&
O_\Pbb\left(c_{n^-_K}\right).
\end{align*}
Finally, from I), II), and since $\hat\Sigma_{\mathcal{U}}^{-1}=\mathbb{E}\left[\bX_1\bX_1\trans\right]^{-1}+o_\Pbb\left(1\right)$ by the LLN, we have 
\begin{align*}
\sqrt n(\bthetahat_1-\bthetabar_1)
=&
\mathbb{E}\left[\bX_1\bX_1\trans\right]^{-1}\hat\Sigma_{\mathcal{U}}^{-1}\mathcal{R}^{(1)}+\mathbb{E}\left[\bX_1\bX_1\trans\right]^{-1}(1+\hat\beta_{21})\hat{\mathcal{R}}^{(1K)}_\mathbb{S}+o_\Pbb\left(1\right)
\\
=&\mathbb{E}\left[\bX_1\bX_1\trans\right]^{-1}(1+\bar\beta_{21})\frac{1}{\sqrt n}\sum_{i=1}^n\{Y_{2i}-\mubar_2(\bUvec_i)\}\\
+&
\mathbb{E}\left[\bX_1\bX_1\trans\right]^{-1}
\mathbb{E}\left[\bX_1
\left(\mubar_2(\bUvec),\bH_{20}\trans\right)\right]
\frac{1}{\sqrt n}\sum_{i=1}^n\bpsi_{2i(\beta )}\\
+&
\mathbb{E}\left[\bX_1\bX_1\trans\right]^{-1}\mathbb{E}\left[\bX_1\bH_{21}\trans|\bH_{21}\trans\bgammabar_2>0\right]\Pbb\left(\bH_{21}\trans\bgammabar_2>0\right)
\frac{1}{\sqrt n}\sum_{i=1}^n\bpsi_{2i\gamma}\\
+&
o_\Pbb\left(1\right),
\end{align*}
using \eqref{impute_constraint} we have $\mathbb{E}\left[\bX_1
\left(\mubar_2(\bUvec),\bH_{20}\trans\right)\right]=\mathbb{E}\left[\bX_1
\left(Y_2,\bH_{20}\trans\right)\right]$ which yields our required results
\end{proof}
Next we discuss some results and assumptions needed for Proposition \ref{lemma: Q parameter var}. First we show the asymptotic results for the supervised estimation of the $Q$-function parameters. Recall $\bthetahat_{1 \scriptscriptstyle \sf SUP},$ $\bthetahat_{2 \scriptscriptstyle \sf SUP}$ are the estimators for the $Q$-function parameters, when using the labeled data $\Lsc$ only. From \cite{laber2014} we have that the following results for $\bthetahat_{2 \scriptscriptstyle \sf SUP}$:
\[
\sqrt n\left(\bthetahat_{2 \scriptscriptstyle \sf SUP}-\bthetabar_2\right)=\Sigma_2^{-1}\frac{1}{\sqrt n}\sum_{i=1}^n
\bpsi_{2 \scriptscriptstyle \sf SUP}(\bL;\bthetabar_2)
\rightarrow\Nsc\left(\bzero,\bV_{2 \scriptscriptstyle \sf SUP}\left[\bthetabar_2\right]\right),
\]
with 
\begin{align*}
\bpsi_{2 \scriptscriptstyle \sf SUP}(\bL;\bthetabar_2)
&=\bXcheck_2\{Y_{3i}-\bXcheck_{2i}\trans\bthetabar_2\},\\ \bV_{2 \scriptscriptstyle \sf SUP}\left[\bthetabar_2\right]&=\bSigma_2^{-1}\Ebb\left[\bpsi_{2 \scriptscriptstyle \sf SUP}(\bL;\bthetabar_2)\bpsi_{2 \scriptscriptstyle \sf SUP}(\bL;\bthetabar_2)\trans\right]\left(\bSigma_2^{-1}\right)\trans,
\end{align*}
and for $\bthetahat_{1 \scriptscriptstyle \sf SUP}$:
\[
\sqrt n\left(\bthetahat_{1 \scriptscriptstyle \sf SUP}-\bthetabar_1\right)=\Sigma_1^{-1}\frac{1}{\sqrt n}\sum_{i=1}^n
\bpsi_{1 \scriptscriptstyle \sf SUP}(\bL;\bthetabar_1)
\rightarrow\Nsc\left(\bzero,\bV_{1 \scriptscriptstyle \sf SUP}\left[\bthetabar_1\right]\right),
\]
with 
\begin{align*}
\bpsi_{1 \scriptscriptstyle \sf SUP}(\bL_i;\bthetabar_2)
=&
\bX_{i1}\{Y_{2i}+Y_{2i}\bar\beta_{21}+\bH_{20i}\trans\bbetabar_{22}+[\bH_{21i}\trans\bgammabar_2]_+-\bX_{1i}\trans\bthetabar_1\}\\
+&
\mathbb{E}\left[\bX_1
\left(Y_2,\bH_{20}\trans\right)\right]
\bpsi_{2 \scriptscriptstyle \sf SUP,(\beta )}(\bL_i)\\
+&
\mathbb{E}\left[\bX_1\bH_{21}\trans|\bH_{21}\trans\bgammabar_2>0\right]\Pbb\left(\bH_{21}\trans\bgammabar_2>0\right)
\bpsi_{2 \scriptscriptstyle \sf SUP,(\gamma )}(\bL_i),\\ 
\bV_{1 \scriptscriptstyle \sf SUP}\left[\bthetabar_1\right]
=&
\bSigma_1^{-1}\Ebb\left[\bpsi_{1 \scriptscriptstyle \sf SUP}(\bL;\bthetabar_1)\bpsi_{1 \scriptscriptstyle \sf SUP}(\bL;\bthetabar_1)\trans\right]\left(\bSigma_1^{-1}\right)\trans.
\end{align*}
Next we discuss the assumption required for Proposition \ref{lemma: Q parameter var}. We need the imputation models $\mubar_s(\bUvec)$, $s\in\{2,3,22,23\}$ to satisfy several additional constraints. For example, for the stage two $Q$-function parameters, recall $\btheta_{2-}=(\bbeta_{22}\trans,\bgamma_2\trans)\trans$, the imputation models should satisfy:
\begin{equation*}
\begin{alignedat}{3}
 \Ebb\left[\bXvec\trans\mubar_j(\bUvec)\{g_s(\bY) - \mubar_s(\bUvec)\}\right] = \bzero &\quad&
\Ebb\left[\bXvec\trans \mubar_2(\bUvec)\bX_2\trans\bthetabar_{2-}\{g_s(\bY) - \mubar_s(\bUvec)\}\right] = \bzero&,s,j\in\{2,3,22,23\} \\
\Ebb\left[\bXvec\bXvec\trans\mubar_j(\bUvec)\{g_s(\bY) - \mubar_s(\bUvec)\}\right] = \bzero, &\quad&
\Ebb\left[\bXvec\bXvec\trans \bX_2\trans\bthetabar_{2-}\{g_s(\bY) - \mubar_s(\bUvec)\}\right] = \bzero&,s,j\in\{2,3\},
\end{alignedat}
\end{equation*}
where $\bX=(1,\bX_1\trans,\bX_2\trans)\trans$, $g_2(\bY)=Y_2,$ $g_3(\bY)=Y_3,$ $g_{22}(\bY)=Y_2^2,$ $g_{23}(\bY)=Y_2Y_3$. 

To summarize all the assumptions needed, we define the following functions:
\begin{align}\label{equation: psi theta difference}
\begin{split}
\mathcal{E}^\theta(\bUvec)\equiv&
\left\{\mathcal{E}_1(\bUvec)\trans,\mathcal{E}_2(\bUvec)\trans\right\}\trans,\\
\mathcal{E}_2(\bUvec)\equiv&
\begin{bmatrix}
    \mubar_{23}(\bUvec)-
    [\mubar_{22}(\bUvec),\mubar_2(\bUvec)\bX_2\trans]\bthetabar_2\\
        \bX_2\{\mubar_3(\bUvec)-
        [\mubar_2(\bUvec),\bX_2\trans]\bthetabar_2\}
        \end{bmatrix},\\
\mathcal{E}_1(\bUvec)
\equiv&
\bX_1\{\mubar_2(\bUvec)(1+\bar\beta_{21})+\Qopt_{2-}(\bH_2;\bthetabar_2)-\bX_{1}\trans\bthetabar_1\}\\
+&
\mathbb{E}\left[\bX_1
\left(Y_2,\bH_{20}\trans\right)\right]
\mathcal{E}_{2\beta}(\bUvec)\\
+&
\mathbb{E}\left[\bX_1\bH_{21}\trans|\bH_{21}\trans\bgammabar_2>0\right]\Pbb\left(\bH_{21}\trans\bgammabar_2>0\right)
\mathcal{E}_{2\gamma}(\bUvec),
\end{split}
\end{align}
where $\mathcal{E}_{2\beta}(\bUvec)$, $\mathcal{E}_{2\gamma}(\bUvec)$ are the elements corresponding to $\bbetabar_2$, $\bgammabar_2$ of $\mathcal{E}_2(\bUvec)$. Now we can succinctly summarize the constraints, by having $\mubar_s(\bUvec)$, $s\in\{2,3,22,23\}$ satisfy 
\begin{align*}
\Ebb\left[\left\{
\bpsi_{2 \scriptscriptstyle \sf SUP}(\bL;\bthetabar_2)
-
\mathcal{E}_2(\bUvec)\right\}\mathcal{E}_2(\bUvec)\trans
\right]=\bzero.
\end{align*}  
This is condensed in the following assumption.

\begin{assumption}\label{assumption: additional constraints}
Let $\mathcal{E}^\theta(\bUvec)$ be as defined in \eqref{equation: psi theta difference}, and \[
\bpsi\subSUP(\bL;\bthetabar)=\left[\bpsi_{1 \scriptscriptstyle \sf SUP}(\bL;\bthetabar_1)\trans,\bpsi_{2 \scriptscriptstyle \sf SUP}(\bL;\bthetabar_2)\trans\right]\trans,
\]
the imputation models $\mubar_s(\bUvec)$, $s\in\{2,3,22,23\}$ satisfy
\begin{align*}
\Ebb\left[\left\{
\bpsi\subSUP(\bL;\bthetabar)
-
\mathcal{E}^\theta(\bUvec)\right\}\mathcal{E}^\theta(\bUvec)\trans
\right]=\bzero.
\end{align*} 
\end{assumption}
\begin{proof}[Proof of Proposition \ref{lemma: Q parameter var}]

We first show the result is true for $\bV_{2 \scriptscriptstyle \sf SSL}\left[\bthetabar_2\right]$. To simplify algebra, we denote the influence function from Theorem \ref{theorem: unbiased theta2} as $\bpsi_{2 \scriptscriptstyle \sf SSL}(\bL;\bthetabar_2)$. Using the influence function of $\bthetahat_{2 \scriptscriptstyle \sf SUP}$ and Theorem \ref{theorem: unbiased theta2} we have the following relationship:
\[
\bpsi_{2 \scriptscriptstyle \sf SSL}(\bL;\bthetabar_2)
=
\bpsi_{2 \scriptscriptstyle \sf SUP}(\bL;\bthetabar_2)
-
\mathcal{E}_2(\bUvec).
\]
Therefore 
\begin{align*}
\bV_{2 \scriptscriptstyle \sf SSL}\left(\bthetabar_2\right)
=&
\bSigma_2^{-1}
\Ebb\left[
\bpsi_{2 \scriptscriptstyle \sf SSL}(\bL;\bthetabar_2)
\bpsi_{2 \scriptscriptstyle \sf SSL}(\bL;\bthetabar_2)\trans
\right]
\left(\bSigma_2^{-1}\right)\trans\\
=&
\bSigma_2^{-1}
\Ebb\left[
\left\{\bpsi_{2 \scriptscriptstyle \sf SUP}(\bL;\bthetabar_2)-\mathcal{E}_2(\bUvec)\right\}
\left\{\bpsi_{2 \scriptscriptstyle \sf SUP}(\bL;\bthetabar_2)-\mathcal{E}_2(\bUvec)\right\}\trans
\right]
\left(\bSigma_2^{-1}\right)\trans\\
=&
\bSigma_2^{-1}
\Ebb\left[
\bpsi_{2 \scriptscriptstyle \sf SUP}(\bL;\bthetabar_2)
\bpsi_{2 \scriptscriptstyle \sf SUP}(\bL;\bthetabar_2)\trans
\right]
\left(\bSigma_2^{-1}\right)\trans\\
+&
\bSigma_2^{-1}
\Ebb\left[\mathcal{E}_2(\bUvec)\mathcal{E}_2(\bUvec)\trans
\right]
\left(\bSigma_2^{-1}\right)\trans\\
-&
2\bSigma_2^{-1}
\Ebb\left[
\bpsi_{2 \scriptscriptstyle \sf SUP}(\bL;\bthetabar_2)
\mathcal{E}_2(\bUvec)\trans
\right]
\left(\bSigma_2^{-1}\right)\trans
\end{align*}
Now, since our imputation models satisfy Assumption \ref{assumption: additional constraints}, it follows that
\begin{align*}
\Ebb\left[\left\{
\bpsi_{2 \scriptscriptstyle \sf SUP}(\bL;\bthetabar_2)
-
\mathcal{E}_2(\bUvec)\right\}\mathcal{E}_2(\bUvec)\trans
\right]=\bzero.
\end{align*}
Therefore we have 
\begin{align*}
\bV_{2 \scriptscriptstyle \sf SSL}\left(\bthetabar_2\right)
=&
\bV_{2 \scriptscriptstyle \sf SUP}\left(\bthetabar_2\right)
-
\bSigma_2^{-1}\text{Var}\left[
\mathcal{E}_2(\bUvec)\right]\left(\bSigma_2^{-1}\right)\trans.
\end{align*}
To show the result is true for $\bV_{1 \scriptscriptstyle \sf SSL}\left[\bthetabar_1\right]$, 
We denote by $\mathcal{E}_{2\beta}(\bUvec)$ and $\mathcal{E}_{2\gamma}(\bUvec)$ the vectors corresponding to $\bbetabar_2$, $\bgammabar_2$ in $\mathcal{E}_2(\bUvec)$ respectively, and further recall the definition of $\mathcal{E}_1(\bUvec)$:

\begin{align*}
\mathcal{E}_1(\bUvec)
=&
\bX_1\{\mubar_2(\bUvec)+\mubar_2(\bUvec)\bar\beta_{21}+\bH_{20}\trans\bbetabar_{22}+[\bH_{21}\trans\bgammabar_2]_+-\bX_{1}\trans\bthetabar_1\}\\
+&
\mathbb{E}\left[\bX_1
\left(Y_2,\bH_{20}\trans\right)\right]
\mathcal{E}_{2\beta}(\bUvec)\\
+&
\mathbb{E}\left[\bX_1\bH_{21}\trans|\bH_{21}\trans\bgammabar_2>0\right]\Pbb\left(\bH_{21}\trans\bgammabar_2>0\right)
\mathcal{E}_{2\gamma}(\bUvec).
\end{align*}
From the form of the influence function of $\bthetahat_{1 \scriptscriptstyle \sf SUP}$, and Theorems \ref{theorem: unbiased theta2} \& \ref{theorem: unbiased theta1} we have that:
\[
\bpsi_{1 \scriptscriptstyle \sf SSL}(\bL;\bthetabar_1)
=
\bpsi_{1 \scriptscriptstyle \sf SUP}(\bL;\bthetabar_1)
-
\mathcal{E}_1(\bUvec)
.
\]

Analogous steps for the proof of $\bthetabar_2$ can then be used to show
\begin{align*}
\bV_{1 \scriptscriptstyle \sf SSL}\left(\bthetabar_1\right)
=&
\bV_{1 \scriptscriptstyle \sf SUP}\left(\bthetabar_1\right)
-
\bSigma_1^{-1}\text{Var}\left[\mathcal{E}_1(\bUvec)\right]\left(\bSigma_1^{-1}\right)\trans.
\end{align*}

The required result is obtained by stacking the influence functions for $\btheta_1,\btheta_2$ for the supervised and semi-supervised versions, noting that 
\[
\bpsi\subSSL(\bL;\bthetabar)
=
\bpsi\subSUP(\bL;\bthetabar)
-
\mathcal{E}^\theta(\bUvec).
\]
and repeating the steps above.
\end{proof}

\subsection{Value Function Results}\label{appendix: value fun}

In this Section we prove the main results for our SSL value function estimator. Before the proofs we go over some useful definitions, notation and lemmas. First recall that, in order to correct for potential biases arising from finite sample estimation and model mis-specifications, the final imputed models for $\{Y_2,\omega_2(\bHcheck_2, A_2; \bThetabar),$ $Y_t\omega_2(\bHcheck_2, A_2; \bThetabar),$ $t=2,3\}$ satisfy the following constraints:
\begin{equation}\label{V function impute_constraint}
\begin{alignedat}{3}
\Ebb\left[\omega_1(\bHcheck_1,A_1;\bThetabar)\left\{Y_2-\mubar_2^v(\bUvec)\right\}\right] &=0, \\
\Ebb\left[\Qopt_{2-}(\bUvec;\btheta_2)\left\{\omega_2(\bHcheck_2,A_2;\bThetabar)-\mubar_{\omega_2}^v(\bUvec)\right\}\right] &= 0,\\
\Ebb\left[\omega_2(\bHcheck_2,A_2;\bThetabar)Y_t- \mubar^v_{t\omega_2}(\bUvec)\right] &= 0,\:t=2,3.
\end{alignedat}
\end{equation}

Next, define the set
\begin{align*}
\mathcal{S}(\delta)=\bigg\{
(\btheta,\bxi)\bigg|
\|\bthetahat-\btheta\|_2^2<\delta,
\|\bxihat-{\bxi}\|_2^2<\delta,
\btheta_t\in\Theta_t,{\bxi}_t\in\Omega_t,t=1,2,\\ \pi_1(\bH_1;{\bxi}_1)>0,\:\pi_2(\bHcheck_2;{\bxi}_2)>0,\:\forall\:\bH\in\mathcal{H}
\bigg\}.
\end{align*}

We will be using the influence functions for our model parameters $\bTheta$. In this regard let $\bpsi^\theta=(\bpsi_1\trans,\bpsi_2\trans)\trans$. By Theorems \ref{theorem: unbiased theta2} \& \ref{theorem: unbiased theta1} $\sqrt n(\bthetahat-\bthetabar)=n^{-1/2}\sum_{i=1}^n\bpsi^\theta(\bUvec_i)+o_\Pbb(1)$. Next, from Assumption \ref{assumption: donsker w}, it can be shown that $\bxihat$ has the following expansion:
    $
    \sqrt n(\bxihat-\bxibar)=n^{-1/2}\sum_{i=1}^n\bpsi^\xi\left(\bL_i;\bxibar\right)+o_\Pbb(1),
    $
    where 
    \[
    \bpsi_t^\xi\left(\bL;\bxibar\right)=\Ebb\left\{\bHcheck_t\trans\bHcheck_t\sigma\left(\bHcheck_t\trans\bxibar_t\right)[1-\sigma\left(\bHcheck_t\trans\bxibar_t\right)]\right\}^{-1}\bHcheck_t\left\{A_t-\sigma\left(\bHcheck_t\trans\bxibar_t\right)\right\},\:\:t=1,2,
    \]
    $\bpsi^\xi\left(\bL;\bxibar\right)=\left[\bpsi_1^\xi\left(\bL;\bxibar\right),\bpsi_2^\xi\left(\bL;\bxibar\right)\right]$ and $\mathbb{E}[\bpsi^\xi]=0$, $\mathbb{E}[(\bpsi^\xi)\trans\bpsi^\xi]<\infty$.
	
	 We now introduce a set of definitions used in this section to make the proofs easier to read. Recall from \eqref{SS_value_fun} we have
	
	$
	\Vhat\subSSLDR=\Pbb_N\left\{
	\Vsc\subSSLDR(\bUvec;\bThetahat,\muhat)\right\},$
	where $\Vsc\subSSLDR(\bUvec;\bThetahat,\muhat)$ is the semi-supervised augmented estimator for observation $\bUvec$, we re-write $\Vsc\subSSLDR(\bUvec;\bThetahat,\muhat)$ as $\Vsc_{\bThetahat,\muhat}(\bUvec)$ recall its definition, and define the following functions:
	\begin{align}\label{eq: Vs muhat}
	\begin{split}
	\Vsc_{\bThetahat,\muhat}(\bUvec)
	\equiv&
	\Qopt_1(\bHcheck_1;\bthetahat_1)
	+\omega_1(\bHcheck_1,A_1,\bThetahat)
	\left[(1+\betahat_{21})\muhat_2^v(\bUvec)-
	 \Qopt_1(\bHcheck_1;\bthetahat_1)+\Qopt_{2-}\{\bH_2;\bthetahat_2\}
	\right]\\
	+&\muhat^v_{3\omega_2}(\bUvec)-\betahat_{21}\muhat^v_{2\omega_2}(\bUvec)-\Qopt_{2-}(\bH_2;\bthetahat_2)\muhat^v_{\omega_2}(\bUvec),
	\\
	\Vsc_{\bThetabar,\muhat}(\bUvec)
	\equiv&
	\Qopt_1(\bHcheck_1;\bthetabar_1)
	+\omega_1(\bHcheck_1,A_1,\bThetabar)
	\left[(1+\bar\beta_{21})\muhat_2^v(\bUvec)-
	 \Qopt_1(\bHcheck_1;\bthetabar_1)+\Qopt_{2-}\{\bH_2;\bthetabar_2\}
	\right]\\
	+&\muhat^v_{3\omega_2}(\bUvec)-\bar\beta_{21}\muhat^v_{2\omega_2}(\bUvec)-\Qopt_{2-}(\bH_2;\bthetabar_2)\muhat^v_{\omega_2}(\bUvec).
	\end{split}
	\end{align}
     We next replace the estimated imputation functions with their limits $\mubar^v_2$, $\mubar^v_{2\omega_2}$, $\mubar^v_{3\omega_2}$ and $\mubar^v_{\omega_2}$, and define:
	\begin{align}\label{eq: Vs mubar}
	\begin{split}
	\mathcal{V}_{\bThetahat,\mubar}(\bUvec)\equiv
	&\Qopt_1(\bHcheck_1;\bthetahat_1)
	+\omega_1(\bHcheck_1,A_1,\bThetahat)
	\left[(1+\betahat_{21})\mubar_2^v(\bUvec)-
	 \Qopt_1(\bHcheck_1;\bthetahat_1)+\Qopt_{2-}(\bH_2;\bthetahat_2)
	\right]\\
	+&\mubar^v_{3\omega_2}(\bUvec)-\hat\beta_{21}\mubar^v_{2\omega_2}(\bUvec)-\Qopt_{2-}(\bH_2;\bthetahat_2)\mubar^v_{\omega_2}(\bUvec),
	\\
    \mathcal{V}_{\bThetabar,\mubar}(\bUvec)\equiv
	&\Qopt_1(\bHcheck_1;\bthetabar_1)
	+\omega_1(\bHcheck_1,A_1,\bThetabar)
	\left[(1+\bar\beta_{21})\mubar_2^v(\bUvec)-
	 \Qopt_1(\bHcheck_1;\bthetabar_1)+\Qopt_{2-}(\bH_2;\bthetabar_2)
	\right]\\
	+&\mubar^v_{3\omega_2}(\bUvec)-\bar\beta_{21}\mubar^v_{2\omega_2}(\bUvec)-\Qopt_{2-}(\bH_2;\bthetabar_2)\mubar^v_{\omega_2}(\bUvec).
	\end{split}
	\end{align}
	Finally we define the following functions which are weighted sums of the imputation function errors:
	\begin{align}\label{eq: errs terms}
	\begin{split}
	\mathcal{E}_{\bThetahat}(\bUvec)
	\equiv&
	\omega_1(\bHcheck_1,A_1;\bThetahat)
	(1+\hat\beta_{21})
	\left\{
	\muhat_2^v(\bUvec)-\mubar_2^v(\bUvec)
	\right\}
	+
	\muhat^v_{3\omega_2}(\bUvec)-\mubar^v_{3\omega_2}(\bUvec)\\
	-&
	\hat\beta_{21}\left\{\muhat^v_{2\omega_2}(\bUvec)-\mubar^v_{2\omega_2}(\bUvec)\right\}
    -
    \Qopt_{2-}\left(\bH_2;\bthetahat_2\right)
    \left\{\muhat^v_{\omega_2}(\bUvec)-\mubar^v_{\omega_2}(\bUvec)\right\}
	,\\
	\mathcal{E}_{\bThetabar}(\bUvec)
	\equiv&
	\omega_1(\bHcheck_1,A_1;\bThetabar)
	(1+\bar\beta_{21})
	\left\{
	\muhat_2^v(\bUvec)-\mubar_2^v(\bUvec)
	\right\}
	+
	\muhat^v_{3\omega_2}(\bUvec)-\mubar^v_{3\omega_2}(\bUvec)
	\\
	-&
	\bar\beta_{21}\left\{\muhat^v_{2\omega_2}(\bUvec)-\mubar^v_{2\omega_2}(\bUvec)\right\}
    -
    \Qopt_{2-}\left(\bH_2;\bthetabar_2\right)
    \left\{\muhat^v_{\omega_2}(\bUvec)-\mubar^v_{\omega_2}(\bUvec)\right\}.
	\end{split}
	\end{align}
     These definitions will come in handy in the following proofs as we can use them to write $\Vsc_{\bThetahat,\muhat}(\bUvec)=\mathcal{V}_{\bThetahat,\mubar}(\bUvec)+\mathcal{E}_{\bThetahat}(\bUvec)$, $\Vsc_{\bThetabar,\muhat}(\bUvec)=\mathcal{V}_{\bThetabar,\mubar}(\bUvec)+\mathcal{E}_{\bThetabar}(\bUvec)$. Finally, recalling that $\Pbb_{\bUvec}$ is the underlying distribution of the data, we define function $g_1:\bTheta\mapsto\mathbb{R}$ as
     \[
     g_1(\bTheta)=\int\Vsc_{\bTheta,\mubar}(\bUvec)d\Pbb_{\bUvec}.
     \]
     With the above definitions we proceed by stating three lemmas that will be used to prove Theorem \ref{thrm_ssV_fun}. We defer the proofs of these lemmas for after proving the main Theorem in this section.
    \begin{lemma}\label{lemma: CLT and expectation terms}
	Under Assumptions \ref{assumption: covariates}-\ref{assumption: V imputation}, we have
	\begin{align*}
	\mbox{I)}&\quad
	\sqrt n\left\{\Pbb_N\left[\Vsc_{\bThetabar,\mubar}\right]-g_1\left(\bThetabar\right)\right\}=o_\Pbb(1),\\
	\mbox{II)}&\quad
	\sqrt n\left\{g_1(\bThetahat)-g_1\left(\bThetabar\right)\right\}
	=
	\frac{1}{\sqrt n}\sum_{i=1}^n\left\{\left(\frac{\partial}{\partial \btheta}g_1\left(\bThetabar\right)\right)\trans 
	\bpsi^\theta(\bUvec_i)
	+
	\left(\frac{\partial}{\partial \bxi}g_1\left(\bThetabar\right)\right)\trans \bpsi^\xi(\bUvec_i)\right\}
	+o_\Pbb(1).
	\end{align*}
	\end{lemma}

	\begin{lemma}\label{lemma: centered sample average}
	Under Assumptions \ref{assumption: covariates}-\ref{assumption: V imputation}, the following holds: 
	\begin{align*}
    \sqrt n\left\{\bigg(\Pbb_N\left[\Vsc_{\bThetahat,\mubar}\right]-g_1(\bThetahat)\bigg)-\bigg(\Pbb_N\left[\Vsc_{\bThetabar,\mubar}\right]-g_1(\bThetabar)\bigg)\right\}
    =
    o_\Pbb\left(1\right).
	\end{align*}
	\end{lemma}
	
	\begin{lemma}\label{lemma: error diff}
	Under Assumptions \ref{assumption: covariates}-\ref{assumption: V imputation}, the following assertions hold:  
	\begin{align*}
	\mbox{I)}&\quad
	\sqrt n\Pbb_N\left\{\mathcal{E}_{\bThetahat}-\mathcal{E}_{\bThetabar}\right\}=o_\Pbb(1),\\
	\mbox{II)}&\quad\sqrt n\Pbb_N\left[\mathcal{E}_{\bThetabar}\right]
	=
	\mathbb{G}_n
	\left\{\nu\subSSLDR(\bL;\bThetabar)\right\}\\
	&\hspace{2.4cm}+
	\frac{1}{\sqrt n}\sum_{i=1}^n
	\left\{\bpsi^\theta(\bL_i)\trans\frac{\partial}{\partial \btheta}\int\nu\subSSLDR(\bL_i;\bTheta)d\Pbb_{\bL}\bigg|_{\bTheta=\bThetabar}
	+
	\bpsi^\xi(\bL_i)\trans\frac{\partial}{\partial \bxi}\int\nu\subSSLDR(\bL_i;\bTheta)d\Pbb_{\bL}\bigg|_{\bTheta=\bThetabar}
    \right\}\\
	&\hspace{2.4cm}+
	o_\Pbb(1).
	\end{align*}

	\end{lemma}
	
	\begin{proof}[Proof of Theorem \ref{thrm_ssV_fun}]
	We start by expanding the expression in \eqref{SS_value_fun} and using definitions \eqref{eq: Vs muhat}, \eqref{eq: Vs mubar}, \eqref{eq: errs terms}:
	
	\begin{align*}
	&\sqrt n
	\left\{
	\Pbb_{N}\left[\Vsc_{\bThetahat,\muhat}\right]-\Ebb_\mathbb{S}\left[\Vsc_{\bThetabar,\mubar}\right]
	\right\}\\
	=&
    \sqrt n\left\{\underbrace{\Pbb_N\left[\Vsc_{\bThetabar,\mubar}\right]+\Pbb_N\left[\mathcal{E}_{\bThetabar}\right]}_{(I)}-\underbrace{g_1(\bThetabar)-\mathbb{E}_\mathbb{S}\left[\mathcal{E}_{\bThetabar}\right]}_{(II)}\right\}\\
	+&
	\sqrt n\left\{\bigg(\Pbb_N\left[\Vsc_{\bThetahat,\mubar}\right]+\Pbb_N\left[\mathcal{E}_{\bThetahat}\right]-\underbrace{g_1(\bThetahat)\bigg)-\mathbb{E}_\mathbb{S}\left[\mathcal{E}_{\bThetahat}\right]}_{(III)}\right\}
	-
	\left\{\underbrace{\Pbb_N\left[\Vsc_{\bThetabar,\mubar}\right]+\Pbb_N\left[\mathcal{E}_{\bThetabar}\right]}_{(I)}-\underbrace{g_1(\bThetabar)-\mathbb{E}_\mathbb{S}\left[\mathcal{E}_{\bThetabar}\right]}_{(II)}\right\}\\
	+&
	\sqrt n\left\{\underbrace{g_1(\bThetahat)+\mathbb{E}_\mathbb{S}\left[\mathcal{E}_{\bThetahat}\right]}_{(III)}-\mathbb{E}_\mathbb{S}\left[\Vsc_{\bThetabar,\mubar}\right]\right\}
	\\
	=&
    \sqrt n\left\{\Pbb_N\left[\Vsc_{\bThetabar,\mubar}\right]-g_1(\bThetabar)\right\}
    \\
	+&
	\sqrt n\left\{g_1(\bThetahat)-g_1(\bThetabar)\right\}
	\\
	+&
	\sqrt n\left\{\left(\Pbb_N\left[\Vsc_{\bThetahat,\mubar}\right]-g_1(\bThetahat)\right)
	-
	\bigg(\Pbb_N\left[\Vsc_{\bThetabar,\mubar}\right]-g_1(\bThetabar)\bigg)\right\}
	\\
	+&
	\sqrt n\Pbb_N\left[\mathcal{E}_{\bThetahat}-\mathcal{E}_{\bThetabar}\right]
	\\
	+&
	\sqrt n\Pbb_N\left[\mathcal{E}_{\bThetabar}\right]\\
	=&
	\frac{1}{\sqrt n}\sum_{i=1}^n\psi^{v}_{\subSSLDR}(\bL_i;\bThetabar)
	+o_\Pbb\left(1\right).
    \end{align*}

	which follows from Lemmas \ref{lemma: CLT and expectation terms}, \ref{lemma: centered sample average} \& \ref{lemma: error diff} with the influence function $\psi^{v}_{\subSSLDR}$ defined as
		
	\begin{align*}
	\psi^{v}\subSSLDR(\bL;\bThetabar)
	=&
	\nu\subSSLDR(\bL;\bThetabar)
	+
	\bpsi^\theta(\bL)\trans\frac{\partial}{\partial \btheta}\int\left\{\Vsc_{\bTheta,\mubar}(\bL)+\nu\subSSLDR(\bL;\bTheta)\right\}d\Pbb_{\bL}\bigg|_{\bTheta=\bThetabar}\\
	&\hspace{2.1cm}+
	\bpsi^\xi(\bL)\trans\frac{\partial}{\partial \bxi}\int\left\{\Vsc_{\bTheta,\mubar}(\bUvec)+\nu\subSSLDR(\bL;\bTheta)\right\}d\Pbb_{\bL}\bigg|_{\bTheta=\bThetabar},\\
	\nu_{\subSSLDR}(\bL;\bThetabar)
	=&
	\omega_1(\bHcheck_1,A_1;\bThetabar_1)
	(1+\bar\beta_{21})
	\left\{
	Y_2-\mubar_2^v(\bUvec)
	\right\}
	+
	\omega_2(\bHcheck_2,A_2,\bThetabar_2)Y_3-
	\mubar_{3\omega_2}(\bUvec)\\
	-&
	\bar\beta_{21}\left\{\omega_2(\bHcheck_2,A_2,\bThetabar_2)Y_2-
	\mubar_{2\omega_2}(\bUvec)\right\}
	-
	\Qopt_{2-}(\bH_2; \bthetabar_2)
	\left\{
	\omega_2(\bHcheck_2,A_2,\bThetabar_2)-\mubar_{\omega_2}(\bUvec)
	\right\}
	\end{align*}

	Next note that 
	\[
	\int\left(\Vsc_{\bTheta,\mubar}(\bUvec)+\nu\subSSLDR(\bL;\bTheta)\right)d\Pbb_{\bL}\bigg|_{\bTheta=\bThetabar}
	=
	\int\Vsc\subSUPDR(\bL;\bTheta)d\Pbb_{\bL}\bigg|_{\bTheta=\bThetabar},
	\]
	where $\Vsc\subSUPDR(\bL;\bTheta)$ is defined in \eqref{eq: lab value fun}.
	Finally, all random variables in the expression of $\psi^{v}_{\subSSLDR}(\bL;\bThetabar)$ are bounded by Assumptions \ref{assumption: covariates} and \ref{assumption: donsker w} we have $\mathbb{E}\left[\psi^{v}_{\subSSLDR}(\bL;\bThetabar)^2\right]<\infty$, the central limit theorem yields that
	\[
	\sqrt n
	\left\{
	\Pbb_{N}\left[\Vsc_{\bThetahat,\muhat}\right]-g_1(\bThetabar)
	\right\}=
    \frac{1}{\sqrt n}\sum_{i=1}^n\psi^{v}_{\subSSLDR}(\bL_i;\bThetabar)
	+o_\Pbb\left(1\right)\stackrel{d}{\longrightarrow}N\left(0,\sigma^2\subSSLDR\right).
	\]
	\end{proof}
	
	\begin{proof}[Proof of Lemma \ref{lemma: CLT and expectation terms}]
	I) We start with $\sqrt n\left\{\Pbb_N\left[\Vsc_{\bThetabar,\mubar}\right]-g_1(\bThetabar)\right\}$. Note that $\Vsc_{\bThetabar, \mubar}(\bUvec)$ is a deterministic function of random variable $\bUvec$ as parameters and imputation functions are fixed.
	We have that $\Ebb\left[\Vsc_{\bThetabar, \mubar}(\bUvec)^2\right]<\infty$ holds by Assumption \ref{assumption: covariates} \& \ref{assumption: donsker w}. Thus the central limit theorem yields $\mathbb{G}_N\left\{\Vsc_{\bThetabar, \mubar}\right\}\stackrel{d}{\longrightarrow}\mathcal{N}\left(0,Var\left[\Vsc_{\bThetabar, \mubar}\right]\right),$ therefore \[
	\sqrt n\left\{\Pbb_N\left[\Vsc_{\bThetabar,\mubar}\right]-g_1(\bThetabar)\right\}=\sqrt{\frac{n}{N}}\mathbb{G}_N\left\{\Vsc_{\bThetabar, \mubar}\right\}=O_\Pbb\left(\frac{\sqrt n}{N}\right)=o_\Pbb\left(1\right).
	\]
	II) We next consider $\sqrt n\left\{g_1(\bThetahat)-g_1(\bThetabar)\right\}$. Using a Taylor series expansion
	\be
	g_1(\bThetahat)
	=
	g_1(\bThetabar)
	+
	(\bthetahat-\bthetabar)\trans\frac{\partial}{\partial \btheta}g_1(\bThetabar)
	+
	(\bxihat-\bxibar)\trans\frac{\partial}{\partial \bxi}g_1(\bThetabar)
	+O_\Pbb\left(n^{-1}\right),
	\ee
	as both $\|\bthetahat-\bthetabar\|_2^2=O_\Pbb\left(n^{-1}\right)$ and $\|\bxihat-\bxibar\|_2^2=O_\Pbb\left(n^{-1}\right)$ by Theorems \ref{theorem: unbiased theta2}, \ref{theorem: unbiased theta1} and Assumption \ref{assumption: donsker w},
	therefore 
	\be
	\sqrt n\left\{g_1(\bThetahat)-g_1(\bThetabar)\right\}
	=&
	\sqrt n(\bthetahat-\bthetabar)\trans\frac{\partial}{\partial \btheta}g_1(\bThetabar)
	+
	\sqrt n(\bxihat-\bxibar)\trans\frac{\partial}{\partial \bxi}g_1(\bThetabar)
	+o_\Pbb(1).
	\ee
    We can write 
	\begin{align*}
	&\sqrt n\left\{g_1(\bThetahat)-g_1(\bThetabar)\right\}
	=
	\frac{\partial}{\partial \btheta}g_1(\bThetabar)\frac{1}{\sqrt n}\sum_{i=1}^n\bpsi^\theta(\bUvec_i)
	+
	\frac{\partial}{\partial \bxi}g_1(\bThetabar)\frac{1}{\sqrt n}\sum_{i=1}^n\bpsi^\xi(\bUvec_i)
	+o_\Pbb(1).
	\end{align*}
	\end{proof}
	
	\begin{proof}[Proof of Lemma \ref{lemma: centered sample average}]
	
	We consider $\sqrt n\left\{\left(\Pbb_N\left[\Vsc_{\bThetahat,\mubar}\right]-g_1(\bThetahat)\right)-\bigg(\Pbb_N\left[\Vsc_{\bThetabar,\mubar}\right]-g_1(\bThetabar)\bigg)\right\}$, recall that $d_t(\bHcheck_t,\btheta_t)=I(\bH_{t1}\trans\bgamma_t>0)$ $t=1,2$, thus the inverse probability weight functions are defined as 
    \begin{align*}
    \omega_1(\bHcheck_1,A_1,\bTheta)
    &\equiv
    \frac{I(\bH_{11}\trans\bgamma_1>0)A_1}{\pi_1(\bHcheck_1;\bxi_1)}+\frac{\{1-I(\bH_{11}\trans\bgamma_1>0)\}\{1-A_1\}}{1-\pi_1(\bHcheck_1;\bxi_1)}, \quad \mbox{and}\\ \omega_2(\bHcheck_2,A_2,\bTheta)
    &\equiv
    \omega_1(\bHcheck_1,A_1,\bTheta)\left(\frac{I(\bH_{21}\trans\bgamma_2>0)A_2}{\pi_2(\bHcheck_2;\bxi_2)}+\frac{\{1-I(\bH_{21}\trans\bgamma_2>0)\}\{1-A_2\}}{1-\pi_2(\bHcheck_2;\bxi_2)}\right).
    \end{align*}
    Define the class 
    \[
	\ell_t=\{I\left(\bH_t\trans\bgamma_t\ge0\right):\mathcal{H}_{t1},\bgamma\in\mathbb{R}^{q_t}\},\:t=1,2
	\]
    and the collection of half spaces $\mathcal{C}_\ell\equiv\left\{\bH_t\in\mathbb{R}^{q_t}:\bH_t\trans\bgamma_t\ge0,\bgamma\in\mathbb{R}^{q_t},t\in\{1,2\}\right\}$. By \citet{Dudley} $\mathcal{C}_\ell$ is a VC class of VC dimension $q_t+1$. Next by \cite{vanderVaartAadW1996WCaE} we have that as $\mathcal{C}_\ell$ is a VC-class $\ell_t$ is a class of the same index. Finally, by Theorem 2.6.7 we have that $\ell_t$ is a $\Pbb$-Donsker class.
    Next define the following function
	\begin{align*}
	f_{\bTheta}(\bUvec)
	=&
	\Qopt_1(\bHcheck_1;\btheta_1)
	+\omega_1(\bHcheck_1,A_1,\bTheta)
	\left[(1+\beta_{21})\mubar_2^v(\bUvec)-
	 \Qopt_1(\bHcheck_1;\btheta_1)+\Qopt_{2-}(\bH_2;\btheta_2)
	\right]\\
	+&
	\mubar^v_{3\omega_2}(\bUvec)-\beta_{21}\mubar^v_{2\omega_2}(\bUvec)-\Qopt_{2-}(\bH_2;\btheta_2)\mubar^v_{\omega_2}(\bUvec).
	\end{align*}
	
	We define the associated class of functions $\mathcal{C}_1=\left\{
	f_{\bTheta}(\bUvec)
	|\bUvec,\bTheta\in\mathcal{S}(\delta)
	\right\}.$\\
	\indent i) By Assumptions \ref{assumption SS linear model}, \ref{assumption: donsker w} and
	Theorem 19.5 in \cite{vaart_donsker}, $\ell_t,\:\mathcal{W}_t,\:\mathcal{Q}_t,t=1,2$ are $\Pbb$-Donsker classes. Thus it follows that $\mathcal{C}_1$ is a Donsker class. 
	
	ii) We estimate ${\bxi}_1,{\bxi}_2$ for (\ref{logit_Ws}) with their maximum likelihood estimators, $\bxihat_1,\bxihat_2$, solving $\Pbb_n\left[S_t(\bxi_t)\right]=\bzero, t=1,2$. By Assumption (\ref{assumption: donsker w}) and Theorem 5.9 in \cite{vaart_donsker} $\bxihat_t\stackrel{p}{\longrightarrow}\bxibar_t, t=1,2$. Next, by Theorems \ref{theorem: unbiased theta2}, \ref{theorem: unbiased theta1}, under Assumptions \ref{assumption: covariates}, \ref{assumption: Q imputation}, $\bthetahat_t\stackrel{p}{\longrightarrow}\bthetabar_t, t=1,2$. Thus $\Pbb\left(\bThetahat\in\mathcal{S}(\delta)\right)\rightarrow1,\:\forall\delta.$ 
	
	iii) We next show  $\int\left(\Vsc_{\bThetahat,\mubar}-\Vsc_{\bThetabar,\mubar}\right)^2d\Pbb_{\bUvec}\longrightarrow0.$
	By Assumptions \ref{assumption: donsker w} (ii), \ref{assumption: V imputation}, and bounded covariates and there exists a constant $c\in\mathbb{R}$ such that we can write

	\begin{align*}
	&\int\left(\Vsc_{\bThetahat,\mubar}-\Vsc_{\bThetabar,\mubar}\right)^2d\Pbb_{\bUvec}\\
	\le&
	\int
	\bigg(\Qopt_1(\bH_1;\bthetahat_1)- \Qopt_1(\bH_1;\bthetabar_1)\bigg)^2d\Pbb_{\bUvec}\\
	+&c
	\int\bigg(\frac{1}{1-\pi_1(\bH_1;\bxihat_1)}-\frac{1}{1-\pi_1(\bH_1;\bxibar_1)}\bigg)^2d\Pbb_{\bUvec}
	\\
	+&c
	\int
	\left(\frac{1}{\pi_1(\bH_1;\bxihat_1)}-\frac{1}{ \pi_1(\bH_1;\bxibar_1)}\right)^2
	\\
	+&
	c\int\left\{\Qopt_{2-}(\bHcheck_2;\bthetahat_2)-\Qopt_{2-}(\bH_2;\bthetabar_2)\right\}^2d\Pbb_{\bUvec}
	\\
	+&
	c\int\left\{I(\bH_{11}\trans\bgammahat_1>0)-I(\bH_{11}\trans\bgammabar_1>0)\right\}^2d\Pbb_{\bUvec}
	\\
	+&
	\left(\hat\beta_{21}-\bar\beta_{21}\right)^2
	\\
	+&
	c
	\int
	\left(\bH_{20}\trans\bbetabar_{22}+[\bH_{20}\trans\bgammabar_2]_+-\bH_{20}\trans\bbetahat_{22}-[\bH_{20}\trans\bgammahat_2]_+
	\right)^2
	d\Pbb_{\bUvec}
	\end{align*}
	
	 where we use $(a-b)^2,(a+b)^2\leq2a^2+2b^2\:\forall a,b\in\mathbb{R}$, $\dhat_1,A_1\leq1$ for all $\bH\in\mathcal{H}$, and boundedness of $\bthetahat_t,t=1,2$ by Assumptions \ref{assumption: covariates}-\ref{assumption SS linear model}. Next note that all terms outside integrals are bounded by Assumptions \ref{assumption: covariates}-\ref{assumption SS linear model}. Finally we consider terms within the integrals with the following example
	
	\begin{align*}
	\int
	\left(\Qopt_{2-}(\bH_2;\bthetahat_2)- \Qopt_{2-}(\bH_2;\bthetabar)\right)^2d\Pbb_{\bUvec}
	=&\int\left(
	\bH_{20}\trans\bbetahat_{22}+[\bH_{21}\trans\bgammahat_2]_+
	-\bH_{20}\trans\bbetabar_{22}-[\bH_{21}\trans\bgammabar_2]_+
	\right)^2d\Pbb_{\bUvec}\\
	=&
	4\|
	\bbetahat_{22}-\bbetabar_{22}\|_2^2\int\bH_{20}\trans\bH_{20}
	d\Pbb_{\bUvec}\\
	+&
	4
	\|\bgammahat_2-\bgammabar_2\|_2^2\int\bH_{21}\trans\bH_{21}
	d\Pbb_{\bUvec}=O_\Pbb(n^{-1}),
	\end{align*}
	
	which follows from Theorem \ref{theorem: unbiased theta2} and Lemma \ref{lemm_gamma_difs} (a). All similar terms can be handled accordingly. We get the convergence in probability to 0: $\int\left(\Vsc_{\bThetahat,\mubar}-\Vsc_{\bThetabar,\mubar}\right)^2d\Pbb_{\bUvec}\rightarrow0$ as all other terms within expectation are $O_\Pbb\left(n^{-1}\right)$ by the dominating convergence theorem, boundedness
	conditions as stated in Assumptions \ref{assumption: Q imputation}, \ref{assumption: donsker w}, and the consistency of $\bxihat$ and $\bthetahat$ as $\Pbb\left(\bThetahat\in\mathcal{S}(\delta)\right)\rightarrow1,\:\forall\delta>0.$ 
	
	Finally, we have i) $\Pbb\left(\bThetahat\in\mathcal{S}(\delta)\right)\rightarrow1,$ ii) $\mathcal{C}_1$ is a Donsker class, and\\ iii) $\int\left(\Vsc_{\bThetahat,\mubar}-\Vsc_{\bThetabar,\mubar}\right)^2d\Pbb_{\bUvec}\longrightarrow0$, then by Theorem 2.1 in \cite{wellner_emp}, 
	\[
	\sqrt {\frac{n}{N}}\sqrt n\left\{\left(\Pbb_N\left[\Vsc_{\bThetahat,\mubar}\right]-g_1(\bThetahat)\right)-\bigg(\Pbb_N\left[\Vsc_{\bThetabar,\mubar}\right]-g_1(\bThetabar)\bigg)\right\}=\sqrt {\frac{n}{N}}o_\Pbb(1).
	\]
	\end{proof}
	
	\begin{proof}[Proof of Lemma \ref{lemma: error diff}]
	I) First note that from the empirical normal equations \eqref{V function reffitting}, we have that the solution $\etahat_2^v$ satisfies $\etahat_2^v-\eta_2^v=O_\Pbb\left(n^{-\frac{1}{2}}\right)$. Therefore
    \begin{align*}
    \sup_{\bUvec}\left|
	\muhat_2^v(\bUvec)-\mu_2^v(\bUvec)
	\right|
	&=
	\sup_{\bUvec}\left|
	\frac{1}{K}
    \mhat_2\supnk(\bUvec)+
    \etahat_2^v-
    m_2(\bUvec)+
    \eta_2^v
	\right|\\
	&\le
	\frac{1}{K}
	\sup_{\bUvec}\left|
    \mhat_2\supnk(\bUvec)+
    m_2(\bUvec)
	\right|
	+
	\left|\etahat_2^v-\eta_2^v\right|\\
	&=o_\Pbb(1)+O_\Pbb\left(n^{-\frac{1}{2}}\right)=o_\Pbb(1),
    \end{align*}
    where we additionally use Assumption \ref{assumption: V imputation} for the difference of estimated and true imputation models $\hat m_2$, $m_2$. Similarly $\sup_{\bUvec}
	\left|
	\muhat^v_{t\omega_2}(\bUvec)-\mubar^v_{t\omega_2}(\bUvec)\right|
	=o_\Pbb(1)$, $\sup_{\bUvec}
	\left|
	\muhat^v_{\omega_2}(\bUvec)-\mubar^v_{\omega_2}(\bUvec)\right|
	=o_\Pbb(1)$, $t=2,3$. Next, using the triangle and Jensen's inequalities, we have
	\begin{align*}
	&\Pbb_N\left[\mathcal{E}_{\bThetahat}-\mathcal{E}_{\bThetabar}\right]\\
	\le&
	\Pbb_N\bigg|
	\omega_1(\bHcheck_1,A_1;\bThetahat_1)(1+\hat\beta_{21})-\omega_1(\bHcheck_1,A_1;\bThetabar_1)(1+\bar\beta_{21})\bigg|\sup_{\bUvec}\left|
	\muhat_2^v(\bUvec)-\mubar_2^v(\bUvec)
	\right|
	\\
	+&
	\bigg|
	\hat\beta_{21}-\bar\beta_{21}\bigg|\sup_{\bUvec}\left|\muhat_{2\omega_2}(\bUvec)-\mu_{2\omega_2}(\bUvec)\right|
	\\
	+&
	\Pbb_N\bigg|
	\Qopt_{2-}(\bH_2;\bthetahat_2)-\Qopt_{2-}(\bH_2;\bthetabar_2)
	\bigg|
	\sup_{\bUvec}\left|\muhat_{\omega_2}(\bUvec)-\mubar_{\omega_2}(\bUvec)\right|
	\\
	\le&
	\Pbb_N\bigg|
	\omega_1(\bHcheck_1,A_1;\bThetahat_1)-\omega_1(\bHcheck_1,A_1;\bThetabar_1)\bigg|o_{\Pbb}(1)
	+
	\Pbb_N\bigg|
	\omega_1(\bHcheck_1,A_1;\bThetahat_1)\hat\beta_{21}-\omega_1(\bHcheck_1,A_1;\bThetabar_1)\bar\beta_{21}\bigg|o_{\Pbb}(1)
	\\
	+&
	\bigg|
	\hat\beta_{21}-\bar\beta_{21}\bigg|o_{\Pbb}(1)
	+
	\Pbb_N\bigg|
	\left(\bbetahat_{22}-\bbetahat_{22}\right)\trans\bH_{20}+[\bgammahat_2\trans\bH_{21}]_+-[\bgammabar_2\trans\bH_{21}]_+\bigg|
	o_{\Pbb}(1).
	\end{align*}
	
	By Theorem \ref{theorem: unbiased theta2} we have $\bthetahat_2-\bthetabar_2=O_\Pbb\left(n^{-\frac{1}{2}}\right)$, also from Lemma \ref{lemm_gamma_difs} (a) it follows that $\Pbb_N\left(\left[\bH_{21}\trans\bgammahat_2\right]_+-\left[\bH_{21}\trans\bgammabar_2\right]_+\right)=O_\Pbb\left(n^{-\frac{1}{2}}\right)$, hence as covariates are bounded we have
	\begin{align*}
	&\bigg|
	\hat\beta_{21}-\bar\beta_{21}\bigg|o_{\Pbb}(1)+\Pbb_N\bigg|
	\left(\bbetahat_{22}-\bbetahat_{22}\right)\trans\bH_{20}+[\bgammahat_2\trans\bH_{21}]_+-[\bgammabar_2\trans\bH_{21}]_+\bigg|\\
	&\le \left\{o_{\Pbb}(1)+\sup_{\bH_{20}}\|\bH_{20}\|_2\|\bbetahat_{22}-\bbetabar_{22}\|_2+\sup_{\bH_{21}}\|\bH_{21}\|_2\right\}O_\Pbb\left(n^{-\frac{1}{2}}\right)
	=O_\Pbb\left(n^{-\frac{1}{2}}\right).
	\end{align*} 
	Next, we can write 
\[
\omega_1(\bH_1,A_1;\bThetahat_1)
=
I\left\{A_1=d_1\left(\bH_1;\bxihat_1\right)
\right\}\left\{
\frac{A_1}{\pi_1\left(\bH_1;\bxihat_1\right)}
+
\frac{1-A_1}{1-\pi_1\left(\bH_1;\bxihat_1\right)}
\right\}.
\]
By Lemma \ref{lemm_gamma_difs} (b) it follows that \begin{align*}
\Pbb_N\left[I\left\{A_1=d_1\left(\bH_1;\bxihat_1\right)\right\}-I\left\{A_1=d_1\left(\bH_1;\bxibar_1\right)
\right\}\right]
&=
O_\Pbb\left(n^{-\frac{1}{2}}\right),\\
\Pbb_N\left[\frac{A_1}{\pi_1(\bH_1;\bxihat_1)}-\frac{A_1}{\pi_1\left(\bH_1;\bxibar_1\right)}\right]&=O_\Pbb\left(n^{-\frac{1}{2}}\right),\\
\Pbb_N\left[\frac{1-A_1}{1-\pi_1(\bH_1;\bxihat_1)}-\frac{1-A_1}{1-\pi_1\left(\bH_1;\bxibar_1\right)}\right]&=O_\Pbb\left(n^{-\frac{1}{2}}\right).
\end{align*}
Using the above and Lemma \ref{lemma: Op product} we get 
\begin{align*}
\hat\beta_{21}\Pbb_N
		\left\{\omega_1(\bHcheck_1,A_1;\bThetahat_1)\right\}
		-
		\bar\beta_{21}\Pbb_N
		\left\{\omega_1(\bHcheck_1,A_1;\bThetabar_1)\right\}
		&=O_\Pbb\left(n^{-\frac{1}{2}}\right),\\
\Pbb_N
		\left\{\omega_1(\bHcheck_1,A_1;\bThetahat_1)\right\}
		-
		\Pbb_N
		\left\{\omega_1(\bHcheck_1,A_1;\bThetabar_1)\right\}&=O_\Pbb\left(n^{-\frac{1}{2}}\right).
\end{align*}

	From the above we get
	\[
	\Pbb_N\left\{\mathcal{E}_{\bThetahat}-\mathcal{E}_{\bThetabar}\right\}
	=
	O_\Pbb\left(n^{-\frac{1}{2}}\right)
	o_\Pbb(1).
	\]

	II) To show the relevant result, we first recall the definition of $\nu\subSSLDR$ from Theorem \ref{thrm_ssV_fun} and show that
	\begin{align}\label{eq: nu IF}
	\begin{split}
	\frac{1}{\sqrt n}\sum_{i=1}^n\nu_{\subSSLDR}(\bL_i;\bThetahat)
	&=
	\frac{1}{\sqrt n}\sum_{i=1}^n
	\nu\subSSLDR(\bL_i;\bThetabar)
	\\
	&+
	\frac{1}{\sqrt n}\sum_{i=1}^n
	\left(
	\frac{\partial}{\partial \btheta}\Ebb\left[\nu\subSSLDR(\bL_i;\bThetabar)\right]\right)\trans
	\bpsi^\theta(\bL_i)\\
	&+
	\frac{1}{\sqrt n}\sum_{i=1}^n
	\left(
	\frac{\partial}{\partial \bxi}\Ebb\left[\nu\subSSLDR(\bL_i;\bThetabar)\right]\right)\trans
	\bpsi^\xi(\bL_i)
	+
	o_\Pbb(1).
	\end{split}
	\end{align}
	We start expanding $\frac{1}{\sqrt n}\sum_{i=1}^n\nu_{\subSSLDR}(\bL_i;\bThetahat)$ as
	\begin{align*}
	&\frac{1}{\sqrt n}\sum_{i=1}^n\nu_{\subSSLDR}(\bL_i;\bThetahat)\\
	&=
	\mathbb{G}_n\left\{\nu_{\subSSLDR}(\bL;\bThetabar)\right\}
	+
	\mathbb{G}_n\left\{\nu_{\subSSLDR}(\bL;\bThetahat)-\nu_{\subSSLDR}(\bL;\bThetabar)\right\}
	+
	\sqrt n\int\nu_{\subSSLDR}(\bL;\bThetahat)d\Pbb_{\bL},
	\end{align*}
	we next consider the limit of each term above.

	1) Using a Taylor series expansion on $\int\nu_{\subSSLDR}(\bL;\bThetahat)d\Pbb_{\bL}$ we get
	\[
	\int\nu_{\subSSLDR}(\bL;\bThetahat)d\Pbb_{\bL}=\int\nu_{\subSSLDR}(\bL;\bThetabar)d\Pbb_{\bL}+\left(\bThetahat-\bThetabar\right)\trans\frac{\partial}{\partial\bTheta}\int\nu_{\subSSLDR}(\bL;\bTheta)d\Pbb_{\bL}\bigg|_{\bTheta=\bThetabar}+O_{\Pbb}\left(n^{-1}\right),
	\]
	where the remaining terms are of order $O\left\{\left(\bThetahat-\bThetabar\right)^2\right\}$ which by Theorems \ref{theorem: unbiased theta2} \& \ref{theorem: unbiased theta1} are $O_{\Pbb}\left(n^{-1}\right)$. Next note that from \eqref{V function impute_constraint} it follows that $\int\nu_{\subSSLDR}(\bL;\bThetabar)d\Pbb_{\bL}=0$, and thus letting $g_2(\bTheta)=\int\nu_{\subSSLDR}(\bL;\bTheta)d\Pbb_{\bL}$ we have
    \be
	\sqrt ng_2(\bThetahat)=\sqrt n(\bthetahat-\bthetabar)\trans\frac{\partial}{\partial \bTheta}g_2(\bTheta)\bigg|_{\bTheta=\bThetabar}
	+
	\sqrt n(\bxihat-\bxibar)\trans\frac{\partial}{\partial \bxi}g_2(\bTheta)\bigg|_{\bTheta=\bThetabar}
	+o_\Pbb(1).
	\ee
	We can write 
	\begin{align*}
	&\sqrt ng_2(\bThetahat)
	=
	\frac{1}{\sqrt n}\sum_{i=1}^n\bpsi^\theta(\bL_i)\trans
	\frac{\partial}{\partial \btheta}g_2(\bTheta)\bigg|_{\bTheta=\bThetabar}
	+
	\frac{1}{\sqrt n}\sum_{i=1}^n\bpsi^\xi(\bL_i)\trans\frac{\partial}{\partial \bxi}g_2(\bTheta)\bigg|_{\bTheta=\bThetabar}
	+o_\Pbb(1).
	\end{align*}
	2) We next show
	\[
	\mathbb{G}_n\left\{\nu_{\subSSLDR}(\bL;\bThetahat)-\nu_{\subSSLDR}(\bL;\bThetabar)\right\}=o_\Pbb(1),
	\]
	define the class 
    \[
	\ell_t=\{I(\bH\trans\bgamma_t\ge0):\mathcal{H}_{t1},\bgamma\in\mathbb{R}^{q_t}\},\:t=1,2
	\]
    and the collection of half spaces $\mathcal{C}_\ell\equiv\left\{\bH_t\in\mathbb{R}^{q_t}:\bH\trans\bgamma_t\ge0,\bgamma\in\mathbb{R}^{q_t},t\in\{1,2\}\right\}$, by \cite{Dudley} $\mathcal{C}_\ell$ is a VC class of VC dimension $q_t+1$, next by \cite{vanderVaartAadW1996WCaE} we have that as $\mathcal{C}_\ell$ is a VC-class $\ell_t$ is a class of the same index. Finally, by Theorem 2.6.7 we have that $\ell_t$ is a Donsker class.
	\begin{align*}
	f_{\bTheta}(\bL_i)
	=&
	\omega_1(\bHcheck_{1i},A_{1i};\bTheta_1)
	(1+\beta_{21})
	\left\{
	Y_{2i}-\mubar_2^v(\bUvec_i)
	\right\}
	+
	\omega_2(\bHcheck_{2i},A_{2i};\bTheta_2)
	Y_{3i}-
	\mubar_{3\omega_2}(\bUvec_i)\\
	-&
	\beta_{21}\left\{
	\omega_2(\bHcheck_{2i},A_{2i};\bTheta_2)Y_{2i}-
	\mubar_{2\omega_2}(\bUvec_i)\right\}
	-
	\Qopt_{2-}(\bH_{2i};\btheta_2)
	\left\{
	\omega_2(\bHcheck_{2i},A_{2i};\bTheta_2)
	-
	\mubar_{\omega_2}(\bUvec_i)
	\right\},
	\end{align*}
	we define the class of functions
	$
	\mathcal{C}_2=\left\{
	f_{\bTheta}(\bL)
    |\bTheta\in\mathcal{S}(\delta)
	\right\}.
	$
	
	i) By Assumptions \ref{assumption SS linear model}, \ref{assumption: donsker w} and
	Theorem 19.5 in \cite{vaart_donsker}, $\mathcal{W}_t,\:\mathcal{Q}_t,t=1,2$ are a $\Pbb$-Donsker class. Additionally, the terms in the $\omega_t(\bH_t,A_t;\bTheta_t)$ functions of the form $\bH_{t1}\trans\bgamma_tI(\bH_{t1}\trans\bgamma_t>0)$ constitute a $\Pbb$-Donsker class, as $\bH_{t1}\trans\bgamma_t$ is linear in $\bgamma_t$ and $I(\bH_{t1}\trans\bgamma_t>0)$ is $\Pbb$-Donsker. Thus it follows that $\mathcal{C}_2$ is a $\Pbb$-Donsker class. 
	
	ii) We estimate ${\bxi}_1,{\bxi}_2$ for (\ref{logit_Ws}) with their maximum likelihood estimators, $\bxihat_1,\bxihat_2$, solving $\Pbb_n\left[S_t(\bxi_t)\right]=\bzero, t=1,2$, by Assumption \ref{assumption: donsker w} and Theorem 5.9 in \cite{vaart_donsker} $\bxihat_t\stackrel{p}{\longrightarrow}\bxibar_t, t=1,2$. Next, by Theorems \ref{theorem: unbiased theta2}, \ref{theorem: unbiased theta1}, under Assumptions \ref{assumption: covariates}, \ref{assumption: Q imputation}, $\bthetahat_t\stackrel{p}{\longrightarrow}\bthetabar_t, t=1,2$. Thus $\Pbb\left(\bThetahat\in\mathcal{S}(\delta)\right)\stackrel{p}{\rightarrow}1,\:\forall\delta.$ Therefore, we have $\nu_{\subSSLDR}(\bL;\bThetahat)\in\mathcal{C}_2$ with high probability.
	
	iii) We then show
	$
	\int\left\{\nu_{\subSSLDR}(\bL;\bThetahat)-\nu_{\subSSLDR}(\bL;\bThetabar)
	\right\}^2d\Pbb_{\bL}\longrightarrow0.
	$
	Using simple algebra for a large enough constant $c$ we have

	\begin{align*}
	&\int\left\{\nu\subSSLDR(\bL;\bThetahat)-\nu_{\subSSLDR}(\bL;\bThetabar)
	\right\}^2d\Pbb_{\bL}\\
	\le
	&c\sup_{Y_2,\bUvec}\left\{
	Y_2-\mubar_2^v(\bUvec)
	\right\}^2
	\\
	\times&
	\sup_{\bHcheck_1,A_1}\left\{ (1+\hat\beta_{21})\omega_1(\bHcheck_1,A_1;\bThetahat_1)-(1+\bar\beta_{21})\omega_1(\bHcheck_1,A_1;\bThetabar_1)\right\}^2
	\\
	+&
	c\sup_{Y_3}
	Y_3^2
	\sup_{\bHcheck_2,A_2}
	\left\{
	\omega_2(\bHcheck_2,A_2;\bThetahat_2)-
	\omega_2(\bHcheck_2,A_2;\bThetabar_2)
	\right\}^2
	\\
	+&
	c\sup_{Y_2}
	Y_2^2
	\sup_{\bHcheck_2,A_2}\left\{
	\hat\beta_{21}\omega_2(\bHcheck_2,A_2;\bThetahat_2)-
	\bar\beta_{21}
	\omega_2(\bHcheck_2,A_2;\bThetabar_2)
	\right\}^2\\
	+&
	c\sup_{\bUvec}
	\mubar_{2\omega_3}(\bUvec)^2
	\left(
	\hat\beta_{21}-
	\bar\beta_{21}
	\right)^2\\
	+&
	c
	\sup_{\bHcheck_2,A_2}\left\{
	\Qopt_{2-}(\bH_2;\bthetahat_2)\omega_2(\bHcheck_2,A_2;\bThetahat_2)-\Qopt_{2-}(\bH_2;\bthetabar_2)
	\omega_2(\bHcheck_2,A_2;\bThetabar_2)
	\right\}^2\\
	+&
	c\sup_{\bUvec}\mubar_{2\omega_2}(\bUvec)^2
	\sup_{\bH_2}\left\{
	\Qopt_{2-}(\bH_2;\bthetahat_2)-\Qopt_{2-}(\bH_2;\bthetabar_2)
	\right\}^2\\
	&\stackrel{p}{\longrightarrow} 0
	\end{align*}
	where we use $(a-b)^2,(a+b)^2\leq2a^2+2b^2\:\forall a,b\in\mathbb{R}$, boundedness of $\bThetabar$ and covariates by Assumptions \ref{assumption: covariates}, \ref{assumption: Q imputation} to bound all supremum quantities. 
	
	By Theorems \ref{theorem: unbiased theta2} and \ref{theorem: unbiased theta1} we have $\bthetahat_2-\bthetabar_2=O_\Pbb\left(n^{-\frac{1}{2}}\right)$, $\bthetahat_1-\bthetabar_1=O_\Pbb\left(n^{-\frac{1}{2}}\right)$, also from Lemma \ref{lemm_gamma_difs} (a) it follows that 
	\begin{align*}
	&\sup_{\bH_2}\left\{\Qopt_{2-}(\bH_2;\bthetahat_2)-\Qopt_{2-}(\bH_2;\bthetabar_2)\right\}^2\\
	&\le
	2\sup_{\bH_{20}}\|\bH_{20}\|_2^2\|\bbetahat_{22}-\bbetabar_{22}\|_2^2+
	2\sup_{\bH_{21}}\|\bH_{21}\|_2^2
	\|\bgammahat_{22}-\bgammabar_{22}\|_2\\
	&=
	O_\Pbb\left(n^{-1}\right).
	\end{align*} 
	Next, we can write 
\begin{align*} 
\omega_1(\bH_1,A_1;\bThetahat_1)
=&
I\left\{A_1=d_1\left(\bH_1;\bxihat_1\right)
\right\}\left\{
\frac{A_1}{\pi_1\left(\bH_1;\bxihat_1\right)}
+
\frac{1-A_1}{1-\pi_1\left(\bH_1;\bxihat_1\right)}
\right\}\\
\omega_2(\bHcheck_2,A_2;\bThetahat_1)
=&
\omega_1(\bH_1,A_1;\bThetahat_1)
I\left\{A_2=d_2\left(\bH_2;\bxihat_2\right)
\right\}\left\{
\frac{A_2}{\pi_2\left(\bHcheck_2;\bxihat_2\right)}
+
\frac{1-A_2}{2-\pi_2\left(\bHcheck_2;\bxihat_2\right)}
\right\}.
\end{align*} 
By Lemma \ref{lemm_gamma_difs} (b) it follows that 
\begin{align*}
	\sup_{\bH_1,\mathbf{a}_1}
	\bigg|I(\dhat_1=A_1)-I( \dbar_1=A_1)\bigg|=&o_\Pbb(1),\\
	\sup_{\bH_2,\mathbf{a}_2}
	\bigg|I(\dhat_1=A_1)I(A_2=\dhat_2)-I( \dbar_1=A_1)I( \dbar_2=A_2)\bigg|=&o_\Pbb(1),\\
	\sup_{\bH_1}
	\bigg|\frac{1}
    {\pi_1(\bH_1;\bxihat_1)}
    -
    \frac{1}{ \pi_1(\bH_1;\bxibar_1)}\bigg|=&O_\Pbb\left(n^{-\frac{1}{2}}\right).
	\end{align*}
Using the above and Lemma \ref{lemma: Op product} we get 
\begin{align*}
\sup_{\bHcheck_1,A_1}\left\{ (1+\hat\beta_{21})\omega_1(\bHcheck_1,A_1;\bThetahat_1)-(1+\bar\beta_{21})\omega_1(\bHcheck_1,A_1;\bThetabar_1)\right\}^2
		&=o_\Pbb\left(1\right),\\
		\sup_{\bHcheck_2,A_2}\left\{ (1+\hat\beta_{21})\omega_2(\bHcheck_2,A_2;\bThetahat_2)-(1+\bar\beta_{21})\omega_2(\bHcheck_2,A_2;\bThetabar_2)\right\}^2
		&=o_\Pbb\left(1\right),\\
\sup_{\bHcheck_2,A_2}\left\{
	\Qopt_{2-}(\bH_2;\bthetahat_2)\omega_2(\bHcheck_2,A_2;\bThetahat_2)-\Qopt_{2-}(\bH_2;\bthetabar_2)
	\omega_2(\bHcheck_2,A_2;\bThetabar_2)
	\right\}^2
	&=o_\Pbb\left(1\right),\\
	\sup_{\bHcheck_2,A_2}\left\{
	\hat\beta_{21}\omega_2(\bHcheck_2,A_2;\bThetahat_2)-\bar\beta_{21}
	\omega_2(\bHcheck_2,A_2;\bThetabar_2)
	\right\}^2
	&=o_\Pbb\left(1\right).
\end{align*}

	which gives us $\int\left\{\nu\subSSLDR(\bL;\bThetahat)-\nu_{\subSSLDR}(\bL;\bThetabar)
	\right\}^2d\Pbb_{\bL}\stackrel{p}{\rightarrow}0.$

	Therefore we have i) $\Pbb\left(\bThetahat\in\mathcal{S}(\delta)\right)\rightarrow1,\:\forall\delta,$ ii) $\mathcal{C}_2$ is a $\Pbb$-Donsker class, and\\ iii) $\int\left(\nu_{\subSSLDR}(\bL;\bThetahat)-\nu_{\subSSLDR}(\bL;\bThetabar)
	\right)^2d\Pbb_{\bL}\rightarrow0.
	$
	By Theorem 2.1 in \cite{wellner_emp} 
	\begin{align*}
	\frac{1}{\sqrt n}\sum_{i=1}^n\left\{\left(\nu_{\subSSLDR}(\bL_i;\bThetahat)-\mathbb{E}_{\mathbb{S}}[\nu_{\subSSLDR}(\bL;\bThetahat)]\right)-\left(\nu_{\subSSLDR}(\bL_i;\bThetabar)-\mathbb{E}_{\mathbb{S}}[\nu_{\subSSLDR}(\bL;\bThetabar)]\right)\right\}=o_\Pbb(1).
    \end{align*}
    by 1), 2) and noting that $\nu_{\subSSLDR}(\bL_i;\bThetabar)$ has mean zero we obtain the result in \eqref{eq: nu IF}.\\
	
	We next re-write $\sqrt n\Pbb_N\left[\mathcal{E}_{\bThetabar}\right]$ by expressing the estimated imputation functions in $\mathcal{E}_{\bThetabar}$ in terms of the labeled sample $\mathcal{L}$. Letting
	\[
	\hat C^{(1)}_{n,N} = \frac{(1+\bar\beta_{21})\Pbb_N
		\left\{
		\omega_1(\bHcheck_1,A_1;\bThetabar)\right\}}
		{(1+\hat\beta_{21})\Pbb_n\left\{\omega_1(\bHcheck_1,A_1;\bThetahat)\right\}},\:
	\hat C^{(2)}_{n,N} = \frac{\Pbb_N
		\left\{
		\Qopt_{2-}(\bH_2;\bthetabar_2)
		\right\}}
	{\Pbb_n
		\left\{
		\Qopt_{2-}(\bH_2;\bthetahat_2)
		\right\}},
	\]
	we can write:
	\begin{align*}
	&
	\frac{1}{N}\sum_{j=1}^N
	\omega_1(\bHcheck_{1j},A_{1j},\bThetabar)
	(1+\bar\beta_{21})
	\left\{
	\muhat_2^v(\bUvec_j)-\mubar_2^v(\bUvec_j)
	\right\}\\
	=&
	\frac{1}{N}\sum_{j=1}^N
	\omega_1(\bHcheck_{1j},A_{1j},\bThetabar)
	(1+\bar\beta_{21})
	\left\{
	\frac{1}{K}\sum_{k=1}^{K}\mhat_2\supnk(\bUvec_j)+\etahat_2^v
	-m_2(\bUvec_j)-\eta_2^v
	\right\}\\
	=&
	(1+\bar\beta_{21})
	\frac{1}{KN}\sum_{j=1}^N\sum_{k=1}^{K}
	\omega_1(\bHcheck_{1j},A_{1j},\bThetabar)
	\hat\Delta_2\supnk(\bUvec_j)
	+
	(\etahat_2^v-\eta_2^v)
	\frac{1}{N}\sum_{j=1}^N
	\omega_1(\bHcheck_{1j},A_{1j},\bThetabar)
	(1+\bar\beta_{21}),\\
	\end{align*}
	where the first step follows from constrains shown in \eqref{V function reffitting} and we simply regroup terms in the second step. 
	
	Next note that we can use Lemma \ref{lemm_chakrabortty} to replace 
	\[
	\Pbb_N\left[(1+\bar\beta_{21})\omega_1(\bHcheck_1,A_1,\bThetabar)
	\hat\Delta_2\supnk(\bUvec_j)\right]
	\quad
	\mbox{by}
	\quad
	\mathbb{E}_\mathcal{L}\left[
	(1+\bar\beta_{21})
	\omega_1(\bHcheck_1,A_1,\bThetabar)
	\hat\Delta_2\supnk(\bUvec_j)\right]+O_\Pbb\left(N^{-\frac{1}{2}}\right),
	\]
	using $\mathbb{E}_\mathcal{L}[\cdot]$ to denote expectation with respect to $\mathcal{L}$. Additionally, using \eqref{V function reffitting} and the definition of $\mubar_2^v(\bUvec)$ for the second term we get:	
	\begin{align*}
	&\frac{1}{N}\sum_{j=1}^N
	\omega_1(\bHcheck_{1j},A_{1j};\bThetabar)
	(1+\bar\beta_{21})
	\left\{
	\muhat_2^v(\bUvec_j)-\mubar_2^v(\bUvec_j)
	\right\}\\
	=&
	\mathbb{E}_\mathcal{L}\left[
	\frac{1}{{K}}\sum_{k=1}^{K}
	(1+\bar\beta_{21})
	\omega_1(\bHcheck_1,A_1,\bThetabar)
	\hat\Delta_2\supnk(\bUvec)\right]+O_\Pbb\left(N^{-\frac{1}{2}}\right)\\
	-&
	\hat C^{(1)}_{n,N}\frac{1}{n}\sum_{k=1}^{K}\sum_{i\in\mathcal{I}_k}
	(1+\hat\beta_{21})
	\omega_1(\bHcheck_{1i},A_{1i};\bThetahat)
	\hat\Delta_2\supnk(\bUvec_i)
	\\
	+&
	\hat C^{(1)}_{n,N}
	(1+\hat\beta_{21})
	\frac{1}{n}\sum_{i=1}^n
	\omega_1(\bHcheck_{1i},A_{1i};\bThetahat)
	\left\{
	Y_{2i}-\mubar_2^v(\bUvec_i)
	\right\}\\
	=&
	\left\{1+O_\Pbb\left(n^{-\frac{1}{2}}\right)\right\}	(1+\hat\beta_{21})
	\frac{1}{n}\sum_{i=1}^n
	\omega_1(\bHcheck_{1i},A_{1i};\bThetahat)
	\left\{
	Y_{2i}-\mubar_2^v(\bUvec_i)
	\right\}
	+O_\Pbb\left(n^{-\frac{1}{2}}c_{n^-_K}\right),
	\end{align*}
    where the last step follows from Assumption \ref{assumption: V imputation} and Lemma \ref{lemm_deltas} choosing $f$ to be the constant function 1, setting $\hat\Delta_{k}(\bUvec)=\hat\Delta_2\supnk(\bUvec)$, $\hat l(\bHcheck_1)=A_1I(\bH_{11}\trans\bgammahat_1>0)$, and $\hat\pi(\bHcheck_1)=\pi_1(\bHcheck_1;\bxihat_1)$
    and with $\hat C_{n,N}=\hat C_{n,N}^{(1)}$ -which satisfies $\hat C_{n,N}^{(1)}=1+O_\Pbb\left(n^{-\frac{1}{2}}\right)$ by Lemma \ref{lemm_gamma_difs} (c). 
    
    Using similar arguments we have
    \begin{align*}
    &\frac{1}{N}\sum_{j=1}^N
	\Qopt_{2-}(\bH_{2j};\bthetabar_2)\{\muhat_{\omega_2}(\bUvec_j)-\mubar_{\omega_2}(\bUvec_j)\}\\
	&=
	\frac{1}{N}\sum_{j=1}^N
	\Qopt_{2-}(\bH_{2j};\bthetabar_2)
	\left\{
	\frac{1}{K}\sum_{k=1}^{K}\mhat_{\omega_2}\supnk(\bUvec_j)+\etahat_{\omega_2}^v
	-m_{\omega_2}(\bUvec_j)-\eta_{\omega_2}^v
	\right\}\\
	&=
	\frac{1}{KN}\sum_{j=1}^N\sum_{k=1}^{K}
	\Qopt_{2-}(\bH_{2j};\bthetabar_2)
	\hat\Delta_{\omega_2k}(\bUvec_j)
	+
	\left(\etahat_{\omega_2}^v
	-\eta_{\omega_2}^v\right)
	\frac{1}{N}\sum_{j=1}^N
	\Qopt_{2-}(\bH_{2j};\bthetabar_2)
	\\
	&=
	\Ebb_\Lsc\left[\frac{1}{K}\sum_{k=1}^{K}
	\Qopt_{2-}(\bH_2;\bthetabar_2)
	\hat\Delta_{\omega_2k}(\bUvec)\right]
	-
	\hat C^{(2)}_{n,N}
	\frac{1}{n}\sum_{k=1}^{K}\sum_{i\in\mathcal{I}_k}
	\Qopt_{2-}(\bH_{2i};\bthetabar_2)
	\hat\Delta_{\omega_2k}(\bUvec_i)\\
	&+
	\hat C^{(2)}_{n,N}
	\frac{1}{n}\sum_{i=1}^n
	\Qopt_{2-}(\bH_{2i};\bthetahat_2)\{\omega_2(\bHcheck_{2i},A_{2i};\bThetahat)-\mubar_{\omega_2}(\bUvec_i)\}
	+
	O_\Pbb\left(N^{-\frac{1}{2}}\right)
	\\
	&=
	\left\{1+O_\Pbb\left(n^{-\frac{1}{2}}\right)\right\}
	\frac{1}{n}\sum_{i=1}^n
	\Qopt_{2-}(\bH_{2i};\bthetahat_2)\{\omega_2(\bHcheck_{2i},A_{2i};\bThetahat)-\mubar_{\omega_2}(\bUvec_i)\}
	+
	O_\Pbb\left(n^{-\frac{1}{2}}c_{n^-_K}\right),
	\end{align*}
	and for $t=2,3$
	    \begin{align*}
	\frac{1}{N}\sum_{j=1}^N
	\left\{
	\muhat_{t\omega_2}^v(\bUvec_j)-\mubar_{t\omega_2}^v(\bUvec_j)
	\right\}
	=&
	\frac{1}{KN}\sum_{j=1}^N
	\bigg\{
	\sum_{k=1}^{K}\mhat_{t\omega_2}\supnk(\bUvec_j)+\etahat_{t\omega_2}^v
	-m_{t\omega_2}(\bUvec_j)-\eta_{t\omega_2}^v
	\bigg\}\\
	=&
	\frac{1}{KN}\sum_{j=1}^N\sum_{k=1}^{K}
	\hat\Delta_{3\omega_2k}(\bUvec_j)
	+
	(\etahat_{t\omega_2}^v-\eta_{t\omega_2}^v)\\
	=&
	\Ebb_\Lsc\left[\frac{1}{{K}}\sum_{k=1}^{K}
	\hat\Delta_{3\omega_2k}(\bUvec)\right]
	-
	\frac{1}{n}\sum_{k=1}^{K}\sum_{i\in\mathcal{I}_k}
	\hat\Delta_{3\omega_2k}(\bUvec_i)\\
	+&
	\frac{1}{n}\sum_{i=1}^n
	\omega_2(\bHcheck_{2i},A_{2i};\bThetahat)Y_{ti}-\mubar_{t\omega_2}^v(\bUvec_i)
	+
	O_\Pbb\left(N^{-\frac{1}{2}}\right)
	\\
	=&
	\frac{1}{n}\sum_{i=1}^n
	\omega_2(\bHcheck_{2i},A_{2i};\bThetahat)Y_{ti}-\mubar_{t\omega_2}^v(\bUvec_i)
	+O_\Pbb\left(n^{-\frac{1}{2}}c_{n^-_K}\right),
	\end{align*}
	finally by Theorem \ref{theorem: unbiased theta2}, $\hat\beta_{21}-\bar\beta_{21}=O_\Pbb\left(n^{-\frac{1}{2}}\right)$.
	
	Therefore, recalling the definition of $\nu\subSSLDR$ from Theorem \ref{thrm_ssV_fun}, using the derivations above, we can write
	\begin{align}\label{eq: err IF equiv}
	\begin{split}
	\frac{\sqrt n}{N}\sum_{j=1}^N\mathcal{E}_{\bThetabar}(\bUvec_j)
	&=
	\frac{1}{\sqrt n}\sum_{i=1}^n\nu_{\subSSLDR}(\bUvec_i;\bThetahat)
	\\
	&+
	O_\Pbb\left(n^{-\frac{1}{2}}\right)
	\frac{1}{\sqrt n}\sum_{i=1}^n
	\omega_1(\bHcheck_1,A_1;\bThetahat_1)
	(1+\hat\beta_{21})
	\left\{
	Y_2-\mubar_2^v(\bUvec)
	\right\}\\
	&-
	O_\Pbb\left(n^{-\frac{1}{2}}\right)
	\frac{1}{\sqrt n}\sum_{i=1}^n
	\hat\beta_{21}\left\{\omega_2(\bHcheck_2,A_2,\bThetahat_2)Y_2-
	\mubar_{2\omega_2}(\bUvec)\right\}\\
	&-
	O_\Pbb\left(n^{-\frac{1}{2}}\right)
	\frac{1}{\sqrt n}\sum_{i=1}^n
	\Qopt_{2-}(\bH_2; \bthetahat_2)
	\left\{
	\omega_2(\bHcheck_2,A_2,\bThetabar_2)-\mubar_{\omega_2}(\bUvec)
	\right\}\\
	&+
	O_\Pbb\left(c_{n^-_K}\right)
	.
	\end{split}
	\end{align}
	Using result \eqref{eq: nu IF} we know 
	$\frac{1}{\sqrt n}\sum_{i=1}^n\nu_{\subSSLDR}(\bUvec_i;\bThetahat)=O_\Pbb\left(1\right)$, therefore the second, third and fourth terms in \eqref{eq: err IF equiv} are $o_\Pbb(1)$. Using \eqref{eq: nu IF} again for the first term in \eqref{eq: err IF equiv} we get our required result:
	\begin{align*}
	\sqrt n\Pbb_N\left[\mathcal{E}_{\bThetabar}\right]
	&=
	\frac{1}{\sqrt n}\sum_{i=1}^n
	\nu\subSSLDR(\bL_i;\bThetabar)
	\\
	&+
	\frac{1}{\sqrt n}\sum_{i=1}^n
	\left(
	\frac{\partial}{\partial \btheta}\Ebb\left[\nu\subSSLDR(\bL_i;\bThetabar)\right]\right)\trans
	\bpsi^\theta(\bL_i)\\
	&+
	\frac{1}{\sqrt n}\sum_{i=1}^n
	\left(
	\frac{\partial}{\partial \bxi}\Ebb\left[\nu\subSSLDR(\bL_i;\bThetabar)\right]\right)\trans
	\bpsi^\xi(\bL_i)
	+
	o_\Pbb(1).
	\end{align*}
	\end{proof}

    \begin{proof}[Proof of Proposition \ref{cor_dr_V}]\\
    Recall the definition of $\Vsc\subSUPDR(\bL;\bThetabar)$ in \eqref{eq: lab value fun}, using \eqref{V function impute_constraint} we have $\mathbb{E}\left[\Vsc\subSSLDR(\bUvec;\bThetabar,\mubar)\right]
	=
	\mathbb{E}\left[\Vsc\subSUPDR(\bL;\bThetabar)\right],$ therefore 
	\[
	\text{Bias}\left\{\Vbar,\mathcal{V}\subSUPDR\left(\bL;\bThetabar\right)\right\}
	=
	\text{Bias}\left\{\Vbar,\Vsc\subSSLDR(\bUvec;\bThetabar,\mubar)\right\}.
	\]
    Therefore, by Lemma \ref{lemm_bias_term} we have
    \begin{align*}
&\text{Bias}\left\{\Vbar,\Vsc\subSSLDR(\bUvec;\bThetabar,\mubar)\right\}\\
\le&
\sqrt{\sup_{\bHcheck_1}|\{1-\pi_1(\bHcheck_1; \bxibar_1)\}^{-1}|}
\sqrt{\|\pi_1(\bHcheck_1; \bxibar_1)-\pi_1(\bHcheck_1)\|_{L_2(\Pbb)}}
\sqrt{\|\Qopt_1(\bHcheck_1;  \bthetabar_1)-\Qopt_1(\bHcheck_1)\|_{L_2(\Pbb)}}\\
+&
\sqrt{\sup_{\bHcheck_2}
\left|
\left\{
\frac{A_1}{\pi_1(\bHcheck_1;\bxibar_1)}
+
\frac{1-A_1}{1-\pi_1(\bHcheck_1;\bxibar_1)}
\right\}\{1-\pi_1(\bHcheck_1; \bxibar_1)\}^{-1}\{1-\pi_2(\bHcheck_2; \bxibar_2)\}^{-1}\right|}\\
\times&
\sqrt{\|\pi_2(\bHcheck_2; \bxibar_2)-\pi_2(\bHcheck_2)\|_{L_2(\Pbb)}}
\sqrt{\|\Qopt_2(\bHcheck_2;  \bthetabar_2)-\Qopt_2(\bHcheck_2)\|_{L_2(\Pbb)}}.
\end{align*}
    Next using Theorem \ref{thrm_ssV_fun} 
    \begin{align}\label{DR dist}
	\sqrt n
	\left\{
	\Vhat\subSSLDR-\Vbar
	\right\}
	+\sqrt n
	\text{Bias}\left\{\Vbar,\Vsc\subSSLDR(\bUvec;\bThetabar,\mubar)\right\}
	\stackrel{d}{\longrightarrow}N\left(0,\sigma\subSSLDR^2\right),
	\end{align}
	if either \eqref{linear_Qs} or \eqref{logit_Ws} are correct then $\text{Bias}\left\{\Vbar,\Vsc\subSSLDR(\bUvec;\bThetabar,\mubar)\right\}=o_\Pbb(1)$,
	multiplying \eqref{DR dist} by $n^{-\frac{1}{2}}$ we have 
	\[
	\Vhat\subSSLDR-\Vbar
	\stackrel{\Pbb}{\longrightarrow}0,
	\]
	which is the required result for Proposition \ref{cor_dr_V} (a).\\
	
	Next, if $\sqrt{\|\pi_t(\bHcheck_t; \bxihat_t)-\pi_t(\bHcheck_t)\|_{L_2(\Pbb)}}
\sqrt{\|\Qopt_t(\bHcheck_t;  \bthetahat_t)-\Qopt_t(\bHcheck_t)\|_{L_2(\Pbb)}}=O_\Pbb\left(n^{-1}\right)$ for $t=1,2$ then $\text{Bias}\left\{\Vbar,\Vsc\subSSLDR(\bUvec;\bThetabar,\mubar)\right\}=O_\Pbb\left(n^{-1}\right)$ and from \eqref{DR dist} we get
	\[
	\sqrt n
	\left\{
	\Vhat\subSSLDR-\Vbar
	\right\}
	\stackrel{d}{\longrightarrow}N\left(0,\sigma\subSSLDR^2\right),
	\]
	 which is the required result for Proposition \ref{cor_dr_V} (b).
    \end{proof}

Before proving Proposition \ref{lemma: v funcion var}, we introduce a useful definition and state the necessary assumption to prove the result. Let $\psi\subSUP^\xi(\bL;\bxibar)$ and $\psi\subSSL^\xi(\bL;\bxibar)$ be the supervised and SSL influence functions respectively for $\bxi$, then we define
\begin{align*}
\mathcal{E}^v(\bUvec)
=&
\Vsc\subSSLDR(\bUvec;\bThetabar,\mubar)
-
\mathbb{E}_\mathbb{S}\left[\Vsc\subSUPDR(\bL;\bThetabar)\right]
+\mathcal{E}^\theta(\bUvec)
\trans
	\frac{\partial}{\partial \btheta}\int\Vsc\subSUPDR(\bL;\bThetabar)d\Pbb_{\bL}\bigg|_{\bTheta=\bThetabar}\\
	+&
	\mathcal{E}^\xi(\bUvec)
\trans
	\frac{\partial}{\partial \bxi}\int\Vsc\subSUPDR(\bL;\bThetabar)d\Pbb_{\bL}\bigg|_{\bTheta=\bThetabar},\\
\mathcal{E}^\xi(\bUvec)
=&
\psi\subSUP^\xi(\bL;\bxibar)-\psi\subSSL^\xi(\bL;\bxibar).
\end{align*} 
We need to ensure that the imputation models $\mubar^v_2(\bUvec),$ $\mubar^v_{\omega_2}(\bUvec),$ $\mubar^v_{t\omega_2}(\bUvec),$  $t=2,3$ used in the SSL value function estimator $V\subSSLDR$ are unbiased when multiplied by several functions. For example, we need additional constraints of the type:
\begin{align*}
\Ebb\left[\omega_1(\bHcheck_1,A_1;\bThetabar_1)\Qopt_{2-}(\bH_2;\bthetabar_1)\{Y_2-\mubar_2(\bUvec)\}\right]
=&
\bzero,\\
\Ebb\left[\omega_1(\bHcheck_1,A_1;\bThetabar_1)^2\{Y_2-\mubar_2(\bUvec)\}\right]
=&
\bzero,\\
\Ebb\left[\omega_1(\bHcheck_1,A_1;\bThetabar_1)^2\{Y_2-\mubar_2(\bUvec)\}\right]
=&
\bzero,
\end{align*}
so the imputation models are unbiased in expectation when multiplied by every term and cross-product of terms in
	$\psi^v\subSUPDR(\bL;\bThetabar)$, 
	$\mathcal{E}^v(\bUvec)$. These constraints can be summarized in the following Assumption.
\begin{assumption}\label{assumption: value additional constraints}
Imputation models $\mubar^v_2(\bUvec),$ $\mubar^v_{\omega_2}(\bUvec),$ $\mubar^v_{t\omega_2}(\bUvec),$  $t=2,3$ satisfy
\[
	\Ebb\left[\left\{\mathcal{E}^v(\bUvec)
	-
	\psi^v\subSUPDR(\bL;\bThetabar)\right\}
	\mathcal{E}^v(\bUvec)\right]=0.
	\]
\end{assumption}
\begin{proof}[Proof of Proposition \ref{lemma: v funcion var}]
From Theorem \ref{thrm_supV_fun} in Appendix \ref{app_DR_Vfun} we have that the influence function for the fully-supervised value function estimator \eqref{eq: lab value fun}  is:
	\begin{align*}
	\psi^{v}\subSUPDR(\bL;\bThetabar)
	=&
	\Vsc\subSUPDR(\bL;\bThetabar)-\mathbb{E}_\mathbb{S}\left[\Vsc\subSUPDR(\bL;\bThetabar)\right]
	+
	\bpsi\subSUP^\theta(\bL)\trans
	\frac{\partial}{\partial \btheta}\int\Vsc\subSUPDR(\bL;\bThetabar)d\Pbb_{\bL}\bigg|_{\bTheta=\bThetabar}\\
	+&
	\bpsi\subSUP^\xi(\bL)\trans
	\frac{\partial}{\partial \bxi}\int\Vsc\subSUPDR(\bL;\bThetabar)d\Pbb_{\bL}\bigg|_{\bTheta=\bThetabar}.
	\end{align*}
	Next, as we estimate $\bxi$ with a semi-supervised approach such that $\bpsi^\xi\subSSL(\bL;\bxibar)=\bpsi^\xi\subSUP(\bL;\bxibar)-\mathcal{E}^\xi(\bUvec)$, simple algebra can be used to show that
	\[
	\psi^v\subSSLDR(\bL;\bThetabar)=\psi^v\subSUPDR(\bL;\bThetabar)-\mathcal{E}^v(\bUvec).
	\]
	Using the above we can write
    \begin{align*}
    \sigma\subSSLDR^2=
	\Ebb\left[\psi^v\subSSLDR(\bL;\bThetabar)^2\right]
	=&
	\Ebb\left[\left\{\psi^v\subSUPDR(\bL;\bThetabar)-\mathcal{E}^v(\bUvec)\right\}^2\right]\\
	=&
	\Ebb\left[\psi^v\subSUPDR(\bL;\bThetabar)^2\right]
	+
	\Ebb\left[\mathcal{E}^v(\bUvec)^2\right]\\
	-&
	2\Ebb\left[\psi^v\subSUPDR(\bL;\bThetabar)\mathcal{E}^v(\bUvec)\right].
	\end{align*}
	By Assumption \ref{assumption: value additional constraints}, we have
	$
	\Ebb\left[\left\{\mathcal{E}^v(\bUvec)
	-
	\psi^v\subSUPDR(\bL;\bThetabar)\right\}
	\mathcal{E}^v(\bUvec)\right]=0,
	$
	hence
	\[
	\sigma\subSSLDR^2=
    \sigma\subSUPDR^2
    -
    \text{Var}\left[\mathcal{E}^v(\bUvec)\right].
    \]
\end{proof}

\subsubsection{Variance Estimation for \texorpdfstring{$\Vhat\subSUPDR$}{Lg}}\label{variance estimation of Vss}
As discussed in Remark \ref{remark: se for V}, to estimate standard errors for $V\subSSLDR(\bUvec;\bThetabar)$, we will approximate the derivatives of the expectation terms $\frac{\partial}{\partial \bTheta}\int\Vsc\subSUPDR(\bL;\bThetabar)d\Pbb_{\bL}$ using kernel smoothing to replace the indicator functions. In particular, let $\mathbb{K}_h(x)=\frac{1}{h}\sigma(x/h)$, with $\sigma$ defined as in \eqref{logit_Ws}, we approximate $d_t(\bH_t,\btheta_t)=I(\bH_{t1}\trans\bgamma_t>0)$ with $\mathbb{K}_h(\bH_{t1}\trans\bgamma_t)$ $t=1,2$, and define the smoothed propensity score weights as
\begin{align*}
\tilde\omega_1(\bHcheck_1,A_1,\bTheta)
&\equiv
\frac{A_1\mathbb{K}_h(\bH_{11}\trans\bgamma_1)}{\pi_1(\bHcheck_1;\bxi_1)}
+
\frac{\left\{1-A_1\right\}\left\{1-\mathbb{K}_h(\bH_{11}\trans\bgamma_1)\right\}}{1-\pi_1(\bHcheck_1;\bxi_1)}, \quad \mbox{and}\\ \tilde\omega_2(\bHcheck_2,A_2,\bTheta)
&\equiv
\tilde\omega_1(\bHcheck_1,A_1,\bTheta)\left[\frac{A_2\mathbb{K}_h(\bH_{21}\trans\bgamma_2)}{\pi_2(\bHcheck_2;\bxi_2)}+
\frac{\left\{1-A_2\right\}\left\{1-\mathbb{K}_h(\bH_{21}\trans\bgamma_2)\right\}}{1-\pi_2(\bHcheck_2;\bxi_2)}\right].
\end{align*}
For simplicity we'll set $h=1$, the derivatives are as follows:
\begin{align*}
\begin{split}
\frac{\partial}{\partial\btheta}\Vsc\subSUPDR(\bL; \bTheta)=
\frac{\partial}{\partial\btheta}\Qopt_1(\bH_1;\btheta_1)
+&\left\{\frac{\partial}{\partial\btheta}\tilde\omega_1(\bHcheck_1,A_1,\bTheta)\right\}
\left[
Y_2-\left\{\Qopt_1(\bH_1, \btheta_1)- \Qopt_2(\bHcheck_2;\btheta_2)
\right\}\right]\\
+&\tilde\omega_1(\bHcheck_1,A_1,\bTheta)
\left[
-\frac{\partial}{\partial\btheta}\Qopt_1(\bH_1, \btheta_1)+ \frac{\partial}{\partial\btheta}\Qopt_2(\bHcheck_2;\btheta_2)
\right]\\
+&\left\{\frac{\partial}{\partial\btheta}\tilde\omega_2(\bHcheck_2,A_2,\bTheta)\right\}\left[
Y_3-\Qopt_2(\bHcheck_2; \btheta_2)\right]\\
-&\tilde\omega_2(\bHcheck_2,A_2,\bTheta)\frac{\partial}{\partial\btheta}\Qopt_2(\bHcheck_2; \btheta_2),
\end{split}
\end{align*}
where 
\begin{align*}
\begin{split}
\frac{\partial}{\partial\btheta}\Qopt_1(\bH_1;\btheta_1)&=[\bH_{10}\trans,\bH_{11}\trans I\left(\bH_{11}\trans\bgamma_1>0\right),\bzero\trans]\trans,\\ 
\frac{\partial}{\partial\btheta}\Qopt_2(\bHcheck_2; \btheta_2)&=[\bzero\trans,\bH_{20}\trans,\bH_{21}\trans I\left(\bH_{21}\trans\bgamma_2>0\right)]\trans,
\\ 
\frac{\partial}{\partial\btheta}\tilde\omega_1(\bHcheck_1,A_1,\bTheta)
&=
\left[\bzero\trans,\bH_{11}\trans\mathbb{K}_h(\bH_{11}\trans\bgamma_1)\{1-\mathbb{K}_h(\bH_{11}\trans\bgamma_1)\}\left\{\frac{A_1}{\pi_1(\bHcheck_1;\bxi_1)}-\frac{1-A_1}{1-\pi_1(\bHcheck_1;\bxi_1)}\right\},\bzero\trans\right]\trans
\\ 
\frac{\partial}{\partial\btheta}\tilde\omega_2(\bHcheck_2,A_2,\bTheta)
&=
\frac{\partial}{\partial\btheta}\tilde\omega_1(\bHcheck_1,A_1,\bTheta)\left\{\frac{A_2d_2(\bH_2;\btheta_2)}{\pi_2(\bHcheck_2;\bxi_2)}+
\frac{\left\{1-A_2\right\}\left\{1-d_2(\bH_2;\btheta_2)\right\}}{1-\pi_2(\bHcheck_2;\bxi_2)}\right\}\\
&+
\tilde\omega_1(\bHcheck_1,A_1,\bTheta)
\left[\bzero\trans,
\bH_{21}\trans\mathbb{K}_h(\bH_{21}\trans\bgamma_2)(1-\mathbb{K}_h(\bH_{21}\trans\bgamma_2))\left\{\frac{A_2}{\pi_2(\bHcheck_2;\bxi_2)}-\frac{1-A_2}{1-\pi_2(\bHcheck_2;\bxi_2)}\right\}\right]\trans.\\
\end{split}
\end{align*}
Next we have
\begin{align*}
\begin{split}
\frac{\partial}{\partial\bxi}\Vsc\subSUPDR(\bL; \bTheta)=
&\left\{\frac{\partial}{\partial\bxi}\tilde\omega_1(\bHcheck_1,A_1,\bTheta)\right\}
\left[
Y_2-\left\{\Qopt_1(\bH_1, \btheta_1)- \Qopt_2(\bHcheck_2;\btheta_2)
\right\}\right]\\
+&\left\{\frac{\partial}{\partial\bxi}\tilde\omega_2(\bHcheck_2,A_2,\bTheta)\right\}\left[
Y_3-\Qopt_2(\bHcheck_2; \btheta_2)\right],
\end{split}
\end{align*}
where 
\begin{align*}
\begin{split}
\frac{\partial}{\partial\bxi}\tilde\omega_1(\bHcheck_1,A_1,\bTheta)&=\left[\varpi_1(\bHcheck_1;\bxi_1)\trans,\bzero\trans\right]\trans,
\\
\frac{\partial}{\partial\bxi}\tilde\omega_2(\bHcheck_2,A_2,\bTheta)&=\left[\tilde\omega_1(\bHcheck_1,A_1,\bTheta)\varpi_2(\bHcheck_2;\bxi_2)\trans,\bzero\trans\right]\trans,\\
\varpi_t(\bHcheck_t;\bxi_t)&\equiv\bHcheck_{t1}
\left\{
-d_t(\bHcheck_t,\btheta_t)A_t\frac{1-\pi_t(\bHcheck_t;\bxi_t)}{\pi_t(\bHcheck_t;\bxi_t)}
+
\{1-d_t(\bH_t,\btheta_t)\}\{1-A_t\}\frac{\pi_t(\bHcheck_t;\bxi_t)}{1-\pi_t(\bHcheck_t;\bxi_t)}
\right\}.
\end{split}
\end{align*}

\section{Technical Lemmas}\label{sec:proof_technical_lemmas}

We start with a simple Lemma that will save us some algebra:

\begin{lemma}\label{lemma: Op product}
For a fixed $\ell$, let $\bX\in\mathbb{R}^\ell$ be a random bounded vector and functions $g_1(\bX),\:g_2(\bX)$ be measurable functions of $\bX$. Let $\mathbb{S}_n=\{\bX\}_{i=1}^n$ be an $i.i.d.$ sample, and $\hat g_1(\cdot)$, $\hat g_2(\cdot)$ be the estimators for functions $g_1,\:g_2\in\mathbb{R}$ respectively with $\sup_{\bX}|g_1(\bX)|$, $\sup_{\bX}|g_2(\bX)|$, $\sup_{\bX}|\hat g_1(\bX)|$, $\sup_{\bX}|\hat g_2(\bX)|<\kappa$ for fixed $\kappa\in\mathbb{R}$. If $\Pbb_n\{ \hat g_k-g_k\}=O_\Pbb\left(n^{-\frac{1}{2}}\right)$, for $k=1,2,$ then $\Pbb_n \{\hat g_1\hat g_2-g_1g_2\}=O_\Pbb\left(n^{-\frac{1}{2}}\right)$.
\end{lemma}
\begin{proof}[Proof of Lemma \ref{lemma: Op product}]
By definition, $\Pbb_n \{\hat g_1\hat g_2-g_1g_2\}=O_\Pbb\left(n^{-\frac{1}{2}}\right)$ if and only if for a given any $\epsilon>0$, $\exists M_\epsilon>0$ such that\\  $\Pbb\left(\left|\Pbb_n \{\hat g_1\hat g_2-g_1g_2\}\right|>M_\epsilon n^{-\frac{1}{2}}\right)\le\epsilon$ $\forall n$.
Let $M_\epsilon>0$,
\begin{align*}
&\Pbb\left(|\Pbb_n \{g_{1}g_{2}-g_1g_2\}|>M_\epsilon n^{-\frac{1}{2}}\right)\\
=&
\Pbb\left(|\Pbb_n \{\hat g_1\hat g_2-\hat g_1g_2+\hat g_1g_2-g_1g_2\}|>M_\epsilon n^{-\frac{1}{2}}\right)\\
&\le
\Pbb\left(|\Pbb_n\{ \hat g_1(\hat g_2-g_2)\}|+|\Pbb_n\{g_2( \hat g_1-g_1)\}|>M_\epsilon n^{-\frac{1}{2}}\right)\\
&\le
\Pbb\left(\sup_{\bX}|\hat g_1(\bX)||\Pbb_n\{\hat g_2-g_2\}|+\sup_{\bX}|g_2(\bX)||\Pbb_n\{\hat g_1-g_1\}|>M_\epsilon n^{-\frac{1}{2}}\right)\\
\end{align*}
which follows from bounded functions, the union bound, now since $\Pbb_n\{\hat g_k(\bX)-g_k(\bX)\}=O_\Pbb\left(n^{-\frac{1}{2}}\right)$, $k=1,2$, there exists $M_\epsilon>0$ such that
\[
\Pbb\left(|\Pbb_n\{ \hat g_2-g_2\}|>M_\epsilon n^{-\frac{1}{2}}\frac{1}{\kappa}\right)
+\Pbb\left(|\Pbb_n\{\hat g_1-g_1\}|>M_\epsilon n^{-\frac{1}{2}}\frac{1}{\kappa}\right)
\le\frac{\epsilon}{2}+\frac{\epsilon}{2}=\epsilon.
\]
 
\end{proof}

\begin{lemma}\label{lemm_chakrabortty}(Lemma (A.1) (a) in \cite{chakrabortty})\\
	Let $\bX\in\mathbb{R}^\ell$ be any random vector and $g(\bX)\in\mathbb{R}^\ell$ be any measurable function of $\bX$, with $\ell$ and $d$ fixed. Let $\mathbb{S}_n=\{\bX\}_{i=1}^n,\mathbb{S}_N=\{\bX\}_{j=1}^N$ be two random samples of $n$ and $N$ $i.i.d$ observations of $\bX$ respectively, such that $\mathbb{S}_n\indep\mathbb{S}_N$. Let $\hat g_n(\cdot)$ be any estimator of $g(\cdot)$ estimated with $\mathbb{S}_n$ such that the random sequence: $\hat T_n=\sup_{x\in\mathcal{X}}\|\hat g_n(\cdot)\|=O_\Pbb(1)$, where   $\bX\in\mathcal{X}\subseteq\mathbb{R}^\ell$. Further define the following random sequences: $\hat{\pmb G}_{n,N}\equiv\frac{1}{N}\sum_{j=1}^N\hat g_n(\bX_j)$, and $\bar{\pmb G}_n\equiv\Ebb_{\mathbb{S}_N}\left[\hat G_{n,N}\right]=\mathbb{E}_{\bX}\left[\hat g_n(\bX)\right]$, where $\Ebb_{\bX}$ is the expectation with respect to $\bX\in\mathbb{S}_N$. We assume all expectations involved are finite almost surely (a.s.) $\mathbb{S}_n$ $\forall n$. Then $G_{n,N}-\bar G_n=O_\Pbb\left(N^{-\frac{1}{2}}\right)$.
\end{lemma}
\begin{proof}[Proof of lemma \ref{lemm_chakrabortty}]

The following proof follows similar arguments to \cite{chakrabortty}. Let $\mathcal{G}_{n,N}$, $\bar{\mathcal{G}}_n$ be the $j^{th}$ element of $\hat{\pmb G}_{n,N}$ and $\bar{\pmb G}_n$ respectively, with $j\in\{1,\dots,\ell\}$. We show that $\mathcal{G}_{n,N}-\bar{\mathcal{G}}_n=O_\Pbb\left(N^{-\frac{1}{2}}\right)$, which implies Lemma \ref{lemm_chakrabortty} for any $\ell$ dimensional $\hat{\pmb G}_{n,N}$, $\bar{\pmb G}_n$. Denote by $\Pbb_{\mathbb{S}_n}$, $\Pbb_{\mathbb{S}_n,\mathbb{S}_N}$ denote the joint probability distributions of samples $\mathbb{S}_n$ and $\mathbb{S}_n,\mathbb{S}_N$ respectively. Further let $\mathbb{E}_{\mathbb{S}_n}[\cdot]$ denote the expectation with respect to $\mathbb{S}_n$. Since $\mathbb{S}_n\indep\mathbb{S}_N$ using Hoeffding's inequality

\[
\Pbb_{\mathbb{S}_N}
\left(\left|
\hat{\mathcal{G}}_{n,N}-\hat{\mathcal{G}}_n
\right|>N^{-\frac{1}{2}}t\bigg|\mathbb{S}_n
\right)
\le
2\exp\left(-\frac{2N^2t^2}{4N^2\hat T^2_n}\right)\text{ a.s. }\Pbb_{\mathbb{S}_n}.
\]

Also, as $\mathbb{S}_n\indep\mathbb{S}_N$ we have
\[
\Pbb_{\mathbb{S}_n,\mathbb{S}_N}\left[\left|
\hat{\mathcal{G}}_{n,N}-\hat{\mathcal{G}}_n
\right|>N^{-\frac{1}{2}}t
\right]
=
\Ebb_{\mathbb{S}_n}\left[\Pbb_{\mathbb{S}_N}\left\{
\left|
\hat{\mathcal{G}}_{n,N}-\hat{\mathcal{G}}_n
\right|>N^{-\frac{1}{2}}t\bigg|\mathbb{S}_n
\right\}\right].
\]

Next, we have that $\hat T_n=\sup_{x\in\mathcal{X}}\|\hat g_n(\cdot)\|=O_\Pbb(1)$ and is non-negative, thus $\forall\epsilon>0$ $\exists\:\delta(\epsilon)>0$ such that\\ $\Pbb_{\mathbb{S}_n}\left(\hat T_n>\delta(\epsilon)\right)<\epsilon/4$, using the above we have that $\forall$ $n,N$:		
\begin{align*}
&\Pbb_{\mathbb{S}_n,\mathbb{S}_N}
\left(\left|
\hat{\mathcal{G}}_{n,N}-\hat{\mathcal{G}}_n
\right|>N^{-\frac{1}{2}}t
\right)
\le
\mathbb{E}_{\mathbb{S}_n}	
\left[
2\exp\left(-\frac{2N^2t^2}{4N^2\hat T^2_n}\right)
\right]\\
=&
\mathbb{E}_{\mathbb{S}_n}	
\left[
2\exp\left(-\frac{t^2}{2\hat T^2_n}\right)
\right]
=
\mathbb{E}_{\mathbb{S}_n}	
\left[
2\exp\left(-\frac{t^2}{2\hat T^2_n}\right)
\left(I\{\hat T_n>\delta(\epsilon)\}+I\{\hat T_n\le\delta(\epsilon)\}\right)
\right]\\
&\le
2\Pbb_{\mathbb{S}_n}\left(\hat T_n<\delta(\epsilon)\right)
+	
2\exp\left(-\frac{t^2}{2\delta^2(\epsilon)}\right)
\Pbb_{\mathbb{S}_n}\left(\hat T_n>\delta(\epsilon)\right)
\le
2\exp\left(-\frac{t^2}{2\delta^2(\epsilon)}\right)+\frac{\epsilon}{2}
\le
\frac{2\epsilon}{2}=\epsilon,
\end{align*}
where the last step follows from choosing $t$ large enough such that 	$\exp\left(-\frac{t^2}{2\delta^2(\epsilon)}\right)\le\epsilon/4.$
\end{proof}
For Assumption \ref{assumption: C bounds} and Lemma \ref{lemm_deltas} we first define some notation and set up the problem. Let $\bX=(\bX_1,\bX_2)\in\mathbb{R}^{\ell_1+\ell_2}$ be any random vector and $g(\bX_1)\in\mathbb{R}$ be any measurable function of $\bX_1\in\mathbb{R}^{\ell_1}$ with $\ell_1,\ell_2$ fixed. Suppose we're interested in estimating $m(\bX_2)=\mathbb{E}[g(\bX_1)|\bX_2]$. Let $\mathbb{S}_n=\{\bX\}_{i=1}^n$ be a random sample of $n$ i.i.d. observations of $\bX$, and $\mathbb{S}_{k=1}^{K}$ denote a random partition of $\mathbb{S}_n$ into $K$ disjoint subsets of size $n_{K}=\frac{n}{K}$ with index sets $\{\mathcal{I}_k\}_{k=1}^{K}$. We will use cross-validation to estimate $\mhat (\bX_2)$, that is, we use subset $\mathcal{I}_k$ to train estimator $\mhat_k$ and we estimate $m(\bX_2)$ with: $\mhat(\bX_{2})=K^{-1}\sum_{k=1}^{K}\sum_{i\in\mathcal{I}_k}\mhat_{k}(\bX_{2})$, $K\ge2$. Denote by $\hat C_{n,N}\in\mathbb{R}$ an estimator which depends on both samples $\mathbb{S}_n,\mathbb{S}_N$. Additionally, let function $\hat \pi_n		(\cdot):\mathbb{R}^{\ell_2}\rightarrow(0,1)$ be a random function with limit $\pi(\cdot)$, $\hat l_n(\bX_2):\mathbb{R}^{\ell_2}\rightarrow\{0,1\}$, be a random function with limit $l(\bX_2)$, and finally function $f:\mathbb{R}^{\ell_2}\rightarrow\mathbb{R}^d$, $d\le\ell_2$ be any deterministic function of $\bX_2$.

\begin{assumption}\label{assumption: C bounds}
Let $\mathcal{X}\subset\mathbb{R}^p$ for an arbitrary $p\in\mathbb{N}$ i) function $w:\mathcal{X}\mapsto\mathbb{R}$ and estimator $\hat \pi_n$ are such that $\sup_{\bX_2}\left|\hat \pi_n(\bX_2)^{-1}-\pi(\bX_2)^{-1}\right|=O_\Pbb\left(n^{-\frac{1}{2}}\right)$, ii) function $l:\mathcal{X}\mapsto\{0,1\}$ and estimator $\hat l_n$ are such that $\sup_{\bX_2}\left|\hat l_n(\bX_2)-l(\bX_2)\right|=O_\Pbb\left(n^{-\frac{1}{2}}\right)$, and iii) function $f:\mathbb{R}^{\ell_2}\rightarrow\mathbb{R}^d$, $d\le\ell_2$ is such that $\sup_{\bX_2}\| f(\bX_2)\|<\infty$.
			
\end{assumption}
\begin{lemma}\label{lemm_deltas}
    Define $\hat {\pmb G}^n_k(\bX_2)=\hat C_{n,N}\frac{\hat l_n(\bX_2)}{\hat \pi_n(\bX_2)}f(\bX_2)\hat\Delta_{k}(\bX_2)-\mathbb{E}\left[\frac{l(\bX_2)}{\pi(\bX_2)}f(\pmb x_2)\hat\Delta_{k}(\bX_2)\right]$ for $\hat\Delta_{k}(\bX_2)=\mhat_k(\bX_2)-m(\bX_2)$, and $\hat C_{n,N}\in\mathbb{R}$ which satisfies $\hat C=1+O_\Pbb\left(n^{-\frac{1}{2}}\right)$. Under Assumptions \ref{assumption: V imputation} and \ref{assumption: C bounds},  there is $c_{n_K^-}=o(1)$ such that  $\mathbb{G}_{n,K}=n^{-\frac{1}{2}}\sum_{k=1}^{K}\sum_{i\in\mathcal{I}_k}\hat {\pmb G}^n_k(\bX_2)=O_\Pbb\left(c_{n_K^-}\right)$, 
\end{lemma}

\begin{proof}[Proof of Lemma \ref{lemm_deltas}]
First we define 
\[
\mathcal{G}^{(n)}_k=n^{-\frac{1}{2}}\sum_{i\in\mathcal{I}_k}
\frac{l(\bX_{2i})}{\pi(\bX_{2i})} f(\bX_{2i})\hat\Delta_{k}(\bX_{2i})-\Ebb\left[\frac{l(\bX_{2i})}{\pi(\bX_{2i})} f(\bX_{2i})\hat\Delta_{k}(\bX_{2i})\right],
\]
for any sample subset $\mathbb{S}_K\subseteq\mathcal{L}$, let $\Pbb_{\mathbb{S}_K}$ denote the joint probability distribution of $\mathbb{S}_K$, and let $\mathbb{E}_{\mathbb{S}_K}[\cdot]$ denote expectation with respect to $\Pbb_{\mathbb{S}_K}$, and $\mathbb{G}_{n,K}=K^{-\frac{1}{2}}\sum_{k=1}^K\mathcal{G}_k^{(n)}$, Next by Assumption \ref{assumption: V imputation} we have $\dhat_k\equiv\sup_{\bX_2}\hat\Delta(\bX_2)=o_\Pbb(1)$. Finally let $B_1=\sup_{\bX_2}\|f(\bX_2)\|_2<\infty$, $B_2<\infty$ be the upperbound to $\sup_{\bX_2}|\pi(\bX_2)^{-1}|,\sup_{\bX_2}|\hat l_n(\bX_2)|\sup_{\bX_2}|\frac{\hat l_n(\bX_2)}{\hat \pi_n(\bX_2)}|$.
 
First note that 

\begin{align*}
&\left\|\mathbb{G}_{n,K}\right\|_2\\
=&
\left\|
n^{-\frac{1}{2}}\sum_{k=1}^{K}\sum_{i\in\mathcal{I}_k}
\hat C_{n,N}\frac{\hat l_n(\bX_{2i})}{\hat \pi_n(\bX_{2i})}f(\bX_{2i})\hat\Delta_{k}(\bX_{2i})-\mathbb{E}\left[\frac{l(\bX_{2i})}{\pi(\bX_{2i})}f(\bX_{2i})\hat\Delta_{k}(\bX_{2i})\right]
\right\|_2\\
\le&
\left\|
\left(\hat C_{n,N}-1\right)
n^{-\frac{1}{2}}\sum_{k=1}^{K}\sum_{i\in\mathcal{I}_k}
f(\bX_{2i})\hat\Delta_{k}(\bX_{2i})\frac{\hat l_n(\bX_{2i})}{\hat \pi_n(\bX_{2i})}
\right\|_2\\
+&
\left\|
n^{-\frac{1}{2}}\sum_{k=1}^{K}\sum_{i\in\mathcal{I}_k}
f(\bX_{2i})\hat\Delta_{k}(\bX_{2i})\hat l_n(\bX_{2i})\left(\frac{1}{\hat \pi_n(\bX_{2i})}
-\frac{1}{\pi(\bX_{2i})}\right)
\right\|_2\\
+&\left\|
n^{-\frac{1}{2}}\sum_{k=1}^{K}\sum_{i\in\mathcal{I}_k}
f(\bX_{2i})\hat\Delta_{k}(\bX_{2i}) \frac{1}{\pi(\bX_{2i})}\left(\hat l_n(\bX_{2i})-l(\bX_{2i})\right)
\right\|_2\\
+&
\left\|
n^{-\frac{1}{2}}\sum_{k=1}^{K}\sum_{i\in\mathcal{I}_k}
\frac{l(\bX_{2i})}{\pi(\bX_{2i})}f(\bX_{2i})\hat\Delta_{k}(\bX_{2i})
-
\mathbb{E}\left[\frac{l(\bX_{2i})}{\pi(\bX_{2i})}f(\bX_{2i})\hat\Delta_{k}(\bX_{2i})\right]
\right\|_2,\\
\end{align*}
which follows from the triangle inequality, next as $f(\cdot),\hat \pi_n(\cdot)^{-1},\pi(\cdot)^{-1},\hat l_n(\cdot)$ are bounded $\forall \bX_2\in\mathcal{X}$, and using uniform bounds of $O_\Pbb\left(n^{-\frac{1}{2}}\right)$ for the difference terms we have  
\begin{align*}
\left\|\mathbb{G}_{n,K}\right\|_2
\le&
O_\Pbb\left(n^{-\frac{1}{2}}\right)
n^{\frac{1}{2}}B_1B_2\left|\sum_{k=1}^{K}
\dhat_k
\right|
+
O_\Pbb\left(n^{-\frac{1}{2}}\right)
n^{\frac{1}{2}}B_1B_2\left|\sum_{k=1}^{K}
\dhat_k  
\right|\\
+&
O_\Pbb\left(n^{-\frac{1}{2}}\right) 
n^{\frac{1}{2}}B_1B_2\left|\sum_{k=1}^{K}
\dhat_k  
\right|
+
\left\|
\frac{1}{K}\sum_{k=1}^{K}\mathcal{G}^{(n)}_k
\right\|_2,\\
&\le
\left\|
n^{-\frac{1}{2}}\frac{1}{K}\sum_{k=1}^{K}\sum_{i\in\mathcal{I}_k}
\frac{l(\bX_{2i})}{\pi(\bX_{2i})}f(\bX_{2i})\hat\Delta_{k}(\bX_{2i})
-
\mathbb{E}\left[\frac{l(\bX_{2i})}{\pi(\bX_{2i})}f(\bX_{2i})\hat\Delta_{k}(\bX_{2i})\right]
\right\|_2+o_\Pbb\left(1\right).
\end{align*}
where the last step follows from $\dhat_k=o_\Pbb(1)$.
Next we want to bound the first term above by $c_{n_K^-}$ in probability, note that $\forall\epsilon$ $\exists M>0$ such that 
\begin{align*}
&\Pbb\left(
\left\|
\sum_{k=1}^{K}\mathcal{G}^{(n)}_k
\right\|_2>Mc_{n_K^-}
\right)
\le
\Pbb\left(
K^{-\frac{1}{2}}
\left\|
\sum_{k=1}^{K}\mathcal{G}^{(n)}_k
\right\|_2>Mc_{n_K^-}
\right)\\
&\le
\sum_{k=1}^{K}\Pbb\left(
\left\|
\mathcal{G}^{(n)}_k
\right\|_2>\frac{Mc_{n_K^-}}{K^\frac{1}{2}}
\right)
\le
\sum_{k=1}^{K}\sum_{j=1}^d
\Pbb\left(
\left|
\mathcal{G}^{(n)}_{k[j]}
\right|>\frac{Mc_{n_K^-}}{(Kd)^\frac{1}{2}}
\right)\\
&\le
\sum_{k=1}^{K}\sum_{j=1}^d
\mathbb{E}_{\mathcal{L}_k^-}\left[
\Pbb_{\mathcal{L}_k}\left(
\left|
\mathcal{G}^{(n)}_{k[j]}
\right|>\frac{Mc_{n_K^-}}{(Kd)^\frac{1}{2}}
\biggr|
\mathcal{L}_k^-
\right)
\right],
\end{align*}
where the first 3 steps follow from applying Boole's inequality and the triangle inequality, the fourth step follows from iterated expectations for the the event $\left\{\left|\mathcal{G}^{(n)}_{k[j]}
\right|>\frac{Mc_{n_K^-}}{(Kd)^\frac{1}{2}}\right\}$.\\

Next, we have $\mathcal{L}_k^-\indep\mathcal{L}_k$, $\forall\:k\in\{1,\dots,K\}$, thus conditional on $\mathcal{L}_k^-$, $n^\frac{1}{2}\mathcal{G}^{(n)}_k$ is a sum of iid centered random vectors $\left\{\frac{l(\bX_{2i})}{\pi(\bX_{2i})}f(\bX_{2i})\hat\Delta_{k}(\bX_{2i})\right\}_{i\in\mathcal{I}_k}$ which are bounded a.s. $\Pbb_\mathcal{L}^-,\:\forall k,n.$ Thus we can apply Hoeffding's inequality to $\mathcal{G}^{(n)}_{k[j]}\forall j$:
\begin{align}\label{hoef_bound}
\Pbb_{\mathcal{L}_k}\left(
\left|
\mathcal{G}^{(n)}_{k[j]}
\right|>\frac{Mc_{n_K^-}}{(Kd)^\frac{1}{2}}
\biggr|
\mathcal{L}_k^-
\right)
\le
2\exp\left\{
-\frac{M^2c_{n_K^-}^2}{2KdB^2\dhat_k^2}
\right\}
\end{align}
a.s. $\Pbb_{\mathcal{L}_k^-}\forall n;$ and for each $k\in\{1,\dots,K\},j\in\{1\dots,d\}.$ Note that $\frac{c_{n_K^-}}{D_k}\ge0$ is stochastically bounded away from zero as $\dhat_k=o_\Pbb(1)$, therefore $\forall k$ and given $\epsilon>0$, $\exists\delta(\epsilon,k)>0$ such that $\Pbb_{\mathcal{L}_k^-}\left(\frac{c_{n_K^-}}{D_k}\le\delta(\epsilon,k)\right)\le\frac{\epsilon}{4Kd}$, let $\delta^*(\epsilon,k)=\min_k\{\delta(\epsilon,k)\}$, we have that\\ $\Pbb_{\mathcal{L}_k^-}\left(\frac{c_{n_K^-}}{D_k}\le\delta^*(\epsilon,k)\right)\le\frac{\epsilon}{4Kd}$.

Therefore using the bound in (\ref{hoef_bound}) and event $\left\{\frac{c_{n_K^-}}{D_k}\le\delta^*(\epsilon,k)\right\}$:
\begin{align*}
&\Pbb\left(
\left\|
\sum_{k=1}^{K}\mathcal{G}^{(n)}_k
\right\|_2>Mc_{n_K^-}
\right)\\
&
\le
\sum_{k=1}^{K}\sum_{j=1}^d
\mathbb{E}_{\mathcal{L}_k^-}\left[
\Pbb_{\mathcal{L}_k}\left(
\left|
\mathcal{G}^{(n)}_{k[j]}
\right|>\frac{Mc_{n_K^-}}{(Kd)^\frac{1}{2}}
\biggr|
\mathcal{L}_k^-
\right)
\right]\\
&\le
\sum_{k=1}^{K}\sum_{j=1}^d
\mathbb{E}_{\mathcal{L}_k^-}\left[
2\exp\left\{
-\frac{M^2c_{n_K^-}^2}{2KdB^2\dhat_k^2}
\right\}
\left(
I\left\{\frac{c_{n_K^-}}{D_k}\le\delta^*(\epsilon,k)\right\}
+
I\left\{\frac{c_{n_K^-}}{D_k}\delta^*(\epsilon,k)\right\}
\right)
\right]\\
&\le
2{K}d
\exp\left\{
-\frac{M^2\delta^*(\epsilon,k)^2}{2KdB^2}
\right\}
\Pbb_{\mathcal{L}_k^-}
\left(\frac{c_{n_K^-}}{D_k}\le\delta^*(\epsilon,k)
\right)
+
2{K}d
\Pbb_{\mathcal{L}_k^-}\left(\frac{c_{n_K^-}}{D_k}>\delta^*(\epsilon,k)\right)\\
&\le
2{K}d
\frac{\epsilon}{4Kd}
+
2{K}d
\exp\left\{
-\frac{M^2\delta^*(\epsilon,k)^2}{2KdB^2}
\right\}
\Pbb_{\mathcal{L}_k^-}
\left(\frac{c_{n_K^-}}{D_k}>\delta^*(\epsilon,k)
\right),
\end{align*}
next note that choosing a large enough $M$ such that
$
\exp\left\{
-\frac{M^2\delta^*(\epsilon,k)^2}{2KdB^2}
\right\}
<	\frac{\epsilon}{4Kd}
$, since $\Pbb_{\mathcal{L}_k^-}
\left(\frac{c_{n_K^-}}{D_k}>\delta^*(\epsilon,k)\le1\right)$ we get $\Pbb\left(
\left\|
\sum_{k=1}^{K}\mathcal{G}^{(n)}_k
\right\|_2>Mc_{n_K^-}
\right)\le\frac{\epsilon}{2}+\frac{\epsilon}{2}=\epsilon$.\\
Finally we have 
\[
\mathbb{G}_{n,K}=O_\Pbb\left(c_{n_K^-}\right)+o_\Pbb(1)=O_\Pbb\left(c_{n_K^-}\right).
\]
\end{proof}
\begin{lemma}\label{lemm_gamma_difs}
	Let $\bgammahat\in\mathbb{R}^d$ be a random variable such that $\sqrt n \left( \bgammahat-\bgammabar\right)=O_\Pbb(1)$, then for any fixed vector $\mathbf{a}\in\mathbb{R}^d$ we have that (a) $\sqrt n\left(\left[\mathbf{a}\trans\bgammahat\right]_+-\left[\mathbf{a}\trans\bgammabar\right]_+\right)=
	\sqrt n \left( \bgammahat-\bgammabar\right)
	I(\mathbf{a}\trans\bgammabar>0)
	+o_\Pbb(1)$, (b) Functions $\dhat_{t}$ $t=1,2,$ defined in Section \ref{section: SS value function} and propensity scores $\pi_1$ in \eqref{logit_Ws} satisfy 
	\begin{align*}
	\sup_{\bH_1,\mathbf{a}_1}
	\bigg|I(\dhat_1=A_1)-I( \dbar_1=A_1)\bigg|=&O_\Pbb\left(n^{-\frac{1}{2}}\right),\\
	\sup_{\bH_2,\mathbf{a}_2}
	\bigg|I(\dhat_1=A_1)I(A_2=\dhat_2)-I( \dbar_1=A_1)I( \dbar_2=A_2)\bigg|=&O_\Pbb\left(n^{-\frac{1}{2}}\right),\\
	\sup_{\bH_1}
	\bigg|\frac{1}
    {\pi_1(\bH_1;\bxihat_1)}
    -
    \frac{1}{ \pi_1(\bH_1;\bxibar_1)}\bigg|=&O_\Pbb\left(n^{-\frac{1}{2}}\right).
	\end{align*}
	(c) For $\bthetahat$, $\bxihat$ estimated via our semi-supervised approach, and limits $\bthetabar$, $\bxibar$ defined in Assumptions \ref{assumption SS linear model} and \ref{assumption: donsker w} respectively 
	\[
	\hat C^{(1)}_{n,N} =
		\frac{(1+\hat\beta_{21})\Pbb_N
		\left\{\omega_1(\bHcheck_1,A_1;\bThetabar_1)\right\}}{(1+\hat\beta_{21})\Pbb_n
		\left\{\omega_1(\bHcheck_1,A_1;\bThetahat_1)\right\}},\:
	\hat C^{(2)}_{n,N} = 
		\frac{\Pbb_N
		\left\{\Qopt_{2-}(\bH_2,A_2;\bthetabar_2)\right\}}{\Pbb_n
		\left\{\Qopt_{2-}(\bH_2,A_2;\bthetahat_2)\right\}},
		\] satisfy $\hat C^{(1)}_{n,N}=1+O_\Pbb(n^{-\frac{1}{2}})$, $\hat C^{(2)}_{n,N}=1+O_\Pbb(n^{-\frac{1}{2}})$.
\end{lemma}
\begin{proof}[Proof  of Lemma \ref{lemm_gamma_difs}]

Define set $\mathcal{A}_q$ for any $q$ dimensional vector $\pmb{\hat\gamma}$ as 
\[
\mathcal{A}_q=\left\{\bgammahat\in\mathbb{R}^q\:\bigg|\:\frac{1}{2}\mathbf{a}\trans\bgammabar<\mathbf{a}\trans\bgammahat<2\mathbf{a}\trans\bgammabar,\:\:\forall \mathbf{a}\in\mathbb{R}^q\right\}.
\]
Now consider $\bgammahat\in\mathcal{A}_q:$\\ 
\begin{itemize}
	\item if sign$(\mathbf{a}\trans\bgammabar)=1$, then $0<\frac{1}{2}\mathbf{a}\trans\bgammabar<\mathbf{a}\trans\bgammahat\implies$sign$(\mathbf{a}\trans\bgammahat)=1$, 
	\item if sign$(\mathbf{a}\trans\bgammabar)=-1$, then $\mathbf{a}\trans\bgammahat<2\mathbf{a}\trans\bgammabar<0\implies$sign$(\mathbf{a}\trans\bgammahat)=-1$.
\end{itemize} 
Assuming $\sqrt n (\bgammahat-\bgammabar)=O_\Pbb(1)$, $\mathcal{A}_q$ exists and in fact it is such that $\Pbb\left(\bgammahat\in\mathcal{A}_q\right)\stackrel{p}{\longrightarrow}1$.\\
(a) Using the above:

\begin{align*}
\sqrt n\left(\left[\mathbf{a}\trans\bgammahat\right]_+-\left[\mathbf{a}\trans\bgammabar\right]_+\right)
=&\sqrt n \left( \bgammahat-\bgammabar\right)
I(\mathbf{a}\trans\bgammabar>0)
I\left(\bgammahat\in\mathcal{A}_q\right)
+
\sqrt n\left(\left[\mathbf{a}\trans\bgammahat\right]_+-\left[\mathbf{a}\trans\bgammabar\right]_+\right)
I\left(\bgammahat\notin\mathcal{A}_q\right)
\\
=&
\sqrt n \left( \bgammahat-\bgammabar\right)
I(\mathbf{a}\trans\bgammabar>0)
+o_\Pbb(1).
\end{align*}
(b) As $A_{ti}\in\{0,1\}$ , $t=1,2,$ we can write
\begin{align*}
I(\dhat_1=A_1)I(\dhat_2=A_2)
=&I\left\{A_1=I(\bH_{11}\trans\bgammahat_1>0)\right\}
I\left\{A_2=I(\bH_{21}\trans\bgammahat_2>0)\right\}\\
=&I\{A_1=I(\bH_{11}\trans\bgammahat_1>0)\}
I\{A_2=I(\bH_{21}\trans\bgammahat_2>0)\}\\
=&
A_1A_2I(\bH_{11}\trans\bgammahat_1>0)
I(\bH_{21}\trans\bgammahat_2>0)\\
+&
(1-A_1)(1-A_2)I(\bH_{11}\trans\bgammahat_1<0)
I(\bH_{21}\trans\bgammahat_2<0)\\
+&
A_1(1-A_2)I(\bH_{11}\trans\bgammahat_1>0)
I(\bH_{21}\trans\bgammahat_2<0)\\
+&
(1-A_1)A_2I(\bH_{11}\trans\bgammahat_1<0)
I(\bH_{21}\trans\bgammahat_2>0),
\end{align*}
therefore 
\begin{align*}
&\bigg|I(\dhat_1=A_1)I(\dhat_2=A_2)
-
I( \dbar_1=A_1)I( \dbar_2=A_2)\bigg|\\
=&
\bigg|
A_1A_2
\left\{I(\bH_{11}\trans\bgammahat_1>0)
I(\bH_{21}\trans\bgammahat_2>0)
-
I(\bH_{11}\trans\bgammabar_1>0)
I(\bH_{21}\trans\bgammabar_2>0)
\right\}\\
+&
(1-A_1)(1-A_2)\left\{
I(\bH_{11}\trans\bgammahat_1<0)
I(\bH_{21}\trans\bgammahat_2<0)
-
I(\bH_{11}\trans\bgammabar_1<0)
I(\bH_{21}\trans\bgammabar_2<0)
\right\}\\
+&
A_1(1-A_2)\left\{
I(\bH_{11}\trans\bgammahat_1>0)
I(\bH_{21}\trans\bgammahat_2<0)
-
I(\bH_{11}\trans\bgammabar_1>0)
I(\bH_{21}\trans\bgammabar_2<0)
\right\}\\
+&
(1-A_1)A_2\left\{
I(\bH_{11}\trans\bgammahat_1<0)
I(\bH_{21}\trans\bgammahat_2>0)
-
I(\bH_{11}\trans\bgammabar_1<0)
I(\bH_{21}\trans\bgammabar_2>0)
\right\}
\bigg|
\\
\le&
A_1A_2
\bigg|
I(\bH_{11}\trans\bgammahat_1>0)
I(\bH_{21}\trans\bgammahat_2>0)
-
I(\bH_{11}\trans\bgammabar_1>0)
I(\bH_{21}\trans\bgammabar_2>0)
\bigg|\\
+&
(1-A_1)(1-A_2)\bigg|
I(\bH_{11}\trans\bgammahat_1<0)
I(\bH_{21}\trans\bgammahat_2<0)
-
I(\bH_{11}\trans\bgammabar_1<0)
I(\bH_{21}\trans\bgammabar_2<0)
\bigg|\\
+&
A_1(1-A_2)\bigg|
I(\bH_{11}\trans\bgammahat_1>0)
I(\bH_{21}\trans\bgammahat_2<0)
-
I(\bH_{11}\trans\bgammabar_1>0)
I(\bH_{21}\trans\bgammabar_2<0)
\bigg|\\
+&
(1-A_1)A_2\bigg|
I(\bH_{11}\trans\bgammahat_1<0)
I(\bH_{21}\trans\bgammahat_2>0)
-
I(\bH_{11}\trans\bgammabar_1<0)
I(\bH_{21}\trans\bgammabar_2>0)
\bigg|
\end{align*}
where the first step follows from above, the second step from the triangle inequality, now as $\bgammahat_1$, $\bgammahat_2$ have dimensions $q_{12},\:q_{22}$ respectively, we use sets $\mathcal{A}_{q_{12}}$, $\mathcal{A}_{q_{22}}$ and have 
\begin{align*}
&\bigg|I(\dhat_1=A_1)I(\dhat_2=A_2)
-
I( \dbar_1=A_1)I( \dbar_2=A_2)\bigg|\\
\le&
A_1A_2
I(\bgammahat_1\notin\mathcal{A}_{q_{12}})
I(\bgammahat_2\notin\mathcal{A}_{q_{22}})
+
(1-A_1)(1-A_2)
I(\bgammahat_1\notin\mathcal{A}_{q_{12}})
I(\bgammahat_2\notin\mathcal{A}_{q_{22}})\\
+&
A_1(1-A_2)
I(\bgammahat_1\notin\mathcal{A}_{q_{12}})
I(\bgammahat_2\notin\mathcal{A}_{q_{22}})
+
(1-A_1)A_2
I(\bgammahat_1\notin\mathcal{A}_{q_{12}})
I(\bgammahat_2\notin\mathcal{A}_{q_{22}})\\
=&
I(\bgammahat_1\notin\mathcal{A}_{q_{12}})
I(\bgammahat_2\notin\mathcal{A}_{q_{22}})
\end{align*}
 which follows from the fact that for any term within absolute value:
\[
\bigg|
I(\bH_{11}\trans\bgammahat_1<0)
I(\bH_{21}\trans\bgammahat_2>0)
-
I(\bH_{11}\trans\bgammabar_1<0)
I(\bH_{21}\trans\bgammabar_2>0)
\bigg|
=
I(\bgammahat_1\notin\mathcal{A}_{q_{12}})
I(\bgammahat_2\notin\mathcal{A}_{q_{22}})
\]
since for $I(\bH_{11}\trans\bgammabar_1<0)
I(\bH_{21}\trans\bgammabar_2>0)\neq I(\bH_{11}\trans\bgammahat_1<0)
I(\bH_{21}\trans\bgammahat_2>0)
$ both $\bgammahat_1,\bgammahat_2$ have to be outside sets $\mathcal{A}_{q_{12}},\mathcal{A}_{q_{22}}$ respectively. Thus
$
\bigg|I(\dhat_1=A_1)I(\dhat_2=A_2)
-
I( \dbar_1=A_1)I( \dbar_2=A_2)\bigg|
=O_\Pbb\left(n^{-\frac{1}{2}}\right),
$
we can analogous show that $\bigg|I(\dhat_1=A_1)-I( \dbar_1=A_1)\bigg|=O_\Pbb\left(n^{-\frac{1}{2}}\right)$ $\forall i$.\\
Next to see $\sup_{\bH_1}\left|\frac{1}{\pi_1(\bH_1;\bxihat_1)}-\frac{1}{\pi_1(\bH_1;\bxibar_1)}\right|=O_\Pbb\left(n^{-\frac{1}{2}}\right)$, note that as $\mathcal{H}_1$, $\Omega_1$ are bounded sets we have
	\begin{align*}
	&\sup_{\bH_{1}}\left|\frac{1}{\pi_1(\bH_1;\bxihat_1)}-\frac{1}{\pi_1(\bH_1;\bxibar_1)}\right|=\sup_{\bH_{1}\in\mathcal{H}_1}\left|e^{-\bH_{1}\trans\hat{
	\bxi_1}}-e^{-\bH_{1}\trans
	\bxibar_1}\right|
	\\
	 &\le\sup_{\bH_{1}\in\mathcal{H}_1,\bxi_1\in\Omega_1}\left|\frac{d}{d x}e^{-x}\big|_{x=\bH_{1}\trans
	 \bxi_1}\right|
	 \sup_{\bH_{1}\in\mathcal{H}_1}\left|\bH_{1}\trans\hat{
	 \bxi_1}-\bH_{1}\trans\bxibar\right|\\
	 &\le\sup_{\bH_{1}\in\mathcal{H}_1,\bxi_1\in\Omega_1}\left|\frac{d}{d x}e^{-x}\big|_{x=\bH_{1}\trans
	 \bxi_1}\right|
	 \sup_{\bH_1\in\mathcal{H}_1}\|\bH_1\|\left\| \hat{
	 \bxi}_1-\bxibar_1\right\|_2=O_\Pbb\left(n^{-\frac{1}{2}}\right),
	\end{align*}
	where we use the definition of $\pi_1$ in \eqref{logit_Ws}, Lipschitz and $\left\| \hat{
	 \bxi}_1-
	 \bxibar_1\right\|_2=O_\Pbb\left(n^{-\frac{1}{2}}\right)$ from Assumptions \eqref{assumption: donsker w} and Theorem 5.21 in \cite{vaart_donsker} as we are using Z-estimation for $\bxi_1$.\\
(c) By Theorem \ref{theorem: unbiased theta2} we have $\hat\beta_{21}-\bar\beta_{21}=O_\Pbb\left(n^{-\frac{1}{2}}\right)$. Next, we can write 
\[
\omega_1(\bH_1,A_1;\bThetahat_1)
=
I\left\{A_1=d_1\left(\bH_1;\bxihat_1\right)
\right\}\left\{
\frac{A_1}{\pi_1\left(\bH_1;\bxihat_1\right)}
+
\frac{1-A_1}{1-\pi_1\left(\bH_1;\bxihat_1\right)}
\right\}.
\]
By Lemma \ref{lemm_gamma_difs} (b) it follows that \begin{align*}
\Pbb_n\left[I\left\{A_1=d_1\left(\bH_1;\bxihat_1\right)\right\}-I\left\{A_1=d_1\left(\bH_1;\bxibar_1\right)
\right\}\right]
&=
O_\Pbb\left(n^{-\frac{1}{2}}\right),\\
\Pbb_n\left[\frac{A_1}{\pi_1(\bH_1;\bxihat_1)}-\frac{A_1}{\pi_1\left(\bH_1;\bxibar_1\right)}\right]&=O_\Pbb\left(n^{-\frac{1}{2}}\right),\\
\Pbb_n\left[\frac{1-A_1}{1-\pi_1(\bH_1;\bxihat_1)}-\frac{1-A_1}{1-\pi_1\left(\bH_1;\bxibar_1\right)}\right]&=O_\Pbb\left(n^{-\frac{1}{2}}\right).
\end{align*}
Using the above and Lemma \ref{lemma: Op product} we get 
\[
(1+\hat\beta_{21})\Pbb_n
		\left\{\omega_1(\bHcheck_1,A_1;\bThetahat_1)\right\}=(1+\bar\beta_{21})\Pbb_n
		\left\{\omega_1(\bHcheck_1,A_1;\bThetabar_1)\right\}+O_\Pbb\left(n^{-\frac{1}{2}}\right)
\]
Also by CLT we have
\begin{align*}
(1+\bar\beta_{21})\Pbb_n
		\left\{\omega_1(\bHcheck_1,A_1;\bThetabar_1)\right\}
		&=
		(1+\bar\beta_{21})\Ebb
		\left\{\omega_1(\bHcheck_1,A_1;\bThetabar_1)\right\}+O_\Pbb\left(n^{-\frac{1}{2}}\right),\\
		(1+\bar\beta_{21})\Pbb_N
		\left\{\omega_1(\bHcheck_1,A_1;\bThetabar_1)\right\}
		&=
		(1+\bar\beta_{21})\Ebb
		\left\{\omega_1(\bHcheck_1,A_1;\bThetabar_1)\right\}+O_\Pbb\left(N^{-\frac{1}{2}}\right),
\end{align*}
finally by Slutsky's theorem
$\hat C_{n,N}^{(1)}-1=O_\Pbb\left(n^{-\frac{1}{2}}\right).$ With similar arguments, and using Lemma \ref{lemm_gamma_difs} (a) to see $\Pbb_n\left(\left[\bH_{21}\trans\bgammahat_2\right]_+-\left[\bH_{21}\trans\bgammabar_2\right]_+\right)=O_\Pbb\left(n^{-\frac{1}{2}}\right)$, we can show $\hat C_{n,N}^{(2)}-1=O_\Pbb\left(n^{-\frac{1}{2}}\right).$\\

\end{proof}
    \begin{lemma}\label{lemm_bias_term}
    
	Let $Q_t(\bHcheck_t;\btheta_t)$, $\pi_t(\bHcheck_t;\bxi_t)$ $t=1,2$ be estimator functions of \eqref{linear_Qs} \& \eqref{logit_Ws} respectively and define the bias as $Bias\left(\Vbar,\mathcal{V}\subSUPDR\left(\bL;\bTheta\right)\right)
	\equiv
	\Vbar-\mathbb{E}\left[\mathcal{V}\subSUPDR\left(\bL;\bTheta\right)\right]$, then
	\begin{align*}
&\text{Bias}\left(\Vbar,\mathcal{V}\subSUPDR\left(\bL;\bTheta\right)\right)\\
=&\mathbb{E}\left[
\left\{1-
\frac{\pi_1(\bHcheck_1)}{\pi_1(\bHcheck_1;\bxi_1)}
\right\}
\left\{
\Qopt_1(\bHcheck_1)-\Qopt_1(\bHcheck_1;\btheta_1)
\right\}
\right]\\
+&
\mathbb{E}\left[
\left\{1-
\frac{1-\pi_1(\bHcheck_1)}{1-\pi_1(\bHcheck_1;\bxi_1)}
\right\}
\left\{
\Qopt_1(\bHcheck_1)-\Qopt_1(\bHcheck_1;\btheta_1)
\right\}
\right]\\
+&
\mathbb{E}\left[
\left\{
\frac{A_1}{\pi_1(\bHcheck_1;\bxi_1)}
+
\frac{1-A_1}{1-\pi_1(\bHcheck_1;\bxi_1)}
\right\}
\left\{
1-\frac{\pi_2(\bHcheck_2)}{\pi_2(\bHcheck_2;\bxi_2)}
\right\}
\left\{
\Qopt_2(\bHcheck_2)-\Qopt_2(\bHcheck_2;\btheta_2)
\right\}
\right]\\
+&
\mathbb{E}\left[
\left\{
\frac{A_1}{\pi_1(\bHcheck_1;\bxi_1)}
+
\frac{1-A_1}{1-\pi_1(\bHcheck_1;\bxi_1)}
\right\}
\left\{
1-\frac{1-\pi_2(\bHcheck_2)}{1-\pi_2(\bHcheck_2;\bxi_2)}
\right\}
\left\{
\Qopt_2(\bHcheck_2)-\Qopt_2(\bHcheck_2;\btheta_2)
\right\}
\right].
\end{align*}
	where $\Vbar=\mathbb{E}[\mathbb{E}[Y_2+\mathbb{E}[Y_3|\bH_2,Y_2,A_2=\dbar_2(\bHcheck_2)]|\bH_1,A_1=\dbar_1(\bHcheck_1)]]$ is the mean population value under the optimal treatment rule.

\end{lemma}

\begin{proof}[Proof of Lemma \ref{lemm_bias_term}]	
    \begin{align*}
    \text{Bias}\left(\Vbar,\mathcal{V}\subSUPDR\left(\bL;\bTheta\right)\right)=
	&\mathbb{E}[\mathbb{E}[Y_2+\mathbb{E}[Y_3|\bH_2,Y_2,A_2=\dbar_2]|\bH_1,A_1=\dbar_1]]-\mathbb{E}\left[\mathcal{V}\subSUPDR\left(\bL;\bTheta\right)\right]\\
    =&\mathbb{E}\left[\Qopt_1(\bH_1)-\Qopt_1(\bH_1;\btheta_1)\right]\\
    -&
    \mathbb{E}\left[
    \omega_1(\bHcheck_1,A_1;\bTheta_1)
    \left\{
    Y_2-\Qopt_1(\bHcheck_1;\btheta_1)
    \right\}
    \right]\\
    -&
    \mathbb{E}\left[
    \omega_1(\bHcheck_1,A_1;\bTheta_1)
    \Qopt_2(\bHcheck_2;\btheta_2)
    \right]\\
    -&
    \mathbb{E}\left[
    \omega_2(\bHcheck_2,A_2;\bTheta_2)
    \left\{
    Y_3-\Qopt_2(\bHcheck_2;\btheta_2)
    \right\}
    \right].
    \end{align*}
    Adding and subtracting $\mathbb{E}\left[
    \omega_1(\bHcheck_1,A_1;\bTheta_1)\Qopt_2(\bHcheck_2)\right]=\mathbb{E}\left[
    \omega_1(\bHcheck_1,A_1;\bTheta_1)\mathbb{E}[Y_3|\bH_2,\dbar_2(\bHcheck_2;\btheta_2),Y_2]\right]$,
    \begin{align*}
    &\text{Bias}\left(\Vbar,\mathcal{V}\subSUPDR\left(\bL;\bTheta\right)\right)\\
    &=
    \mathbb{E}\left[\Qopt_1(\bH_1)-\Qopt_1(\bH_1;\btheta_1)\right]\\
    &-
    \mathbb{E}\bigg[
    \omega_1(\bHcheck_1,A_1;\bTheta_1)
    \bigg\{
    Y_2+\mathbb{E}[Y_3|\bH_2\dbar_2(\bHcheck_2;\btheta_2),Y_2]-\Qopt_1(\bHcheck_1;\btheta_1)
    \bigg\}
    \bigg]\\
    &-
    \mathbb{E}\bigg[
    \omega_1(\bHcheck_1,A_1;\bTheta_1)
    \bigg\{
    \Qopt_2(\bHcheck_2;\btheta_2)-\Qopt_2(\bHcheck_2)
    \bigg\}
    \bigg]\\
    &
    -\mathbb{E}\left[
    \omega_2(\bHcheck_2,A_2;\bTheta_2)
    \left\{
    Y_3-\Qopt_2(\bHcheck_2;\btheta_2)
    \right\}
    \right],
    \end{align*}
using iterated expectations in the second and fourth terms:
\begin{align*}
&\text{Bias}\left(\Vbar,\mathcal{V}\subSUPDR\left(\bL;\bTheta\right)\right)
\\
&=\mathbb{E}\left[\Qopt_1(\bHcheck_1)-\Qopt_1(\bHcheck_1;\btheta_1)\right]
\\
&-
\mathbb{E}\left[\mathbb{E}\left[
\omega_1(\bHcheck_1,A_1;\bTheta_1)
\left\{
Y_2+\mathbb{E}[Y_3|\bHcheck_2,\dbar_2(\bHcheck_2),Y_2]-\Qopt_1(\bHcheck_1;\btheta_1)
\right\}
\bigg|\bHcheck_1,A_1\right]\right]
\\
&-
\mathbb{E}\left[
\omega_1(\bHcheck_1,A_1;\bTheta_1)
\left\{
\Qopt_2(\bHcheck_2;\btheta_2)-\Qopt_2(\bHcheck_2)
\right\}
\right]
\\
&-
\mathbb{E}\left[\mathbb{E}\left[
\omega_2(\bHcheck_2,A_2;\bTheta_2)
\left\{
Y_3-\Qopt_2(\bHcheck_2;\btheta_2)
\right\}
\bigg|\bHcheck_2,A_2,Y_2\right]\right]\\
&=
\mathbb{E}\left[\Qopt_1(\bHcheck_1)-\Qopt_1(\bHcheck_1;\btheta_1)\right]\\
&-
\mathbb{E}\left[
\omega_1(\bHcheck_1,A_1;\bTheta_1)
\left\{
\mathbb{E}\left[Y_2+\mathbb{E}[Y_3|\bHcheck_2,\dbar_2(\bHcheck_2),Y_2]\bigg|\bHcheck_1,A_1\right]-\Qopt_1(\bHcheck_1;\btheta_1)
\right\}
\right]
\\
&-\mathbb{E}\left[
\omega_1(\bHcheck_1,A_1;\bTheta_1)
\left\{
\Qopt_2(\bHcheck_2;\btheta_2)-\Qopt_2(\bHcheck_2)
\right\}
\right]
\\
&-
\mathbb{E}\left[
\omega_2(\bHcheck_2,A_2;\bTheta_2)
\left\{
\mathbb{E}\left[Y_3|\bHcheck_2,A_2,Y_2\right]-\Qopt_2(\bHcheck_2;\btheta_2)
\right\}
\right].
\end{align*}
using definitions of $\omega_t(\bHcheck_t,A_t;\bTheta_t)$ $t=1,2$ we can write:
\begin{align*}
&\text{Bias}\left(\Vbar,\mathcal{V}\subSUPDR\left(\bL;\bTheta\right)\right)\\
=&\mathbb{E}\left[
\Qopt_1(\bHcheck_1)-\Qopt_1(\bHcheck_1;\btheta_1)\right]\\
-&
\mathbb{E}\left[
\left\{
\frac{\dbar_1A_1}{\pi_1(\bHcheck_1;\bxi_1)}
+
\frac{(1-\dbar_1)(1-A_1)}{1-\pi_1(\bHcheck_1;\bxi_1)}
\right\}
\left\{
\Qopt_1(\bHcheck_1)-\Qopt_1(\bHcheck_1;\btheta_1)
\right\}
\right]\\
-&
\mathbb{E}\left[
\frac{\dbar_1A_1}{\pi_1(\bHcheck_1;\bxi_1)}
\left\{
\Qopt_2(\bHcheck_2)-\Qopt_2(\bHcheck_2;\btheta_2)
\right\}
\right]
-
\mathbb{E}\left[
\frac{(1-\dbar_1)(1-A_1)}{1-\pi_1(\bHcheck_1;\bxi_1)}
\left\{
\Qopt_2(\bHcheck_2)-\Qopt_2(\bHcheck_2;\btheta_2)
\right\}
\right]\\
-&
\Ebb\left[
\left\{
\frac{\dbar_1A_1}{\pi_1(\bHcheck_1;\bxi_1)}
+
\frac{(1-\dbar_1)(1-A_1)}{1-\pi_1(\bHcheck_1;\bxi_1)}
\right\}
\left\{
\frac{\dbar_2A_2}{\pi_2(\bHcheck_2;\bxi_2)}
+
\frac{(1-\dbar_2)(1-A_2)}{1-\pi_2(\bHcheck_2;\bxi_2)}
\right\}
\left\{
\Qopt_2(\bHcheck_2)-\Qopt_2(\bHcheck_2;\btheta_2)
\right\}
\right]
\end{align*}
assuming $A_1\perp A_2|\bH_2,Y_2$, we use iterated expectations:
\begin{align*}
&\text{Bias}\left(\Vbar,\mathcal{V}\subSUPDR\left(\bL;\bTheta\right)\right)\\
=&\mathbb{E}\left[
\Qopt_1(\bHcheck_1)-\Qopt_1(\bHcheck_1;\btheta_1)\right]\\
-&
\mathbb{E}\left[\mathbb{E}\left[
\left\{
\frac{\dbar_1A_1}{\pi_1(\bHcheck_1;\bxi_1)}
+
\frac{(1-\dbar_1)(1-A_1)}{1-\pi_1(\bHcheck_1;\bxi_1)}
\right\}
\left\{
\Qopt_1(\bHcheck_1)-\Qopt_1(\bHcheck_1;\btheta_1)
\right\}
\bigg|\bHcheck_1\right]\right]\\
-&
\mathbb{E}\left[
\left\{
\frac{\dbar_1A_1}{\pi_1(\bHcheck_1;\bxi_1)}
+
\frac{(1-\dbar_1)(1-A_1)}{1-\pi_1(\bHcheck_1;\bxi_1)}
\right\}
\left\{
\Qopt_2(\bHcheck_2)-\Qopt_2(\bHcheck_2;\btheta_2)
\right\}
\right]\\
-&
\Ebb\left[\mathbb{E}\left[
\left\{
\frac{\dbar_1A_1}{\pi_1(\bHcheck_1;\bxi_1)}
+
\frac{(1-\dbar_1)(1-A_1)}{1-\pi_1(\bHcheck_1;\bxi_1)}
\right\}
\left\{
\frac{\dbar_2A_2}{\pi_2(\bHcheck_2;\bxi_2)}
+
\frac{(1-\dbar_2)(1-A_2)}{1-\pi_2(\bHcheck_2;\bxi_2)}
\right\}
\left\{
\Qopt_2(\bHcheck_2)-\Qopt_2(\bHcheck_2;\btheta_2)
\right\}
\bigg|\bHcheck_2\right]\right]\\
=&\mathbb{E}\left[
\Qopt_1(\bHcheck_1)-\Qopt_1(\bHcheck_1;\btheta_1)\right]\\
-&
\mathbb{E}\left[
\left\{
\frac{\dbar_1\pi_1(\bHcheck_1)}{\pi_1(\bHcheck_1;\bxi_1)}
+
\frac{\{1-\dbar_1\}\{1-\pi_1(\bHcheck_1)\}}{1-\pi_1(\bHcheck_1;\bxi_1)}
\right\}
\left\{
\Qopt_1(\bHcheck_1)-\Qopt_1(\bHcheck_1;\btheta_1)
\right\}
\right]\\
-&
\mathbb{E}\left[
\left\{
\frac{\dbar_1A_1}{\pi_1(\bHcheck_1;\bxi_1)}
+
\frac{(1-\dbar_1)(1-A_1)}{1-\pi_1(\bHcheck_1;\bxi_1)}
\right\}
\left\{
\Qopt_2(\bHcheck_2)-\Qopt_2(\bHcheck_2;\btheta_2)
\right\}
\right]\\
-&
\Ebb\left[
\left\{
\frac{\dbar_1A_1}{\pi_1(\bHcheck_1;\bxi_1)}
+
\frac{(1-\dbar_1)(1-A_1)}{1-\pi_1(\bHcheck_1;\bxi_1)}
\right\}
\left\{
\frac{\dbar_2\pi_2(\bHcheck_2)}{\pi_2(\bHcheck_2;\bxi_2)}
+
\frac{\{1-\dbar_2\}\{1-\pi_2(\bHcheck_2)\}}{1-\pi_2(\bHcheck_2;\bxi_2)}
\right\}
\left\{
\Qopt_2(\bHcheck_2)-\Qopt_2(\bHcheck_2;\btheta_2)
\right\}
\right]
\end{align*}
 
finally, factorizing common terms:
\begin{align*}
&\text{Bias}\left(\Vbar,\mathcal{V}\subSUPDR\left(\bL;\bTheta\right)\right)\\
=&\mathbb{E}\left[
\dbar_1\left\{1-
\frac{\pi_1(\bHcheck_1)}{\pi_1(\bHcheck_1;\bxi_1)}
\right\}
\left\{
\Qopt_1(\bHcheck_1)-\Qopt_1(\bHcheck_1;\btheta_1)
\right\}
\right]\\
+&
\mathbb{E}\left[
\{1-\dbar_1\}\left\{1-
\frac{1-\pi_1(\bHcheck_1)}{1-\pi_1(\bHcheck_1;\bxi_1)}
\right\}
\left\{
\Qopt_1(\bHcheck_1)-\Qopt_1(\bHcheck_1;\btheta_1)
\right\}
\right]\\
+&
\mathbb{E}\left[
\dbar_2\left\{
\frac{A_1}{\pi_1(\bHcheck_1;\bxi_1)}
+
\frac{1-A_1}{1-\pi_1(\bHcheck_1;\bxi_1)}
\right\}
\left\{
1-\frac{\pi_2(\bHcheck_2)}{\pi_2(\bHcheck_2;\bxi_2)}
\right\}
\left\{
\Qopt_2(\bHcheck_2)-\Qopt_2(\bHcheck_2;\btheta_2)
\right\}
\right]\\
+&
\mathbb{E}\left[
\{1-\dbar_2\}\left\{
\frac{A_1}{\pi_1(\bHcheck_1;\bxi_1)}
+
\frac{1-A_1}{1-\pi_1(\bHcheck_1;\bxi_1)}
\right\}
\left\{
1-\frac{1-\pi_2(\bHcheck_2)}{1-\pi_2(\bHcheck_2;\bxi_2)}
\right\}
\left\{
\Qopt_2(\bHcheck_2)-\Qopt_2(\bHcheck_2;\btheta_2)
\right\}
\right]\\
\le
&
\sqrt{\sup_{\bHcheck_1}|\{1-\pi_1(\bHcheck_1; \bxi_1)\}^{-1}|}
\sqrt{\|\pi_1(\bHcheck_1; \bxi_1)-\pi_1(\bHcheck_1)\|_{L_2(\Pbb)}}
\sqrt{\|\Qopt_1(\bHcheck_1;  \bthetahat_1)-\Qopt_1(\bHcheck_1)\|_{L_2(\Pbb)}}\\
+&
\sqrt{\sup_{\bHcheck_2}
\left|
\left\{
\frac{A_1}{\pi_1(\bHcheck_1;\bxi_1)}
+
\frac{1-A_1}{1-\pi_1(\bHcheck_1;\bxi_1)}
\right\}\{1-\pi_1(\bHcheck_1; \bxi_1)\}^{-1}\{1-\pi_2(\bHcheck_2; \bxi_2)\}^{-1}\right|}\\
\times&
\sqrt{\|\pi_2(\bHcheck_2; \bxi_2)-\pi_2(\bHcheck_2)\|_{L_2(\Pbb)}}
\sqrt{\|\Qopt_2(\bHcheck_2;  \bthetahat_2)-\Qopt_2(\bHcheck_2)\|_{L_2(\Pbb)}},
\end{align*}
which follows by Cauchy–Schwarz Inequality.
\end{proof}

\section{Additional Theoretical Results}
\subsection{Augmented value function estimation}\label{app_DR_Vfun}
We first re-write Assumption \ref{assumption: donsker w} to account for only using sample $\mathcal{L}$ in estimation of the $Q$ functions and propensity scores. 
\begin{assumption}\label{assumption: lab donsker}
	Define the following class of functions:
	\begin{align}\label{lab_func_sets}
	\mathcal{Q}_1&\equiv\left\{Q_1(\bH_1,A_1;\btheta_1)|\btheta_1\in\Theta_1\subset\mathbb{R}^{q_1}\right\},\nonumber \\
	\mathcal{Q}_2&\equiv\left\{Q_2(\bH_2,A_2,Y_2;\btheta_2)|\btheta_2\in\Theta_2\subset\mathbb{R}^{q_2}\right\},\\
	\mathcal{W}_1&\equiv\left\{\pi_1(\bH_1;\bxi_1)|\bxi_1\in\Omega_1\subset\mathbb{R}^{p_1}\right\},\nonumber \\
	\mathcal{W}_2&\equiv\left\{\pi_2(\bHcheck_2;\bxi_2)|\bxi_2\in\Omega_2\subset\mathbb{R}^{p_2}\right\},\nonumber
	\end{align}
	with $p_1,p_2,q_1,q_2$ fixed under model definitions \eqref{linear_Qs} \& \eqref{logit_Ws}. Let the population equations $\mathbb{E}\left[S^\xi_t(\bxi_t)\right]=\bzero, t=1,2$ have solutions $\bxibar_1,\bxibar_2$, where
	\begin{align*}
	S^\xi_1(\bxi_1)=&\frac{\partial}{\partial\bxi_1}\log \left[\pi_1(\bH_1;\bxi_1)^{A_1}(1-\pi_1(\bH_1;\bxi_1))^{(1-A_1)}\right],\\ S^\xi_2(\bxi_2)=&\frac{\partial}{\partial\bxi_2}\log \left[ \pi_2(\bHcheck_2;\bxi_2)^{A_2}(1-\pi_2(\bHcheck_2;\bxi_2))^{(1-A_2)}\right],
	\end{align*}
	and the population equations for the $Q$ functions $\mathbb{E}[S^\theta_t(\btheta_t)]=\bzero,t=1,2$ have solutions $\bthetabar_1,\bthetabar_2$, where 
	\begin{align*}
	S^\theta_2(\btheta_2)=&\frac{\partial}{\partial\btheta_2\trans}
	\|Y_3-Q_2(\bHcheck_2,A_2;\btheta_2)\|_2^2,\\
	S^\theta_1(\btheta_1)=&\frac{\partial}{\partial\btheta_1\trans}
	\|Y_2+\Qbar_2(\bHcheck_2;\bthetabar_2)-Q_1(\bH_1,A_1;\btheta_1)\|_2^2,
	\end{align*}
	(i) $\xi_1,\xi_2$ are bounded sets. (ii) $\Theta_1,\Theta_2$ are open bounded sets and for some $r>0$ and $g_t(\cdot)$ 
	\begin{align}\label{lab_donsker_q}
	\bigg|Q_t(\cdot;\btheta_t)-Q_t(\cdot;\btheta_t')\bigg|\leq g_t(\cdot)\|\btheta_t-\btheta_t'\|\:\:\forall\btheta_t,\btheta_t'\in\Theta_t,\:\mathbb{E}\left[|g_t(\cdot)|^r\right]<\infty,\:t=1,2.
	\end{align}
	(iii) The population minimizers satisfy $\bthetabar_t\in\Theta_t,\bxibar_t\in\Omega_t,t=1,2$.
	(iv) For $\bxibar_t,t=1,2$, $\bar \pi_1(\bH_1;\bxibar_1)>0,\bar \pi_2(\bHcheck_2;\bxibar_2)>0\:\forall \bH\in\mathcal{H}$.
\end{assumption}
Existence of solutions $\bthetabar_t\in\Theta_t$, $t=1,2$ is clear as $\Theta_1,\Theta_2$ are open and bounded.

\begin{theorem}[Asymptotic Normality for $\Vhat\subSSLDR$]\label{thrm_supV_fun}
	Under Assumptions \ref{assumption: covariates}, \ref{assumption: non-regularity}, and \ref{assumption: lab donsker}, $\Vhat\subSUPDR$ as defined in \eqref{eq: lab value fun} is such that 
	\[
	\sqrt n
	\left\{
	\Vhat\subSUPDR-\mathbb{E}_\mathbb{S}\left[\Vsc\subSUPDR(\bL;\bThetabar)\right]
	\right\}
	=
	\frac{1}{\sqrt n}\sum_{i=1}^n\psi^{v}_{\subSUPDR}(\bL_i;\bThetabar)
	+o_\Pbb\left(1\right)\stackrel{d}{\longrightarrow}N\left(0,\sigma\subSUPDR^2\right).
	\]

	where
	
	\begin{align*}
	\psi^{v}\subSUPDR(\bL;\bThetabar)
	=&
	\Vsc\subSUPDR(\bL;\bThetabar)-\mathbb{E}_\mathbb{S}\left[\Vsc\subSUPDR(\bL;\bThetabar)\right]+
	\bpsi\subSUP^\theta(\bL)\trans
	\frac{\partial}{\partial \btheta}\int\Vsc\subSUPDR(\bL;\bThetabar)d\Pbb_{\bL}\bigg|_{\bTheta=\bThetabar}\\
	&
	+
	\bpsi\subSUP^\xi(\bL)\trans
	\frac{\partial}{\partial \bxi}\int\Vsc\subSUPDR(\bL;\bThetabar)d\Pbb_{\bL}\bigg|_{\bTheta=\bThetabar},\\
	\sigma\subSUPDR^2=&\mathbb{E}\left[\psi^{v}_{\subSUPDR}(\bL;\bThetabar)^2\right].
	\end{align*}
\end{theorem}

\begin{proof}[proof of theorem \ref{thrm_supV_fun}]

Letting $g(\bTheta)=\int\Vsc\subSUPDR(\bL; \bTheta)d\Pbb_{\bL}$, we start by centering \eqref{eq: lab value fun} and scaling by $\sqrt n$: 

\begin{align*}
&\sqrt n
\left\{
\Pbb_n\left(\Vsc\subSUPDR(\bL; \bThetahat\subSUP)\right)-\mathbb{E}\left[\Vsc\subSUPDR(\bL; \bThetabar)\right]
\right\}
\\=
&\mathbb{G}_n
\left\{
\Vsc\subSUPDR(\bL; \bThetabar)
\right\}
+
\mathbb{G}_n
\left\{
\Vsc\subSUPDR(\bL; \bThetahat\subSUP)-\Vsc\subSUPDR(\bL; \bThetabar)
\right\}
+
\sqrt n
\left\{
g(\bThetahat\subSUP)-g(\bThetabar)
\right\}
\end{align*}

\textbf{ I) Empirical Process Term}\\

	We first show that under Assumption \ref{assumption: lab donsker}, $\mathbb{G}_n
    \left\{
    \Vsc\subSUPDR(\bL; \bThetahat\subSUP)-\Vsc\subSUPDR(\bL; \bThetabar)
    \right\}
    =o_\Pbb(1)$, let	
	\begin{align*}
	f_{\bTheta}(\bUvec)
	=
	\Qopt_1(\bHcheck_1;\btheta_1)
	+&
	\omega_1(\bHcheck_1,A_1,\bTheta)
	\left\{Y_2-
	 \Qopt_1(\bHcheck_1;\btheta_1)+\Qopt_2(\bHcheck_2;\btheta_2)
	\right\}\\
	+&
	\omega_2(\bHcheck_1,A_1;\bTheta)
	\left\{
	Y_3-\Qopt_2(\bHcheck_2;\btheta_2)
	\right\},
	\end{align*}
	
	we define the class of functions $\mathcal{C}_3=\left\{
	f_{\bTheta}(\bUvec)
	|\bUvec,\bTheta\in\mathcal{S}(\delta)
	\right\},$ and \[
	\ell=\{l:\{0,1\}^2\mapsto\{0,1\}\}.
	\]
	
	i) By Assumptions \ref{assumption: lab donsker} and
	Theorem 19.5 in \cite{vaart_donsker}, $\ell,\:\mathcal{W}_t,\:\mathcal{Q}_t,t=1,2$ are a $\Pbb$-Donsker class, thus it follows that $\mathcal{C}_3$ is a Donsker class. 
	
	ii) We estimate $\bxi_1,\bxi_2$ from \eqref{lab_func_sets} with their maximum likelihood estimator $\bxihat_{1 \scriptscriptstyle \sf SUP},\bxihat_{2 \scriptscriptstyle \sf SUP}$, solving $\Pbb_n\left[S_t(\bxi_t)\right]=\bzero, t=1,2$ and estimate functions $ \pi_1(\bH_1;\bxihat_{1 \scriptscriptstyle \sf SUP}),$ $\pi_2(\bHcheck_2;\bxihat_{2\subSUP})$ with $\bxihat_{1 \scriptscriptstyle \sf SUP},\bxihat_{2 \scriptscriptstyle \sf SUP}$. By Assumption \ref{assumption: lab donsker} and weak law of large numbers $\bxihat_{t \scriptscriptstyle \sf SUP}\stackrel{p}{\longrightarrow}\bxibar_t, t=1,2$.

    Analogous, under regularity conditions \eqref{stage2_mu} and \eqref{stage1} have unique solutions $\bthetahat_{t \scriptscriptstyle \sf SUP}$ for which $\bthetahat_{t \scriptscriptstyle \sf SUP}\stackrel{p}{\longrightarrow}\bthetabar_t, t=1,2$ by Assumption \ref{assumption: lab donsker} and weak law of large numbers. Both regardless of whether models \eqref{linear_Qs} \& \eqref{logit_Ws} are correct.
    Thus $\Pbb\left(\bThetahat\subSUP\in\mathcal{S}(\delta)\right)\rightarrow1,\:\forall\delta.$ 

	iii) We next show $\int\left\{\Vsc\subSUPDR(\bL; \bThetahat\subSUP)-\Vsc\subSUPDR(\bL; \bThetabar)\right\}^2d\Pbb_{\bL}\longrightarrow0.$
	Using \eqref{SS_value_fun}, for a large enough constant $c$ we can write
	\begin{align*}
	&\int\left\{\Vsc\subSUPDR(\bL; \bThetahat\subSUP)-\Vsc\subSUPDR(\bL; \bThetabar)\right\}^2d\Pbb_{\bL}\\
	\le&
	c\sup_{\bH_1}
    \left(\bH_{10}\trans\bbetabar_1+[\bH_{11}\trans\bgammabar_1]_+-\bH_{10}\trans\bbetahat_{1 \scriptscriptstyle \sf SUP}-[\bH_{11}\trans\bgammahat_{1 \scriptscriptstyle \sf SUP}]_+\right)^2
	\\
	+&
	c\sup_{\bHcheck_2}
	\left(\bHcheck_{20}\trans\bbetabar_2+[\bH_{21}\trans\bgammabar_2]_+-\bHcheck_{20}\trans\bbetahat_{2 \scriptscriptstyle \sf SUP}-[\bH_{21}\trans\bgammahat_{2 \scriptscriptstyle \sf SUP}]_+\right)^2
	\\
	+&
	c\sup_{\bH_1,A_1}
	\left\{
	\omega_1(\bH_1,A_1;\bThetahat_{1 \scriptscriptstyle \sf SUP})
	-
	\omega_1(\bH_1,A_1;\bThetabar_1)
	\right\}^2
	+
	\left(\hat\beta_{21 \scriptscriptstyle \sf SUP}-\bar\beta_{21}
	\right)^2
	\\
	&\longrightarrow0
	\end{align*}
	
		where we use $(a-b)^2,(a+b)^2\leq2a^2+2b^2\:\forall a,b\in\mathbb{R}$, boundedness of $\bThetabar$ and covariates by Assumptions \ref{assumption: covariates}, \ref{assumption: Q imputation}, and \ref{assumption: lab donsker}. Next,

	from assumption \eqref{SS_value_fun} it can be shown that $\bthetahat_{2 \scriptscriptstyle \sf SUP}-\bthetabar_2=O_\Pbb\left(n^{-\frac{1}{2}}\right)$, $\bthetahat_{1 \scriptscriptstyle \sf SUP}-\bthetabar_1=O_\Pbb\left(n^{-\frac{1}{2}}\right)$, also from Lemma \ref{lemm_gamma_difs} (a) it follows that for $t=1,2$
	\begin{align*}
	&
	\sup_{\bHcheck_t}
    \left(\bH_{t0}\trans\bbetabar_t+[\bH_{t1}\trans\bgammabar_t]_+-\bH_{t0}\trans\bbetahat_{t \scriptscriptstyle \sf SUP}-[\bH_{t1}\trans\bgammahat_{t \scriptscriptstyle \sf SUP}]_+\right)^2
	\\
	&\le
	2\sup_{\bHcheck_{t0}}\|\bHcheck_{t0}\|_2^2\|\bbetahat_t-\bbetabar_t\|_2^2+
	2\sup_{\bH_{t1}}\|\bH_{t1}\|_2^2
	\|\bgammahat_t-\bgammabar_t\|_2\\
	&=
	O_\Pbb\left(n^{-1}\right).
	\end{align*} 
	Next, we can write 
\begin{align*} 
\omega_1(\bH_1,A_1;\bThetahat_{1 \scriptscriptstyle \sf SUP})
=&
I\left\{A_1=d_1\left(\bH_1;\bxihat_{1 \scriptscriptstyle \sf SUP}\right)
\right\}\left\{
\frac{A_1}{\pi_1\left(\bH_1;\bxihat_{1 \scriptscriptstyle \sf SUP}\right)}
+
\frac{1-A_1}{1-\pi_1\left(\bH_1;\bxihat_{1 \scriptscriptstyle \sf SUP}\right)}
\right\}.
\end{align*} 
By Lemma \ref{lemm_gamma_difs} (b) it follows that 
\begin{align*}
	\sup_{\bH_1}
	\bigg|\frac{1}
    {\pi_1(\bH_1;\bxihat_{1 \scriptscriptstyle \sf SUP})}
    -
    \frac{1}{ \pi_1(\bH_1;\bxibar_1)}\bigg|=&O_\Pbb\left(n^{-\frac{1}{2}}\right).
	\end{align*}
Using the above and Lemma \ref{lemma: Op product} we get 
\begin{align*}
\sup_{\bHcheck_1,A_1}\left\{ \omega_1(\bHcheck_1,A_1;\bThetahat_{1 \scriptscriptstyle \sf SUP})-\omega_1(\bHcheck_1,A_1;\bThetabar_1)\right\}^2
		&=o_\Pbb\left(1\right),
\end{align*}
	which gives us $\int\left\{\Vsc\subSUPDR(\bL; \bThetahat\subSUP)-\Vsc\subSUPDR(\bL; \bThetabar)\right\}^2d\Pbb_{\bL}\longrightarrow0.$

	Hence, we have i) $\Pbb\left(\bThetahat\subSUP\in\mathcal{S}(\delta)\right)\rightarrow1,\:\forall\delta,$ ii) $\mathcal{C}_1$ is a Donsker class, and\\ iii) $\int\left\{\Vsc\subSUPDR(\bL; \bThetahat\subSUP)-\Vsc\subSUPDR(\bL; \bThetabar)\right\}^2d\Pbb_{\bL}\longrightarrow0$, then by Theorem 2.1 in \cite{wellner_emp} 
	\[
	\sqrt n\left[\Pbb_n\left\{\Vsc\subSUPDR(\bL; \bThetahat\subSUP)-g(\bThetahat\subSUP)\right\}
	-
	\Pbb_n\left\{\Vsc\subSUPDR(\bL; \bThetabar)-g(\bThetabar)\right\}\right]=o_\Pbb(1).
	\]
	
\textbf{Centered Sample Average}\\

Next we consider $\mathbb{G}_n
\left\{
\Vsc\subSUPDR(\bL; \bThetabar)
\right\}$. Note that $\Vsc\subSUPDR(\bL; \bThetabar)$ is a deterministic function of random variable $\bL$ as parameters are fixed. We have that $\Ebb\left[\left(\Vsc\subSUPDR(\bL; \bThetabar\right)^2\right]<\infty$ holds by Assumption \ref{assumption: covariates} \& \ref{assumption: lab donsker}. Thus the central limit theorem yields 
\[
\mathbb{G}_n
\left\{
\Vsc\subSUPDR(\bL; \bThetabar)
\right\}\stackrel{d}{\longrightarrow}\mathcal{N}\left(0,Var\left[\Vsc\subSUPDR(\bL; \bThetabar)\right]\right).
\]

\textbf{Bias Term}\\
We finally analyze the bias: $\sqrt n\left\{g(\bThetahat\subSUP)-g(\bThetabar)\right\}$. Using a Taylor series expansion
	\be
	g(\bThetahat\subSUP)
	=
	g(\bThetabar)
	+
	(\bthetahat\subSUP-\bthetabar)\trans\frac{\partial}{\partial \btheta\subSUP}g(\bThetabar)
	+
	(\bxihat\subSUP-\bxibar)\trans\frac{\partial}{\partial \bxi\subSUP}g(\bThetabar)
	+O_\Pbb\left(n^{-1}\right),
	\ee
	therefore 
	\be
	\sqrt n\left\{g(\bThetahat\subSUP)-g(\bThetabar)\right\}
	=&
	\sqrt n(\bthetahat\subSUP-\bthetabar)\trans\frac{\partial}{\partial \btheta\subSUP}g(\bThetabar)
	+
	\sqrt n(\bxihat\subSUP-\bxibar)\trans\frac{\partial}{\partial \bxi\subSUP}g(\bThetabar)
	+o_\Pbb(1).
	\ee
    Using the $Q$-function and propensity score function influence functions we can write 
	\begin{align*}
	&\sqrt n\left\{g(\bThetahat\subSUP)-g(\bThetabar)\right\}
	=
	\frac{\partial}{\partial \btheta\subSUP}g(\bThetabar)\frac{1}{\sqrt n}\sum_{i=1}^n\bpsi\subSUP^\theta(\bL_i)
	+
	\frac{\partial}{\partial \bxi}g(\bThetabar)\frac{1}{\sqrt n}\sum_{i=1}^n\bpsi\subSUP^\xi(\bL_i)
	+o_\Pbb(1)
	\end{align*}
\comment{
Now lets analyze the bias term,
By Lemma \ref{lemm_bias_term}
\begin{align*}
\sqrt nBias(\hat{\mathcal{V}}\subSSLDR,\mathcal{V}_{Q^0,w^0})=
&\sqrt n\mathbb{E}\left[
\left\{1-
\frac{\pi_1^0(\bH_1)}{\hat \pi_1(\bH_1)}
\right\}
\left\{
Q^{0*}_1(\bH_1)-\hat Q^*_1(\bH_1)
\right\}
\right]\\
+&
\sqrt n\mathbb{E}\left[
\frac{\pi_1^0(\bH_1)}{\hat \pi_1(\bH_1)}
\left\{
1-\frac{\pi_2^0(\bHcheck_2)}{\hat \pi_2(\bHcheck_2)}
\right\}
\left\{
Q_2^{*0}(\bHcheck_2)-\hat \Qbar_2(\bHcheck_2)
\right\}
\right]\\
\leq
&\sqrt n\sqrt{\mathbb{E}\left[
	\left\{1-
	\frac{\pi_1^0(\bH_1)}{\hat \pi_1(\bH_1)}
	\right\}^2
	\right]
	\mathbb{E}\left[
	\left\{
	Q^{0*}_1(\bH_1)-\hat Q^*_1(\bH_1)
	\right\}^2
	\right]}\\
+&
\sqrt n\sqrt{\mathbb{E}\left[
	\left(\frac{\pi_1^0(\bH_1)}{\hat \pi_1(\bH_1)}\right)^2
	\left\{
	1-\frac{\pi_2^0(\bHcheck_2)}{\hat \pi_2(\bHcheck_2)}
	\right\}^2
	\right]}\\
	&\hspace{.5cm}\times
	\sqrt{
	\mathbb{E}\left[
	\left\{
	Q_2^{*0}(\bHcheck_2)-\hat \Qbar_2(\bHcheck_2)
	\right\}^2
	\right]}\\
=&\sqrt n\sqrt{\mathbb{E}\left[\frac{1}{\hat \pi_1(\bH_1)^2}
	\left\{\hat \pi_1(\bH_1)-
	\pi_1^0(\bH_1)
	\right\}^2
	\right]}\\
	&\hspace{.5cm}\times
	\sqrt{
	\mathbb{E}\left[
	\left\{
	Q^{0*}_1(\bH_1)-\hat Q^*_1(\bH_1)
	\right\}^2
	\right]}\\
+&
\sqrt n\sqrt{\mathbb{E}\left[
	\left(\frac{\pi_1^0(\bH_1)}{\hat \pi_1(\bH_1)\hat \pi_2(\bHcheck_2)}\right)^2
	\left\{
	\hat \pi_2(\bHcheck_2)-\pi_2^0(\bHcheck_2)
	\right\}^2
	\right]}\\
	&\hspace{.5cm}\times
	\sqrt{
	\mathbb{E}\left[
	\left\{
	Q_2^{*0}(\bHcheck_2)-\hat \Qbar_2(\bHcheck_2)
	\right\}^2
	\right]}\\
    \leq&\sqrt{\sup_{\bH_1}\frac{1}{\hat \pi_1(\bH_1)^2}}
    \sqrt n\bigg\|\hat \pi_1(\bH_1)-
    \pi_1^0(\bH_1)
    \bigg\|_\Pbb
    \bigg\|
    Q^{0*}_1(\bH_1)-\hat Q^*_1(\bH_1)
    \bigg\|_\Pbb\\
    +&
    \sqrt{\sup_{\bH\in\mathcal{H}}\left(\frac{\pi_1^0(\bH_1)}{\hat \pi_1(\bH_1)\hat \pi_2(\bHcheck_2)}\right)^2}\\
    \times&
    \sqrt n\bigg\|\hat \pi_2(\bHcheck_2)-\pi_2^0(\bHcheck_2)\bigg\|_\Pbb
    \bigg\|
    Q_2^{*0}(\bHcheck_2)-\hat \Qbar_2(\bHcheck_2)
    \bigg\|_\Pbb
\end{align*}
by Assumptions (\ref{assumption: donsker w}) and bounded covariates $\sqrt{\sup_{\bH_1}\frac{1}{\hat \pi_1(\bH_1)^2}},\sqrt{\sup_{\bH\in\mathcal{H}}\left(\frac{\pi_1^0(\bH_1)}{\hat \pi_1(\bH_1)\hat \pi_2(\bHcheck_2)}\right)^2}$ are bounded, thus as long as either the $\hat Q_t$, or $\hat \pi_t,t=1,2$ are correctly specified $Bias(\hat{\mathcal{V}}\subSSLDR,\mathcal{V}_{Q^0,w^0})\stackrel{p}{\rightarrow}0$. Further, if $\sqrt n\|\hat \pi_2-\pi_2^0\|_\Pbb
\|
Q_2^{*0}-\hat \Qbar_2
\|_\Pbb=o_\Pbb(1)$ and $\sqrt n\|\hat \pi_1-
\pi_1^0
\|_\Pbb
\|
Q^{0*}_1-\hat Q^*_1
\|_\Pbb=o_\Pbb(1)$ then $\sqrt n Bias(\hat{\mathcal{V}}\subSSLDR,\mathcal{V}_{Q^0,w^0})=o_\Pbb(1).$
}
\end{proof}


\end{document}